\documentclass{article}

\PassOptionsToPackage{numbers, sort&compress}{natbib}
\usepackage[main,final]{neurips_2026}
\usepackage[utf8]{inputenc} 
\usepackage[T1]{fontenc}    
\usepackage{hyperref}       
\usepackage{url}            
\usepackage{booktabs}       
\usepackage{amsfonts}       
\usepackage{nicefrac}       
\usepackage{microtype}      
\usepackage[dvipsnames]{xcolor}         
\hypersetup{hidelinks}
\usepackage{amsmath}
\usepackage{amsthm}
\newtheorem{theorem}{Theorem}
\newtheorem{lemma}{Lemma}
\usepackage{graphicx}
\usepackage{titletoc}
\usepackage{enumitem}

\newtheoremstyle{assumptionstyle} 
  {6pt}   
  {6pt}   
  {\normalfont}  
  {}      
  {\bfseries} 
  {.}     
  {0.5em} 
  {\thmname{#1}\thmnumber{ #2}\thmnote{ \textbf{(#3)}}}      

\newtheorem{assumption}{Assumption}
\newtheorem{corollary}{Corollary}
\theoremstyle{remark}
\newtheorem{remark}{Remark}
\newcommand{\Tr}{\text{Tr}}
\newcommand{\ba}{\mathbf{a}}
\newcommand{\bb}{\mathbf{b}}

\newcommand{\bv}{\mathbf{v}}
\newcommand{\bq}{\mathbf{q}}
\newcommand{\bs}{\mathbf{s}}
\newcommand{\bz}{\mathbf{z}}
\newcommand{\w}{\mathbf{w}}
\newcommand{\x}{\mathbf{x}}

\newcommand{\bxi}{\mathbf{\xi}}
\newcommand{\X}{\mathbf{X}}
\newcommand{\y}{\mathbf{y}}
\newcommand{\bu}{\mathbf{u}}
\newcommand{\T}{\top}
\newcommand{\U}{\mathbf{U}}
\newcommand{\E}{\mathbf{E}}

\newcommand{\I}{\mathbf{I}}
\newcommand{\V}{\mathbf{V}}

\newcommand{\Ss}{\mathbf{S}}
\newcommand{\Z}{\mathbf{Z}}
\newcommand{\A}{\mathbf{A}}
\newcommand{\B}{\mathbf{B}}

\newcommand{\Pm}{\mathcal{P}_m}
\newcommand{\Pperp}{\mathcal{P}_\perp}
\newcommand{\Xm}{\mathbf{X}_m}
\newcommand{\Xperp}{\mathbf{X}_\perp}

\newcommand{\bSigma}{\mathbf{\Sigma}}
\newcommand{\wstar}{{\mathbf{w^\star}}}
\newcommand{\what}{{\mathbf{\hat{w}}}}

\newcommand{\qtilde}{\tilde{\mathbf{q}}}
\newcommand{\lamtilde}{\tilde{\lambda}}

\title{Optimal Representation Size: High-Dimensional Analysis of Pretraining and Linear Probing}

\author{ Valentina Njaradi \\
 	Gatsby Unit\\
 	University College London\\
 	\texttt{valentina.njaradi.23@ucl.ac.uk} \\
	\And
	Clémentine Dominé  \\
	Harvard University \\
    Institute of Science and Technology Austria 
    \And 
    Rachel Swanson \\
	 Sainsbury Wellcome Centre\\
	 University College London\\
     \And
    Marco Mondelli \thanks{Equal contribution.} \\
	Institute of Science and Technology Austria \\
    \And
    Andrew Saxe \footnotemark[1] \\
	 Gatsby Unit \& Sainsbury Wellcome Centre\\
	 University College London \\
}

\begin{document}

\maketitle

\begin{abstract} 
Data-efficient generalisation is a central objective in both artificial and biological learning. A common strategy is extracting reusable structure from abundant unlabelled data, enabling efficient adaptation to new tasks using limited supervision. This two-stage paradigm underlies modern training pipelines, in which pretraining is followed by fine-tuning or linear probing. We provide an analytical model of this process, formalising structure extraction as principal component analysis on unlabelled data and downstream learning as linear regression on a separate labelled dataset. In the high-dimensional regime, we derive exact expressions for training and generalisation errors, characterising their dependence on representation dimensionality, unlabelled and labelled sample sizes, and task alignment. Our results show that pretrained representations strongly influence downstream generalisation, and we characterize the optimal representation size across data regimes. With abundant pretraining but scarce downstream data, maximally compressed representations are optimal, whereas with limited pretraining data, higher-dimensional representations generalise better. Furthermore, we establish an exact trade-off between pretraining and supervision, quantifying how much unlabelled data is required to replace a single labelled sample. Beyond our idealised model, we observe similar phenomenology in autoencoders and pretrained LLMs. Altogether, we highlight that optimising representation size is critical for downstream generalisation and identify when compression during pretraining is beneficial.
\end{abstract}

\section{Introduction}
\label{sec:intro}
Large language models, vision transformers, and related architectures have demonstrated language, reasoning and mathematical capabilities widely considered out of reach just a decade ago \cite{openai2024GPT4, grattafiori_llama_2024, wang_glue_2019, hendrycks2021Measuring}. This progress has been made possible by multi-stage training pipelines capable of leveraging massive amounts of unlabelled data:
models are first pretrained on a large corpus of unlabelled data, then fine-tuned on labelled tasks and finally refined, e.g., via reinforcement learning from human feedback \cite{christiano2023Deep, ouyang2022Training}, low-rank adaptation \cite{hu2022LoRA} or test-time training \cite{sun2020test,liu2021ttt++,akyurek2025surprising}. Such a paradigm unlocks remarkable sample efficiency, as large pretrained models generalise on new tasks where they have access only to thousands (or even hundreds) of samples \cite{aghajanyan2021Intrinsic, howard-ruder-2018-universal}.

At their core, these training pipelines rely on a key fact: pretraining on large datasets yields rich representations that effectively transfer to downstream tasks. This idea is not new, as greedy layer-wise training and stacked autoencoders were designed to learn reusable representations in pretraining a decade before modern transformers  \cite{bengio2007Greedy, erhanWhy, vincentStacked}. However, different pretraining objectives emphasize different aspects of the data, and it remains unclear which features of a representation are most useful for downstream learning. 

Characterizing these representations, how they support computation, transfer to new tasks and enable adaptive behaviour in artificial and biological systems, has been a longstanding focus in machine learning, as well as cognitive science and neuroscience \cite{bengio2014representationlearningreviewnew,martin2007representation,fusi_why_2016}. In machine learning, structured representations and prior knowledge can accelerate learning and enhance generalisation, but may also hinder adaptation, as exemplified by catastrophic interference in continual and reversal learning \cite{parisi2019continual,kirkpatrick2017overcoming,zenke2017continual, braun2022exact}. Similar trade-offs arise in biological systems, where representations must support flexible and efficient adaptation \cite{badre_dimensionality_2021}. A crucial property underlying these effects is the size (or dimensionality) of representations. Both high and low-dimensional representations have been observed across a range of brain regions and tasks \cite{boyle2024Tuned, courellis2024Abstract, mishchanchuk2024hidden, nogueira2023Geometry} as well as in machine learning models  \cite{aghajanyan2021Intrinsic, valeriani2023Geometry}, yet a principled account of how dimensionality interacts with task and environment structure to shape learning performance and sample efficiency remains missing. 

Here, we study a two-stage learning setting in which representations are first learned via principal component analysis (PCA) from unlabelled samples drawn from a spiked covariance model, and then reused by linear regression on an independent labelled dataset. The covariance structure plays a central role: it produces both a tractable model for analysis and controlled alignment between the pretraining and supervised training stages, thus enabling us to isolate how pretrained representations influence downstream generalisation. More precisely, our contributions are as follows:
\begin{itemize}[leftmargin=*]
    \item We derive exact asymptotic expressions for estimation, generalisation and training error with distinct pretraining and downstream spikes (Theorem~\ref{thm:estimation_error},~\ref{thm:generalisation_error} and ~\ref{thm:training_error}, respectively). From this, we establish the improvement in generalisation coming from tuning the representation size, as well as the improvement due to adding more pretraining and finetuning data. For optimally-tuned errors, we identify regions where additional unlabelled samples reduce generalisation error more than additional labelled samples (Figure~\ref{fig:fig1}).
    \item We characterize when either \emph{(i)} low-dimensional representations are useful for generalisation on downstream tasks, or \emph{(ii)} tasks benefit from detailed 
    high-dimensional representations (Figure~\ref{fig:fig2}). In particular, we identify phase transitions where compression becomes useful (Corollaries~\ref{cor:sharp_transition},~\ref{cor:smooth_transition}), and describe their behaviour for an infinite number of pretraining samples (Corollary~\ref{cor:phases_infinite_unlabelled_limit}).
    \item We finally show that the theoretical predictions derived in our idealized setting extend to trained autoencoders when the number of samples is larger than the number of features, and further demonstrate that the insights unveiled by our theory persist in more realistic settings through linear probing of pretrained transformers (Figure~\ref{fig:fig3}).
\end{itemize}

\section{Related work}

\textbf{High-dimensional regression.} High-dimensional regression, in the regime where number of features and samples grow proportionally, has been the subject of intense investigation, leading to the precise characterization of the test error \cite{dobriban2015HighDimensional,hastie2020surprises,richardsAsymptotics,wu_optimal_2020} and phenomena such as benign overfitting \citep{Bartlett_2020} and double descent \citep{belkin2019reconciling, advani2020Highdimensional}.
This line of work further characterizes the distribution of the empirical risk minimiser \cite{han2023distribution}, the impact of spurious correlations \cite{bombarispurious}, performative learning \cite{cyffers2025optimal}, as well as training on multiple or surrogate data sources \cite{rezaei2025high,ildiz2024highdimensional,kolossovtowards,jain2024scaling}. 
Principal component regression (PCR), where regression is performed on principal components (PCs) of the sample covariance, provides a natural way to restrict regression to a low-dimensional subspace. Its performance depends critically on the number of retained components, which governs the bias–variance trade-off \cite{jolliffe_note_1982,xu2019Number,green_high-dimensional_2025,massyPrincipala}, and has been characterized both non-asymptotically \cite{dhillon2013Risk} and in proportional asymptotics for population and sample PCs \cite{yao2015Large,xu2019Number,green_high-dimensional_2025}. In contrast, we study pretrained regression (PR), with PCs formed during a pretraining phase and reused for a downstream task, separating the sample sets for PCA and subsequent regression. As a result, both the asymptotic risk and the optimal number of principal components for PR differ fundamentally from the classical PCR setting.

\textbf{Pretraining and transfer learning.}
More broadly, the reuse of pretrained representations across tasks has been extensively studied in transfer learning, primarily in supervised settings where generalisation depends on task similarity or alignment with pretrained features \cite{pan2010survey, zhuang2020comprehensive, gerace2022probing, lampinen2018analytic, tahir2025features,yang2020precise, song2024generalization}. These approaches typically do not model representations learned from unlabelled data. Complementary work establishes benefits of unsupervised and self-supervised pretraining \cite{lee2021predicting, azar2024semisupervised, arora2019theoretical, zhang-hashimoto-2021-inductive}. Closest to our paper, recent work analyses PCA-based pretraining in spiked models \cite{jonesmccormick2025provable}, but  considers fixed representations without addressing the role of representation dimensionality in generalisation. 

\textbf{Representation dimensionality.} 
A growing body of work highlights that pretrained representations often lie in low-dimensional subspaces, as evidenced by studies of intrinsic dimensionality \cite{aghajanyan2021Intrinsic, valeriani2023Geometry}, low-rank adaptation \cite{hu2022LoRA}, and linear probing \cite{alain2018Understanding}. Recent theoretical work has begun to formalise aspects of this structure \cite{wang2026Mathematical, anguita2026theory}. While these papers study the structure of pretrained features and how they support downstream prediction, they do not provide a closed-form characterisation of how the optimal retained dimension depends jointly on pretraining, downstream data, and task structure. Addressing this question through an analytically tractable theory is the focus of our work.

\textbf{Spiked covariance models.} 
The spiked covariance model \cite{johnstone2001Distribution} provides a tractable framework for studying low-dimensional structure in high-dimensional data. A key feature is the Baik–Ben Arous–Péché (BBP) phase transition \cite{baik2005phase,marcenko_distribution_1967,bai_spectral_2010}: above this threshold, empirical principal components align with the population spike, while below it they are asymptotically uninformative \cite{baik2005phase,baik_eigenvalues_2004,paulASYMPTOTICS,bai2012Sample}. This transition makes the spiked models well-suited for analysing how representation quality depends on sample size and signal strength. Accordingly, they have been used to analyse kernel methods, neural networks and feature learning \cite{ba2023learning,ghorbani2020neural,mousavi2023gradient}, as well as the emergence of low-rank structure during training \cite{ba2022high,cui2024asymptotics,dandi2025random,sonthalia2025lowrank}. Finally, spiked covariances have been
used to analyse PCR, showing how the projection onto empirical PCs has a regularising effect \cite{gedon2024nodoubledescent}. Our setting differs from these works: the signal subspace is estimated from an independent unlabelled pretraining sample and then used for a separate labelled task, allowing us to characterise the optimal representation dimension. 

\section{Setup} \label{sec:setup}

We study two-stage learning with unlabelled pretraining followed by supervised learning:
\begin{enumerate}[leftmargin=*]
    \item In the pretraining stage, we observe inputs $\X_u \in \mathbb R^{n_u \times p}$ whose rows are i.i.d.\ $\mathcal{N}(\mathbf 0, \bSigma_u)$, forming the \emph{unlabelled} dataset $\mathcal D_u$. From this data, the model extracts a representation via PCA: it computes the top $m$ principal components of $\X_u$, stacks them into $\U_m \in \mathbb R^{p \times m}$, and defines a projection of retained directions $\Pm = \U_m \U_m^\top$ with a discarded complement $\Pperp = I - \Pm = \U_\perp \U_\perp^\T$. If $m > n_u$, the PCA subspace is completed with uniformly random orthonormal vectors.
    \item In the downstream stage, we observe a \emph{labelled} dataset $\mathcal D_l =\{\X_l, \y_l\}$, where $\y_l = \X_l \wstar + \bxi\in \mathbb R^{n_l}$, the rows of $\X_l \in \mathbb R^{n_l \times p}$ are i.i.d.\ $\mathcal{N} (\mathbf 0, \bSigma_l)$ and independent of $\X_u$, $\wstar \in \mathbb{R}^{p}$ is the ground-truth signal and $\bxi \in \mathbb{R}^{n_l}$ is Gaussian noise with i.i.d. entries of variance $\sigma^2$.
\end{enumerate}

The pretrained regression model uses the representation learned from $\mathcal D_u$ to fit the downstream task by regressing on the projected inputs (see Figure~\ref{fig:fig0}\textbf{(a)} for a schematic representation):
\begin{equation}\label{eq:weight_estimate}
     \what = (\X_l \Pm )^{\dagger} \y_l  =  (\X_l \Pm)^{\dagger} \X_l \wstar +  (\X_l \Pm)^{\dagger} \bxi,
\end{equation}
where $(\cdot)^\dagger$ denotes the pseudoinverse. Unlike classical PCR~\cite{green_high-dimensional_2025}, $\Pm$ is learned from independent unlabelled data, so $(\X_l\Pm)^\dagger\X_l\neq\Pm$ in general, yielding distinct dependence on pretraining and downstream sample sizes (see Appendix~\ref{app:pcr_comparison} for a longer discussion).

The quantities of interest are estimation, generalisation and training error:
\begin{align*}
E^\mathrm{est} &= \mathbb{E}_{\bxi}\!\left[\|\wstar - \what\|^2\right], \quad 
E^\mathrm{gen} = \mathbb{E}_{\x, y, \bxi}\!\left[(y - \x^\top\what)^2 \right], \quad 
E^\mathrm{train} = \frac{1}{n_l}\mathbb{E}_{\bxi}\!\left[\|\y_l - \X_l\what\|^2 \right]
\end{align*}
where the test point $(\x, y)$ in $E^\mathrm{gen}$ is drawn from the labelled dataset $\mathcal{N}(0, \bSigma_l)$. 
We work in the high-dimensional regime where $p/n_l \to \gamma_l \in (0,\infty)$, $
p /n_u \to \gamma_u \in (0,\infty)$, $m /p \to \alpha \in [0,1]$ with $\alpha$ the fraction of dimensions to keep, and $\gamma_u, \gamma_l$ the aspect ratios for the datasets $\mathcal D_u, \mathcal D_l$.
Due to the projection of the downstream inputs, we define the effective aspect ratio $\gamma_{\mathrm{eff}} := \frac{m}{n_l}\to\alpha \gamma_l$ and assume $\gamma_\mathrm{eff} \neq 1$. The same analysis also extends to ridge-regularised downstream regression; the formulation and results are given in Appendix~\ref{app:ridge_extension}.

 The pretraining and downstream population covariances follow rank-one spiked models
\begin{equation}\label{eq:two_spike_covariances}
\bSigma_u=\mathbf I_p+(\lambda_u-1)\bv\bv^\top,
\qquad
\bSigma_l=\mathbf I_p+(\lambda_l-1)\bq\bq^\top,
\end{equation}
where $\lambda_u,\lambda_l>1$ and $\bv,\bq\in\mathbb R^p$ are deterministic unit vectors. Finally, we define the following overlaps:
\[
r := \bv^\top \bq \qquad r_v := \wstar^\top \bv / \|\wstar\| \qquad r_q := \wstar^\top \bq / \|\wstar\|
\]

\begin{figure}[t]
\begin{center}
\includegraphics[width=1\textwidth]{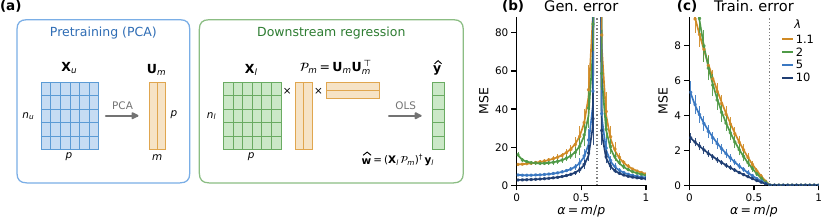}
\end{center}
\vspace{-1em}
\caption{\textbf{Analytically tractable model of two-stage learning}. \textbf{(a)} In the pretraining stage, PCA extracts the top-$m$ principal components from $n_u$ unlabelled samples. The downstream task with $n_l$ labelled samples is learned via regression on inputs projected through $\Pm = \U_m \U_m^\top$. Theoretically derived generalisation \textbf{(b)} and training errors \textbf{(c)} match numerical simulations. Solid lines show the theory; dots and error bars show the mean $\pm$ one standard deviation over 50 simulations. Colours indicate the spike strength $\lambda$; the dotted line marks the interpolation threshold $m = n_l$. }

\label{fig:fig0}
\vspace{-1em}
\end{figure}

\section{Theoretical analysis of two-stage learning} \label{sec:Theory}

In this section, we derive asymptotic expressions for the estimation, generalisation, and training errors of the estimator in~\eqref{eq:weight_estimate} and find the representation dimensionality $\alpha$ that minimizes the errors. We then characterise when the optimal $\alpha$ shifts from retaining all PCs to retaining only a few PCs, as a function of $n_u$ and $n_l$. We further study how $n_u$, $n_l$, $\lambda$ and the signal-to-noise ratio $\mathrm{SNR} = \frac{\|\mathbf{w}^*\|^2}{\sigma^2}$ 
affect both the errors and the optimal $\alpha$, as well as the limiting behaviour with infinite pretraining data ($\gamma_u=0$). Proofs of all results are deferred to Appendix~\ref{app:theory}.

\subsection{Asymptotic expressions}
The following deterministic limits describe how the pretraining projections $\Pm$ and $\Pperp$ act on fixed directions $\ba,\bb$. For the retained and discarded pretraining subspaces,
\begin{equation}
\frac{\ba^\top\Pperp\bb}{|\ba||\bb|}
\xrightarrow{a.s.}\bar p_{\perp,\ba,\bb},
\qquad
\frac{\ba^\top\Pm\bb}{|\ba||\bb|}
\xrightarrow{a.s.}\bar p_{\ba,\bb}.
\label{eq:main_pretraining_projections}
\end{equation}
These limits are obtained by expanding the projections over the sample PCs, e.g.,
$\ba^\top\Pperp\bb=\sum_{i=m+1}^p(\ba^\top\bu_i)(\bb^\top\bu_i)$,
and relating sample-PC overlaps to their population counterparts using random-matrix theory (Appendix~\ref{app:projections}).

To characterise the additional projection induced by finite downstream data, define
\[
\Pi:=(\X_l\Pm)^\dagger(\X_l\Pm),
\]
the projector onto the row space of the projected downstream dataset. When $\gamma_{\mathrm{eff}}<1$, $\X_l\U_m$ has full column rank and $\Pi=\Pm$, so all retained directions are identifiable. When $\gamma_{\mathrm{eff}}>1$, its nullspace removes part of the retained subspace, with the reduction depending on $\gamma_{\mathrm{eff}}$ and the effective downstream spike
$\bar\lambda_l:=1+(\lambda_l-1)\bar p_{\bq,\bq}$:
\begin{equation}
\frac{\ba^\top\Pi\bb}{|\ba||\bb|}
\xrightarrow{a.s.}\bar\pi_{\ba,\bb}
=
\begin{cases}
\bar p_{\ba,\bb},
& \gamma_{\mathrm{eff}}<1,\\[5pt]
\displaystyle
\frac{\bar p_{\ba,\bb}}{\gamma_{\mathrm{eff}}}
+
\frac{\bar p_{\ba,\bq}\bar p_{\bb,\bq}}
{\bar p_{\bq,\bq}}
\left(
\frac{\bar\lambda_l}{\gamma_{\mathrm{eff}}-1+\bar\lambda_l}
-\frac{1}{\gamma_{\mathrm{eff}}}
\right),
& \gamma_{\mathrm{eff}}>1.
\end{cases}
\label{eq:main_effective_projection}
\end{equation}
These quantities determine the asymptotic errors below.
\begin{theorem}[Asymptotic estimation error]\label{thm:estimation_error}
As $n_l, n_u, p, m \to \infty$ with $p/n_l \to \gamma_l$, $p/n_u \to \gamma_u$, $m/p \to \alpha$ and $\gamma_\mathrm{eff} \neq 1$, the estimation error converges a.s.\ to
\begin{equation}\label{eq:esterr}    
E^{\mathrm{est}}_\infty = \bar w^{\ast}\bigg(\underbrace{1 - \bar\pi_{\wstar, \wstar}}_{\text{missing signal}} + \underbrace{\frac{(\lambda_l-1)^2}{\bar\lambda_l^2}\bar{p}_{\perp,\wstar,\bq}^2\,\bar\pi_{\bq, \bq}}_{\text{leaked signal}}\bigg) + \underbrace{\frac{\min\{\gamma_\mathrm{eff},1\}}{|\gamma_\mathrm{eff}-1|}\bigl(\bar w^{\ast}\bar\sigma^2_\mathrm{eff} + \sigma^2\bigr)}_{\text{variance}},
\end{equation}
where $ \bar w^{\ast}$ is the deterministic limit of $\|\wstar\|^2$, $\bar{\lambda}_l = 1 + (\lambda_l-1)\bar{p}_{\bq, \bq}$ is the effective spike eigenvalue, and the variance from the unrecoverable part of $\wstar$ is \(
\bar \sigma^2_{\mathrm{eff}} = \bar p_{\perp, \wstar, \wstar} + \frac{\lambda_l-1}{\bar\lambda_l}\bar p_{\perp, \wstar, \bq}^2.
\) 
\end{theorem}

The expression in \eqref{eq:esterr} admits a simple geometric interpretation: the estimator observes $\wstar$ through the subspace induced by $\X_l \Pm$, and the error reflects how $\wstar$ decomposes relative to this subspace. By decomposing into two components, aligned or orthogonal to $\U_m$, $\wstar = \wstar_\parallel + \wstar_\perp$, the error splits into three contributions: \emph{(a)} the \textit{missing signal}, \emph{(b)} the \textit{leaked signal}, and \emph{(c)} the \textit{variance} term. 

\begin{itemize}[leftmargin=2em]
    \item[\emph{(a)}] The term $1 - \bar\pi_{\wstar,\wstar}$ captures the \emph{missing signal}, i.e. directions that lie in the retained subspace but cannot be reliably estimated from finite samples. 
    \item[\emph{(b)}]A distinct contribution comes from the \emph{leak}. In an isotropic model, the components $\X_l\Pm$ and $\X_l \Pperp$ are independent, so $\wstar_\perp$ is completely invisible to the estimator. Here, the spike in the covariance couples these two parts of the data, so directions outside the retained subspace become correlated with it. As a result, part of $\wstar_\perp$ leaks into $\Pi$, where it appears as a spurious signal aligned with $\Pi \bq$, producing the bias term \( \frac{(\lambda_l-1)^2}{\bar\lambda_l^2}\bar p_{\perp,\wstar,\bq}^2\,\bar\pi_{\bq,\bq} \). For a detailed discussion of parameter effects on the leak and a breakdown of different terms in the estimation error, see Figure \ref{fig:app_fig_estimations} in Appendix~\ref{app:Theory-Experiment Validation}.
    \item[\emph{(c)}] The remaining part of $\wstar_\perp$ does not couple back into the learned subspace. 
    Together with label noise, this contributes to the \emph{variance}, which exhibits a double descent \cite{belkin2019reconciling} controlled by $\gamma_\mathrm{eff}=\alpha\gamma_l$. Thus, even without explicit regularisation in the regression step, the model mitigates double descent by controlling the representation size $\alpha$, see Figure \ref{fig:app_fig_estimations} in Appendix~\ref{app:Theory-Experiment Validation}.
\end{itemize}

\begin{figure}[t]
\begin{center}
\includegraphics[width=1\textwidth]{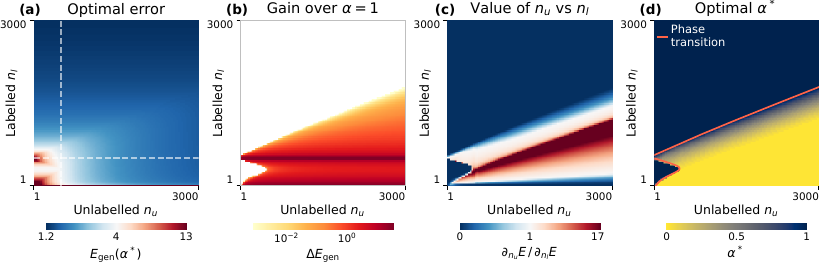}
\end{center}
\vspace{-1em}
\caption{\textbf{Benefits of optimizing representation size}. \textbf{(a)}~Minimum achievable generalisation error with optimal $\alpha^\star$. \textbf{(b)}~Gain in generalisation error ($\Delta E_\mathrm{gen} = E_\mathrm{gen}(1) - E_\mathrm{gen}(\alpha^*)$) from using $\alpha^\star$ instead of all PCs ($\alpha=1$). 
\textbf{(c)}~Marginal rate of substitution between unlabelled and labelled data,
$(\partial E^\mathrm{gen}_\infty/\partial n_u)\,/\,(\partial E^\mathrm{gen}_\infty/\partial n_l)$;
values above~1 (red) indicate that unlabelled data reduces the error more than labelled data, and vice-versa below~1 (blue). \textbf{(d)}~Generalisation-optimal representation size $\alpha^\star$.\setcounter{footnote}{0}\protect\footnotemark Simulation details are given in Appendix~\ref{app:implementation}.}
\label{fig:fig1}
\vspace{-1em}
\end{figure}
\begin{theorem}[Asymptotic generalisation error]\label{thm:generalisation_error}
Under the same regime, the generalisation error converges a.s. to
\[
E^\mathrm{gen}_\infty = E^\mathrm{est}_\infty + \bar w^{\ast} (\lambda_l - 1) \bigg ( \underbrace{r_q - \bar\pi_{\wstar,\bq}}_{\text{missing signal}} - \underbrace{\frac{\lambda_l-1}{\bar\lambda_l}\bar{p}_{\perp,\wstar,\bq}\,\bar\pi_{\bq,\bq}}_{\text{leaked signal}}\bigg) ^2 + \sigma^2.
\]
\end{theorem}
The calculation of the generalisation error builds on that of the estimation error, with an additional contribution measuring how the prediction error $\wstar - \what$ projects onto the spike $\bq$ in the test population covariance. The additional error depends on the mismatch between the true alignment $r_q$ and the alignment recovered by the estimator. As in the estimation error, this mismatch decomposes into a \emph{missing signal} term and a \emph{leaked signal} term, arising from finite-sample limitations and the coupling induced by the spike. The final $\sigma^2$ term reflects the irreducible error due to label noise. 
For additional discussions and a breakdown of different terms in the generalisation error, see Figure \ref{fig:app_fig_generlisation} in Appendix~\ref{app:Theory-Experiment Validation}. 

\begin{theorem}[Asymptotic training error]\label{thm:training_error}
Under the same regime, the training error converges a.s.\ to
\[
E^\mathrm{train}_\infty = \bigl(\bar w^{\ast}\bar\sigma^2_\mathrm{eff} + \sigma^2 \bigr)\max\!\left(0,\, 1 - \gamma_\mathrm{eff}\right).
\]
\end{theorem}
This expression has the standard interpolation form. When $\gamma_{\mathrm{eff}} > 1$, the model is overparameterised and perfectly fits the training data, yielding zero error. When $\gamma_{\mathrm{eff}} < 1$, the model cannot interpolate, and the error decreases linearly with $\gamma_{\mathrm{eff}}$, with slope set by the total residual variance $\bar w^{\ast}\bar\sigma^2_\mathrm{eff} + \sigma^2$.

\footnotetext{If $m$ does not grow with $p$ (e.g., $m=1$), then $m/p\to\alpha=0$. Thus, in simulations, $\alpha\approx 0$ refers to taking $m$ not growing with $p$ (e.g., retaining only the top principal component $m=1$).} 
Altogether, these results provide exact expressions for generalisation, estimation, and training errors, predicting  several key trends. In Figure~\ref{fig:fig0}\textbf{(b-c)}, we illustrate these errors as functions of the compression ratio $\alpha$ under varying data structures $\lambda$. The plots demonstrate an excellent agreement between theoretical predictions and experimental results for moderate values of $n_u, n_l, p$. Increasing $\lambda$ reduces the bias from the missing signal, the leak and also the variance. However, it does not reduce the errors across $\alpha$ uniformly, as errors with low $\alpha$ benefit more from stronger structure than errors with high $\alpha$. We validate our theory across a range of parameter settings, with experiments exploring variations in the signal-to-noise ratio, the alignments $r, r_v, r_q$ and the number of unlabelled and labelled samples $n_u, n_l$ (see Figure \ref{fig:gen_error_full} in Appendix \ref{app:Theory-Experiment Validation}).

Importantly, optimising the representation dimensionality reduces generalisation error across a broad region of parameter space relative to full-rank representations ($\alpha= 1$, Figure~\ref{fig:fig1}\textbf{(b)}), with the optimum  shown in Figure~\ref{fig:fig1}\textbf{(a)}. These gains are particularly pronounced in regimes with substantial pretraining data and limited downstream samples, where standard PCR or regular regression fail to achieve comparable performance even in misaligned task settings (see Figure~\ref{fig:compare_regressions} in Appendix~\ref{app:compare_regressions} for a comparison).

Next, we study the trade-off between pretraining and supervision at the optimal $\alpha^\star$ by quantifying how much error reduction is obtained from additional unlabelled versus labelled samples (Figure~\ref{fig:fig1}\textbf{(c)}). The resulting phase map reveals a nontrivial dependence on both data sources. In particular, there exists a broad interior region of the phase diagram—spanning moderate to large $n_u$ and increasing $n_l$—in which additional pretraining data yields larger reductions in generalisation error than labelled samples. 
Interestingly, this regime is closely linked to the optimal representation size $\alpha^\star$ that minimises generalisation error (see Figure~\ref{fig:fig1}\textbf{(d)}), which is discussed below.

\subsection{Optimal representation size}
\label{sec:optimal}

\begin{figure}[t]
\begin{center}
\includegraphics[width=1\textwidth]{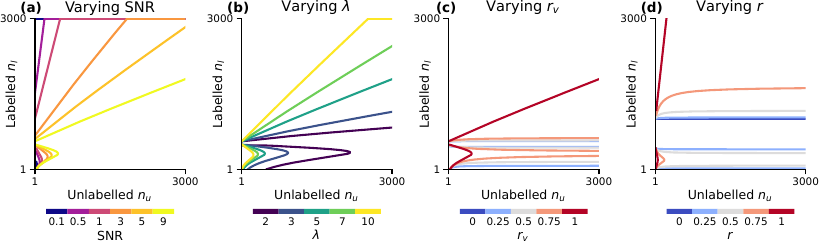}
\end{center}
\vspace{-1em}
\caption{\textbf{When is compression useful}.
\textbf{Phase boundaries in the $(n_u, n_l)$ plane for (a)} varying SNR \textbf{(b)} varying $\lambda$ \textbf{(c)} varying $r_v$ \textbf{(d)} varying $r$. Non-varying parameters are kept fixed, see simulation details are given in Appendix~\ref{app:implementation}.
} 
\label{fig:fig2}
\vspace{-1em}
\end{figure}
Generalisation in two-stage learning can be hindered by many factors, including misalignment between the pretraining and downstream datasets, weak signal, task difficulty, and limited sample sizes (Theorem~\ref{thm:generalisation_error}). Yet representation dimensionality provides a controllable means of improving performance (Figures~\ref{fig:fig0} and~\ref{fig:gen_error_full}): choosing the representation size appropriately can substantially mitigate these sources of degradation. We therefore characterise the optimal retained fraction $\alpha=m/p$ as a function of $n_u$, $n_l$, the SNR, the spike strength $\lambda$, and the alignment of the task $\wstar$ with the salient data directions $\bv$ and $\bq$.

\textbf{Generalisation error.}
We find that the optimal level of compression minimising generalisation error adapts to the availability of unlabelled and labelled data. For fixed SNR, $\lambda$ and an aligned task $\wstar \parallel \bv$, Figure~\ref{fig:fig1}\textbf{(d)} reveals the following three regimes.
\begin{itemize}[leftmargin=2em]
    \item[\emph{(i)}] Scarce \(n_l\) and \(n_u\): not enough pretraining samples to reliably estimate PCs, thus downstream tasks should not be compressed ($\alpha^\star \approx 1$).
    \item[\emph{(ii)}] Large \(n_u\), low \(n_l\): maintain a compact representation subspace (\(\alpha^\star < 1\)).
    \item[\emph{(iii)}] High $n_l$: little compression needed 
    ($\alpha^\star \approx 1$), as the number of task samples is enough to learn. 
\end{itemize}
We argue that regime~\emph{(ii)} is the most representative of modern machine-learning pipelines, where pretraining data is abundant but downstream task-specific data are limited. In practice, modern models are typically pretrained on large-scale datasets~\cite{grattafiori_llama_2024,biderman2023pythia} and then fine-tuned on substantially smaller task-specific datasets~\cite{wang_glue_2019}. Altogether, we observe two phase transitions: one between regimes \emph{(i)} and \emph{(ii)}, and another between regimes \emph{(ii)} and \emph{(iii)}. We discuss them below in detail.

\textbf{Phase transitions for generalisation error.}
We derive approximate analytical phase transition lines as a function of signal-to-noise ratio (SNR), pretraining and downstream spike strengths ($\lambda_u,\lambda_l$), task and spike alignments and sample allocation between pretraining and downstream tasks ($n_u$ and $n_l$). These approximations characterise the boundaries separating regimes in which different representation dimensions ($\alpha^\star$) are optimal, and they provide a tractable description of the underlying transitions observed in simulations.

First we consider the transition between regions \emph{(i)} and \emph{(ii)} when $\gamma_l>1$. In this regime, the generalisation error exhibits a double descent, with competing minima at large and small $\alpha$ (see Figure~\ref{fig:fig1}\textbf{(b)}). Although their exact locations can in principle be obtained from the closed-form error expressions, the resulting formulas are not analytically transparent. Instead, we provide an interpretable approximation by comparing the error at the two extremes $\alpha=1$ and $\alpha=0$, corresponding to retaining all components and only a fixed positive number of leading PCs.

\begin{corollary}[Phase transition for $\gamma_l>1$]
\label{cor:sharp_transition}
Let $\gamma_l>1$, $\mathrm{SNR}\to S=\bar w^\ast/\sigma^2$. The overlap between the leading PC of the population and sample
 covariance of the dataset $\mathcal D_u$ has a deterministic limit
\[
c_{\gamma_u}
=
\mathbf 1_{\gamma_u<(\lambda_u-1)^2}
\frac{1-\gamma_u / (\lambda_u-1)^2}
     {1+\gamma_u / (\lambda_u-1)}.
\]
Then, the transition between the small-$\alpha$ and large-$\alpha$ minima is approximated by the endpoint condition $E^\mathrm{gen}_\infty(0) = E^\mathrm{gen}_\infty(1),$
which yields a parametric curve
\begin{align}
\label{eq:curve}
1-r_v^2\;c_{\gamma_u}
+\frac{(\lambda_l-1)(r_q-r_vr\;c_{\gamma_u})^2}
{1+(\lambda_l-1)r^2\;c_{\gamma_u}}
&=
(\gamma_l-1)\left[
\frac{1-r_q^2}{\gamma_l}
+
\frac{r_q^2\lambda_l\gamma_l}
{(\gamma_l-1+\lambda_l)^2}
\right]+\frac{1}{S(\gamma_l-1)}
\end{align}
\end{corollary}

When the competing minima lie close to the two endpoints, crossing the curve \eqref{eq:curve} approximately marks an abrupt switch in the optimal representation size $\alpha^\star$ between the small- and large-$\alpha$ regimes. The derivation is given in Appendix~\ref{app:sharp_transition}.

Corollary \ref{cor:smooth_transition} characterises the loss of local
stability of the full representation in the underparametrised regime
$\gamma_l<1$. When the generalisation error exhibits a single minimum,
as in the shared-spike examples
(see e.g. Figure~\ref{fig:gen_error_full}\textbf{(e,g,h)}
in Appendix \ref{app:errcurves}), this describes the transition between
regions \emph{(ii)} and \emph{(iii)}.
The point $\alpha=1$ is locally optimal as long as decreasing $\alpha$
increases the generalisation error. The boundary is determined by
the condition
$\left.\frac{\partial E^\mathrm{gen}_\infty}{\partial\alpha}
\right|_{\alpha=1^-}=0$.

\begin{corollary}[Phase transition for $\gamma_l<1$]
\label{cor:smooth_transition}
Let $\gamma_l<1$,
$\mathrm{SNR}\to S=\bar w^\ast/\sigma^2$, and denote by
$D_\bv(\gamma_u;\lambda_u)$ the $\alpha$-derivative of
$\bar p_{\bv,\bv}$ at $\alpha=1^-$. Then, the solution $\alpha=1$ is locally stable if and only if $\gamma_l < \gamma_l^{crit}$:
\[
\gamma_l^{\mathrm{crit}}
=
\frac{S\,[(1-r_v^2)
+
r_v^2\;D_\bv(\gamma_u;\lambda_u)]}{1+S\,[(1-r_v^2)
+
r_v^2\;D_\bv(\gamma_u;\lambda_u)]},
\qquad
D_\bv(\gamma_u;\lambda_u)
=
\begin{cases}
\displaystyle
\frac{\lambda_u\gamma_u}
{(\lambda_u-1+\sqrt{\gamma_u})^2},
& \gamma_u<1,\\[10pt]
\displaystyle
\frac{1}{\lambda_u},
\frac{\gamma_u}{\gamma_u-1+\lambda_u},
& \gamma_u>1.
\end{cases}
\]
\end{corollary}

Depending on the regime, the resulting boundary induces either a linear or nonlinear scaling of the critical number of labelled samples $n_l$ with the unlabelled sample size $n_u$. Below this threshold, the optimal representation shifts to $\alpha < 1$, and compression improves downstream generalisation. The derivation of this result is in Appendix~\ref{app:smooth_transition}. 

Of particular interest is the limit $\gamma_u \to 0$, corresponding to infinite unlabelled data—a regime reminiscent of modern pretraining datasets \cite{grattafiori_llama_2024, openai2024GPT4, biderman2022datasheet}. Here, errors and phase boundaries are independent of $\gamma_u$ and have simple dependencies on $r_v$,, $\mathrm{SNR}$, $\lambda$, and $r$. Precise expressions are 
in Appendix~\ref{app:theory_inf_limits} (Corollary~\ref{cor:infinite_unlabelled_limit}), where we also consider the limit of infinite labelled data, $\gamma_l \to 0$ (Corollary~\ref{cor:phases_infinite_unlabelled_limit}). In this last regime, the generalisation error simplifies (see Eq.~\eqref{eq:inf_downstream_gen}), the double-descent behaviour disappears, and the effective projector $(\X \Pm)^\dagger \X \Pm$ reduces to the population projector onto $\Pm$.

The full phase transition curve, derived in Corollaries \ref{cor:sharp_transition} and \ref{cor:smooth_transition}, matches the phase transition observed in the heat map of Figure~\ref{fig:fig1}\textbf{(d)} (shown in white; see also additional heatmaps in Figures \ref{fig:app_fig_alpha_SNR} and \ref{fig:app_fig_alpha_lambda} in Appendix \ref{app:Optimal}).
Increasing the SNR progressively shrinks the region where compression is beneficial (Figure~\ref{fig:fig2}\textbf{(a)}), making full-rank representations (\( \alpha = 1 \)) optimal over a wider range of sample sizes. The converse is true when increasing the spike strength $\lambda$, the task-spike alignment $r_v$ or the spikes-alignment $r$ (Figure~\ref{fig:fig2}\textbf{(b, c, d)}). 
These trends are evident from Corollary 
\ref{cor:smooth_transition}:
when $\gamma_l<1$, the stability boundary depends on $S(1-r_v^2)$, so increasing SNR or decreasing $r_v$ directly reduces $\gamma_l^\mathrm{crit}$. 
Additional heatmaps for $\alpha$  in Appendix~\ref{app:Optimal} (Figures \ref{fig:app_fig_alpha_SNR}, \ref{fig:app_fig_alpha_lambda}, \ref{fig:app_fig_r}) further validate the joint dependence on SNR, $\lambda$, $r_v$  and  $r$, as well as the match with the theoretically derived transitions. In particular, increasing SNR, and decreasing $\lambda$, $r_v$ and $r$ all act to suppress the benefits of compression: for $\gamma_l <1$ the phase boundary flattens, while for $\gamma_l >1$ it shifts along the $n_l$ axis, reducing the region where $\alpha < 1$ is optimal. Intuitively, a weaker spike is harder to estimate, lower alignment reduces the relevance of prior structure, and higher SNR makes the downstream task easy to learn from scratch. These effects diminish the value of compression, even when labelled data is abundant.

\textbf{Training error.}
Our analysis further characterises the optimal value of $\alpha$ for minimising the training error. In the oversampled regime ($n_l > p$, corresponding to $\gamma_l<1$), the optimal representation is full-rank, i.e., $\alpha = 1$.
Conversely, in the undersampled regime ($n_l < p$, corresponding to $\gamma_l>1$), the training error is equal to zero. In this case, the optimal dimensionality matches the sample size, i.e., $\alpha = n_l / p$, since increasing the dimension beyond this point provides no additional information. This is illustrated in Figures~\ref{fig:app_fig_alpha_SNR_2} and~\ref{fig:app_fig_alpha_lambda_2} in Appendix~\ref{app:Optimal}.

Altogether, our exact asymptotic analysis provides a theoretical framework to understand when low- or high-dimensional representations are optimal for downstream generalisation. We next validate such framework in practical settings.

\section{Validation in autoencoders and pretrained transformers}
\label{sec:experiements}

\begin{figure}[t]
\begin{center}
\includegraphics[width=\textwidth]{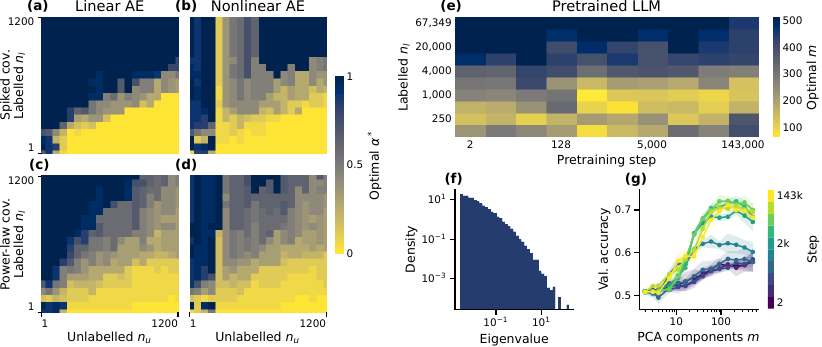}
\end{center}
\vspace{-1em}
\caption{\textbf{Optimal representation in autoencoders and pretrained LLMs}. Optimal representation size as a function of unlabelled ($n_u$) and labelled ($n_l$) sample sizes. Linear autoencoders are considered in panels \textbf{(a,c)} and nonlinear ones in panels \textbf{(b,d)}. 
Autoencoders are trained on Gaussian data with either spiked identity covariance $\bSigma= \mathbf I_p + \lambda \bv \bv^\top$ (panels \textbf{(a,b)}) or a power-law covariance $\bSigma = \mathrm{diag}(1^{-\beta}, 2^{-\beta}, \dots, p^{-\beta})$ with $\beta=0.3$ (panels \textbf{(c,d)}). \textbf{(e)} In a language transformer (\texttt{PYTHIA} models), PCA on representations also improves downstream linear probing on the \texttt{sst2} task from the \texttt{GLUE} benchmark.
\textbf{(f)} Eigenvalue spectrum of last-token representations from the final hidden layer.
\textbf{(g)} Validation accuracy versus representation size $m$ for a small downstream dataset ($n_l=2000$); colors denote different checkpoints and shaded regions indicate standard deviation across seeds. 
Simulation details are given in Appendix~\ref{app:implementation}.
}
\label{fig:fig3}
\vspace{-1em}
\end{figure}

In this section, we evaluate the extent to which theoretical predictions are valid beyond the idealised setting employed in their derivation. We first consider a controlled synthetic experiment in which the representation learned by an autoencoder approximates PCA. We then relax the assumptions of linearity and spiked isotropic covariance, allowing for a more general covariance structure. Finally, we show that our predictions extend to representations learned by modern neural networks.

\subsection{Autoencoders as learned representations}
\label{sec:experiements_1}
We train the autoencoder on a pretraining task and use the learned representations for regression on downstream tasks. This approach is motivated by the fact that, in several regimes, autoencoders with a bottleneck dimension $m$ recover the subspace spanned by the top-$m$ principal components of the pretraining data \cite{refinetti2022Dynamics, baldi1989Neural, plaut2018Principal, bourlard1988Autoassociation}. We discuss below the regimes under which this correspondence holds.

\textbf{Spiked isotropic covariance.} We first train autoencoders on the spiked identity covariance structures considered in the theory. For linear autoencoders, the optimal bottleneck size, defined as the value minimising the generalisation error on a downstream task, closely follows our theoretical results (compare Figure~\ref{fig:fig3}\textbf{(a)} with Figure~\ref{fig:fig1}\textbf{(d)}). We highlight that the boundaries of the phase transitions in the heatmaps precisely match the predictions from Section~\ref{sec:optimal}. Nonlinear autoencoders exhibit the same qualitative behaviour, except when the number of pretraining samples is small, see Figure~\ref{fig:fig3}\textbf{(b)}. This discrepancy is explained by a data-dependent transition: both linear and nonlinear autoencoders recover PCA only once sufficient pretraining data is available, see Figure~\ref{fig:app_autoencoders_overlaps} in Appendix \ref{app:autoencoder}. In the low-sample regime, autoencoders instead tend to memorise training samples \cite{fumero2026Navigating, jiang2020Associative, radhakrishnan2020Overparameterized}. Despite this, when $n_u < p$, higher-dimensional representations still yield lower errors on downstream tasks.

\textbf{General covariance structure.}
For a power law covariance, often appearing in representations in modern neural networks \cite{agrawal_$alpha$req_2022, li_tracing_2025}, the same qualitative behaviour persists. In fact, while phase boundaries shift, the overall structure is preserved: no compression when pretraining data is limited; small bottlenecks for large pretraining and small downstream datasets; and intermediate-sized bottlenecks for large pretraining datasets. This is illustrated in Figure~\ref{fig:fig3}\textbf{(c,d)}, and it is consistent with PCA-based regression under general covariance, see Figure~\ref{fig:app_autoencoders_larger_snr} in Appendix \ref{app:autoencoder}. Furthermore, when optimising for training error, the results closely match the spiked identity case (Figure~\ref{fig:app_autoencoders_larger_snr} in Appendix \ref{app:autoencoder}).
At lower signal-to-noise ratios, the regime where compression is optimal expands (Figure~\ref{fig:app_autoencoders_smaller_snr} in Appendix \ref{app:autoencoder}), in agreement with the theoretical analysis. Nonlinear autoencoders again exhibit deviations from PCA at low pretraining sample sizes.

Overall, our empirical results show that the theoretical predictions derived in the idealised setting extend to trained linear and nonlinear autoencoders when enough pretraining data is available, and that the qualitative phase structure persists under general covariance spectra.

\subsection{Linear probing transformers}
\label{sec:experiements_2}
We next examine whether representations learned by modern architectures exhibit similar dimensionality effects under linear probing.  We use models from the Pythia family \cite{biderman2023pythia}, which provide access to intermediate checkpoints throughout pretraining. As these models are trained for $\sim$1.5 epochs on the Pile dataset \cite{gao2020pile, biderman2022datasheet}, the training step serves as a proxy for the number of pretraining samples $n_u$. The learned representations exhibit an approximate power-law structure (Figure~\ref{fig:fig3}\textbf{(f)}). For each checkpoint, we project representations onto the top-$m$ principal components and train linear probes on a downstream sentiment classification task (SST-2 from GLUE \cite{wang_glue_2019}), selecting $m$ to optimise validation or training accuracy. We consider two regimes: one where PCA is applied to pretraining representations, and one where it is applied to representations extracted from the downstream task itself; the latter is explored in Appendix~\ref{app:experiments_real}.

The dependence on dataset size is consistent with theoretical predictions, see Figure~\ref{fig:fig3}\textbf{(e)}. For small downstream datasets, models that are pretrained longer benefit from stronger compression, see Figure~\ref{fig:fig3}\textbf{(g)}. For larger downstream datasets, performance improves monotonically with $m$, indicating that little or no compression is optimal (see the right-most plot of Figure~\ref{fig:app_transformers_val_accuracy} in Appendix \ref{app:experiments_real}). For extremely small downstream datasets ($\sim$100 samples), the optimal dimensionality is high again (see the left-most plot in Figure~\ref{fig:app_transformers_val_accuracy} in Appendix \ref{app:experiments_real}), consistent with the weak alignment between the task signal and dominant pretraining directions. Finally, tuning the representation dimensionality via PCA consistently matches or outperforms using the full feature space, see Figure~\ref{fig:app_transformers_compare} in Appendix \ref{app:experiments_real}.

We further validate these findings across a broader set of tasks and model scales in Appendix~\ref{sec:app_datasets}. Extending to nine GLUE datasets and two Pythia model sizes (70M and 410M parameters), we find qualitatively consistent behaviour with no systematic dependence on model scale. We also confirm the theoretical prediction that lower signal-to-noise ratio favours stronger compression: corrupting downstream training labels shifts the validation-optimal dimension downward across datasets and both model sizes (Figures~\ref{fig:all_datasets_downstream} and \ref{fig:all_datasets_pretrain}). In summary, these results show that real-world representations exhibit a theoretically predictable  qualitative phase structure when optimising dimensionality for downstream generalisation.

\section{Discussion}

We present a theoretical model quantifying how representation dimensionality affects training and generalisation error on downstream tasks. We demonstrate that \emph{(i)} tuning this dimensionality significantly improves generalisation compared to full representations or naive fixed dimensionality reduction, and that \emph{(ii)} the optimal dimensionality is determined by the amount of pretraining and downstream data, the task signal-to-noise ratio, the strength of structure in the inputs, and the alignment between task-relevant signals and dominant covariance directions.
In particular, when pretraining data is abundant and downstream data is limited, the optimal representations are maximally compressed. Furthermore, we unveil a trade-off between pretraining and supervision, identifying regimes where unlabelled data yields larger gains than labelled samples.
These results provide a principled explanation for why large models trained on extensive data and exhibiting low intrinsic dimensionality \cite{aghajanyan2021Intrinsic}, are especially effective for fine-tuning and linear probing. In Appendix \ref{app:finetuning}, we discuss how these findings could extend to the fine-tuning setting.

The question of optimal representations is central to machine learning as well as neuroscience, where internal representations support categorisation, reasoning, planning, navigation, and decision-making \cite{stokes2013dynamic,genovesio2005prefrontal,shima2000neuronal}. A key property of these representations is their dimensionality. High-dimensional representations may support flexible readout across tasks \cite{raposo2014category,tang2019effective,rigotti2013importance,bernardi2020geometry}, whereas low-dimensional representations are associated with abstraction, efficient coding, or structured prior knowledge \cite{tye2024selectivity,nogueira2023Geometry}. Our results provide a quantitative characterization of how task structure and data availability determine the optimal representation tradeoff in a simplified setting, offering predictions for how biological systems might organize representations for effective learning.

Several limitations motivate future work such as considering richer covariance models. Richer covariance models require matrix-valued rather than rank-one resolvent arguments. Another important direction is to replace the closed-form regression estimator with gradient-based training, enabling the study of optimisation-induced bias and its dependence on initialisation from pretraining. Finally, extending the framework beyond PCA-based projections to learned feature representations would allow for a closer connection to empirical observations in deep networks.

\newpage
\section*{Acknowledgments}
This research was funded in whole or in part by the Austrian Science Fund (FWF) (10.55776/COE12), Gatsby Charitable Foundation (GAT3850 and GAT4058), Sainsbury Wellcome Centre Core Grant from Wellcome (219627/Z/19/Z). For
the purpose of open access, the authors have applied a CC BY public copyright license to any Author Accepted Manuscript version arising from this submission. CD was supported by the Harvard Center for Brain Science (CBS)–NTT Fellowship Program on the Physics of Intelligence.

\section*{Contributions}
A.S. and R.S. initiated and conceptualized the project. V.N. and C.D. developed the theoretical analysis and derived the main results, with V.N. leading the derivations. V.N. and C.D. designed and performed the experiments. M.M. co-supervised the project and contributed to the technical development of the work, including proofs, experimental design, and interpretation of results. A.S. co-supervised the project and contributed to the conceptual framing of the project, interpretation of the theoretical results, experimental design and overall scientific direction. V.N. and C.D. wrote the manuscript. M.M. and A.S. reviewed and edited the paper. All authors discussed the results and commented on the manuscript.

\bibliography{bibliography}

@article{green_high-dimensional_2025,
  title = {The High-Dimensional Asymptotics of Principal Component Regression},
  author = {Green, Alden and Romanov, Elad},
  year = 2025,
  journal = {The Annals of Statistics},
  volume = {53},
}

@article{baik2005phase,
  title = {Phase Transition of the Largest Eigenvalue for Nonnull Complex Sample Covariance Matrices},
  author = {{Jinho Baik} and {G\'erard Ben Arous} and {Sandrine P\'ech\'e}},
  year = 2005,
  journal = {The Annals of Probability},
  volume = {33},
}

@inproceedings{ba2022high,
  title={High-dimensional asymptotics of feature learning: How one gradient step improves the representation},
  author={Ba, Jimmy and Erdogdu, Murat A and Suzuki, Taiji and Wang, Zhichao and Wu, Denny and Yang, Greg},
  booktitle={Advances in Neural Information Processing Systems},
  volume={35},
  year={2022}
}

@inproceedings{cui2024asymptotics,
  title={Asymptotics of feature learning in two-layer networks after one gradient-step},
  author={Cui, Hugo and Pesce, Luca and Dandi, Yatin and Krzakala, Florent and Lu, Yue M and Zdeborov{\'a}, Lenka and Loureiro, Bruno},
  booktitle={International Conference on Machine Learning},
  year={2024}
}

@inproceedings{gedon2024nodoubledescent,
  title     = {No Double Descent in Principal Component Regression: A High-Dimensional Analysis},
  author    = {Gedon, Daniel and Ribeiro, Ant{\^o}nio H. and Sch{\"o}n, Thomas B.},
  booktitle = {International Conference on Machine Learning},
  year      = {2024}
}

@inproceedings{sonthalia2025lowrank,
  title     = {Low Rank Gradients and Where to Find Them},
  author    = {Sonthalia, Rishi and Murray, Michael and Mont{\'u}far, Guido F.},
  booktitle = {Advances in Neural Information Processing Systems},
  year      = {2025}
 }

@inproceedings{dandi2025random,
  title={A Random Matrix Theory Perspective on the Spectrum of Learned Features and Asymptotic Generalization Capabilities},
  author={Dandi, Yatin and Pesce, Luca and Cui, Hugo and Krzakala, Florent and Lu, Yue and Loureiro, Bruno},
  booktitle={International Conference on Artificial Intelligence and Statistics},
  year={2025},
  organization={PMLR}
}

@inproceedings{mousavi2023gradient,
  title={Gradient-based feature learning under structured data},
  author={Mousavi-Hosseini, Alireza and Wu, Denny and Suzuki, Taiji and Erdogdu, Murat A},
  booktitle={Advances in Neural Information Processing Systems},
  volume={36},
  year={2023}
}

@inproceedings{ghorbani2020neural,
  title={When do neural networks outperform kernel methods?},
  author={Ghorbani, Behrooz and Mei, Song and Misiakiewicz, Theodor and Montanari, Andrea},
  booktitle={Advances in Neural Information Processing Systems},
  volume={33},
  year={2020}
}

@inproceedings{ba2023learning,
  title={Learning in the presence of low-dimensional structure: a spiked random matrix perspective},
  author={Ba, Jimmy and Erdogdu, Murat A and Suzuki, Taiji and Wang, Zhichao and Wu, Denny},
  booktitle={Advances in Neural Information Processing Systems},
  volume={36},
  year={2023}
}

@inproceedings{akyurek2025surprising,
  title={The Surprising Effectiveness of Test-Time Training for Few-Shot Learning},
  author={Aky{\"u}rek, Ekin and Damani, Mehul and Zweiger, Adam and Qiu, Linlu and Guo, Han and Pari, Jyothish and Kim, Yoon and Andreas, Jacob},
  booktitle={International Conference on Machine Learning},
  year={2025},
  organization={PMLR}
}

@book{bai_spectral_2010,
    series = {Springer {Series} in {Statistics}},
    title = {Spectral {Analysis} of {Large} {Dimensional} {Random} {Matrices}},
    publisher = {Springer New York},
    author = {Bai, Zhidong and Silverstein, Jack W.},
    year = {2010},
}

@article{marcenko_distribution_1967,
    title = {{Distribution} {of} {eigenvalues} {for} {some} {sets} {of} {random} {matrices}},
    volume = {1},
    journal = {Mathematics of the USSR-Sbornik},
    author = {Marčenko, V. A. and Pastur, L. A.},
    year = {1967},
}

@book{Vershynin_2026,
author = {Vershynin, Roman},
title = {High-Dimensional Probability: An Introduction with Applications in Data Science},
publisher = {Cambridge University Press},
series = {Cambridge Series in Statistical and Probabilistic Mathematics},
year = {2026}
}

@article{rubio2011Spectral,
    title = {Spectral convergence for a general class of random matrices},
    volume = {81},
    journal = {Statistics \& Probability Letters},
    author = {Rubio, Francisco and Mestre, Xavier},
    year = {2011},
}

@article{massyPrincipala,
 author = {William F. Massy},
 journal = {Journal of the American Statistical Association},
 title = {Principal Components Regression in Exploratory Statistical Research},
 volume = {60},
 year = {1965}
}

@article{dobriban2015HighDimensional,
  title = {High-Dimensional Asymptotics of Prediction: {{Ridge}} Regression and Classification},
  author = {Dobriban, Edgar and Wager, Stefan},
  year = 2018,
  journal = {The Annals of Statistics},
}

@article{hastie2020Surprises,
  title = {Surprises in High-Dimensional Ridgeless Least Squares Interpolation},
  author = {{Trevor Hastie} and {Andrea Montanari} and {Saharon Rosset} and {Ryan J. Tibshirani}},
  year = 2022,
  journal = {The Annals of Statistics},
  volume = {50},
}

@inproceedings{richardsAsymptotics,
  title={Asymptotics of Ridge(less) Regression under General Source Condition},
  author={Dominic Richards and Jaouad Mourtada and Lorenzo Rosasco},
  booktitle={International Conference on Artificial Intelligence and Statistics},
  year={2020},
}

@inproceedings{wu_optimal_2020,
  title = {On the Optimal Weighted \textbackslash ell\_2 Regularization in Overparameterized Linear Regression},
  booktitle = {Advances in Neural Information Processing Systems},
  author = {Wu, Denny and Xu, Ji},
  year = 2020,
  volume = {33},
}

@article{jolliffe_note_1982,
    title = {A {Note} on the {Use} of {Principal} {Components} in {Regression}},
    volume = {31},
    journal = {Applied Statistics},
    author = {Jolliffe, Ian T.},
    year = {1982},
}

@inproceedings{xu2019Number,
  title = {On the Number of Variables to Use in Principal Component Regression},
  booktitle = {Advances in Neural Information Processing Systems},
  author = {Xu, Ji and Hsu, Daniel J},
  year = 2019,
  volume = {32},
}

@book{yao2015Large,
    series = {Cambridge series in statistical and probabilistic mathematics},
    title = {Large sample covariance matrices and high-dimensional data analysis},
    language = {en},
    publisher = {Cambridge university press},
    author = {Yao, Jianfeng and Bai, Zhidong and Zheng, Shui-Rong},
    year = {2015},
}

@article{johnstone2001Distribution,
    title = {On the distribution of the largest eigenvalue in principal components analysis},
    volume = {29},
    journal = {The Annals of Statistics},
    author = {Johnstone, Iain M.},
    year = {2001},
}

@article{baik_eigenvalues_2004,
  title={Eigenvalues of large sample covariance matrices of spiked population models},
  author={Baik, Jinho and Silverstein, Jack W},
  journal={Journal of multivariate analysis},
  volume={97},
  year={2006},
}

@article{paulASYMPTOTICS,
  title={Asymptotics of sample eigenstructure for a large dimensional spiked covariance model},
  author={Paul, Debashis},
  journal={Statistica Sinica},
  year={2007},
}

@inproceedings{howard-ruder-2018-universal,
    title = "Universal Language Model Fine-tuning for Text Classification",
    author = "Howard, Jeremy  and
      Ruder, Sebastian",
    booktitle = "Proceedings of the 56th Annual Meeting of the Association for Computational Linguistics (Volume 1: Long Papers)",
    year = "2018",
}

@article{bai2012Sample,
    title = {On sample eigenvalues in a generalized spiked population model},
    volume = {106},
    journal = {Journal of Multivariate Analysis},
    author = {Bai, Zhidong and Yao, Jianfeng},
    year = {2012},
}

@article{dhillon2013Risk,
  author  = {Paramveer S. Dhillon and Dean P.  Foster and Sham M.  Kakade and Lyle H. Ungar},
  title   = {A Risk Comparison of Ordinary Least Squares vs Ridge Regression},
  journal = {Journal of Machine Learning Research},
  year    = {2013},
  volume  = {14},
}

@article{grattafiori_llama_2024,
    title = {The {Llama} 3 {Herd} of {Models}},
    author = {Grattafiori, Aaron and others},
    year = {2024},
    journal = {arXiv:2407.21783 },
}

@inproceedings{sun2020test,
  title={Test-time training with self-supervision for generalization under distribution shifts},
  author={Sun, Yu and Wang, Xiaolong and Liu, Zhuang and Miller, John and Efros, Alexei and Hardt, Moritz},
  booktitle={International conference on machine learning},
  year={2020},
  organization={PMLR}
}

@inproceedings{liu2021ttt++,
  title={Ttt++: When does self-supervised test-time training fail or thrive?},
  author={Liu, Yuejiang and Kothari, Parth and Van Delft, Bastien and Bellot-Gurlet, Baptiste and Mordan, Taylor and Alahi, Alexandre},
  booktitle={Advances in Neural Information Processing Systems},
  volume={34},
  year={2021}
}

@inproceedings{wang_glue_2019,
    title = "{GLUE}: A Multi-Task Benchmark and Analysis Platform for Natural Language Understanding",
    author = "Wang, Alex  and
      Singh, Amanpreet  and
      Michael, Julian  and
      Hill, Felix  and
      Levy, Omer  and
      Bowman, Samuel",
    booktitle = "Proceedings of the 2018 {EMNLP} Workshop {B}lackbox{NLP}: Analyzing and Interpreting Neural Networks for {NLP}",
    year = "2018",
}

@article{Bartlett_2020,
   title={Benign overfitting in linear regression},
   volume={117},
   journal={Proceedings of the National Academy of Sciences},
   author={Bartlett, Peter L. and Long, Philip M. and Lugosi, Gábor and Tsigler, Alexander},
   year={2020},
}

@article{belkin2019reconciling,
  title={Reconciling modern machine-learning practice and the classical bias--variance trade-off},
  author={Belkin, Mikhail and Hsu, Daniel and Ma, Siyuan and Mandal, Soumik},
  journal={Proceedings of the National Academy of Sciences},
  volume={116},
  year={2019},
}

@inproceedings{biderman2023pythia,
  title={Pythia: A Suite for Analyzing Large Language Models Across Training and Scaling},
  author={Biderman, Stella and Schoelkopf, Hailey and Anthony, Quentin Gregory and Bradley, Herbie and O'Brien, Kyle and Hallahan, Eric and Khan, Mohammad Aflah and Purohit, Shivanshu and Prashanth, USVSN Sai and Raff, Edward and others},
  booktitle={International Conference on Machine Learning},
  year={2023}
}

@article{gao2020pile,
  title={The {P}ile: An 800{GB} dataset of diverse text for language modeling},
  author={Gao, Leo and Biderman, Stella and Black, Sid and Golding, Laurence and Hoppe, Travis and Foster, Charles and Phang, Jason and He, Horace and Thite, Anish and Nabeshima, Noa and others},
  journal={arXiv:2101.00027},
  year={2020}
}

@article{biderman2022datasheet,
  title={Datasheet for the pile},
  author={Biderman, Stella and Bicheno, Kieran and Gao, Leo},
  journal={arXiv:2201.07311},
  year={2022}
}

@article{bernardi2020geometry,
  title={The geometry of abstraction in the hippocampus and prefrontal cortex},
  author={Bernardi, Silvia and Benna, Marcus K and Rigotti, Mattia and Munuera, J{\'e}r{\^o}me and Fusi, Stefano and Salzman, C Daniel},
  journal={Cell},
  volume={183},
  year={2020},
}

@article{raposo2014category,
	title        = {A category-free neural population supports evolving demands during decision-making},
	author       = {Raposo, David and Kaufman, Matthew T and Churchland, Anne K},
	year         = 2014,
	journal      = {Nature neuroscience},
	volume       = 17,
}

@article{tang2019effective,
	title        = {Effective learning is accompanied by high-dimensional and efficient representations of neural activity},
	author       = {Tang, Evelyn and Mattar, Marcelo G and Giusti, Chad and Lydon-Staley, David M and Thompson-Schill, Sharon L and Bassett, Danielle S},
	year         = 2019,
	journal      = {Nature neuroscience},
	volume       = 22,
}

@article{rigotti2013importance,
	title        = {The importance of mixed selectivity in complex cognitive tasks},
	author       = {Rigotti, Mattia and Barak, Omri and Warden, Melissa R and Wang, Xiao-Jing and Daw, Nathaniel D and Miller, Earl K and Fusi, Stefano},
	year         = 2013,
	journal      = {Nature},
	volume       = 497,
}

@article{tye2024selectivity,
author = {Tye, Kay and Miller, Earl and Taschbach, Felix and Benna, Marcus and Rigotti, Mattia and Fusi, Stefano},
year = {2024},
title = {Mixed selectivity: Cellular computations for complexity},
volume = {112},
journal = {Neuron},
}

@article{stokes2013dynamic,
  title={Dynamic coding for cognitive control in prefrontal cortex},
  author={Stokes, Mark G and Kusunoki, Makoto and Sigala, Natasha and Nili, Hamed and Gaffan, David and Duncan, John},
  journal={Neuron},
  volume={78},
  year={2013},
}

@article{genovesio2005prefrontal,
  title={Prefrontal cortex activity related to abstract response strategies},
  author={Genovesio, Aldo and Brasted, Peter J and Mitz, Andrew R and Wise, Steven P},
  journal={Neuron},
  volume={47},
  year={2005},
}

@article{shima2000neuronal,
  title={Neuronal activity in the supplementary and presupplementary motor areas for temporal organization of multiple movements},
  author={Shima, Keisetsu and Tanji, Jun},
  journal={Journal of neurophysiology},
  volume={84},
  year={2000},
}

@article{bengio2014representationlearningreviewnew, 
    author = {Bengio, Yoshua and Courville, Aaron and Vincent, Pascal}, 
    title = {Representation Learning: A Review and New Perspectives}, 
    year = {2013}, 
    journal = {IEEE Trans. Pattern Anal. Mach. Intell.}
}

@article{martin2007representation,
  title={The representation of object concepts in the brain},
  author={Martin, Alex},
  journal={Annu. Rev. Psychol.},
  volume={58},
  year={2007},
}

@inproceedings{zenke2017continual,
	title        = {Continual learning through synaptic intelligence},
	author       = {Zenke, Friedemann and Poole, Ben and Ganguli, Surya},
	year         = 2017,
	booktitle    = {International Conference on Machine Learning},
	organization = {PMLR},
}

@article{kirkpatrick2017overcoming,
	title        = {Overcoming catastrophic forgetting in neural networks},
	author       = {Kirkpatrick, James and Pascanu, Razvan and Rabinowitz, Neil and Veness, Joel and Desjardins, Guillaume and Rusu, Andrei A and Milan, Kieran and Quan, John and Ramalho, Tiago and Grabska-Barwinska, Agnieszka and others},
	year         = 2017,
	journal      = {Proceedings of the national academy of sciences},
	volume       = 114,
}

@inproceedings{lampinen2018analytic,
title={An analytic theory of generalization dynamics and transfer learning in deep linear networks},
author={Andrew K. Lampinen and Surya Ganguli},
booktitle={International Conference on Learning Representations},
year={2019},
}

@article{gerace2022probing,
	title        = {Probing transfer learning with a model of synthetic correlated datasets},
	author       = {Gerace, Federica and Saglietti, Luca and Mannelli, Stefano Sarao and Saxe, Andrew and Zdeborov{\'a}, Lenka},
	year         = 2022,
	journal      = {Machine Learning: Science and Technology},
}

@article{parisi2019continual,
	title        = {Continual lifelong learning with neural networks: A review},
	author       = {Parisi, German I and Kemker, Ronald and Part, Jose L and Kanan, Christopher and Wermter, Stefan},
	year         = 2019,
	journal      = {Neural Networks},
	volume       = 113,
}

@article{refinetti2022Dynamics,
  title = {The Dynamics of Representation Learning in Shallow, Non-Linear Autoencoders},
  author = {Refinetti, Maria and Goldt, Sebastian},
  year = 2023,
  journal = {Journal of Statistical Mechanics: Theory and Experiment},
  volume = {2023},
}

@article{baldi1989Neural,
    title = {Neural networks and principal component analysis: {Learning} from examples without local minima},
    volume = {2},
    journal = {Neural Networks},
    author = {Baldi, Pierre and Hornik, Kurt},
    year = {1989},
}

@article{plaut2018Principal,
    title = {From {Principal} {Subspaces} to {Principal} {Components} with {Linear} {Autoencoders}},
    author = {Plaut, Elad},
    year = {2018},
    journal = {arXiv:1804.10253 },
}

@article{bourlard1988Autoassociation,
  title = {Auto-Association by Multilayer Perceptrons and Singular Value Decomposition},
  author = {Bourlard, H. and Kamp, Y.},
  year = 1988,
  journal = {Biological Cybernetics},
  volume = {59},
}

@inproceedings{
fumero2026navigating,
title={Navigating the Latent Space Dynamics of Neural Models},
author={Marco Fumero and Luca Moschella and Emanuele Rodol{\`a} and Francesco Locatello},
booktitle={The Fourteenth International Conference on Learning Representations},
year={2026},
}

@inproceedings{jiang2020Associative,
author = {Jiang, Yibo and Pehlevan, Cengiz},
title = {Associative memory in iterated overparameterized sigmoid autoencoders},
year = {2020},
booktitle = {Proceedings of the 37th International Conference on Machine Learning},
}

@article{radhakrishnan2020Overparameterized,
    title = {Overparameterized {Neural} {Networks} {Implement} {Associative} {Memory}},
    volume = {117},
    journal = {Proceedings of the National Academy of Sciences},
    author = {Radhakrishnan, Adityanarayanan and Belkin, Mikhail and Uhler, Caroline},
    year = {2020},
}

@inproceedings{bombarispurious,
  title={Spurious Correlations in High Dimensional Regression: The Roles of Regularization, Simplicity Bias and Over-Parameterization},
  author={Bombari, Simone and Mondelli, Marco},
  booktitle={International Conference on Machine Learning},
year={2025}
}

@article{yang2020precise,
 title={Precise High-Dimensional Asymptotics for Quantifying Heterogeneous Transfers},
  author={Yang, Fan and Zhang, Hongyang R and Wu, Sen and Re, Christopher and Su, Weijie J},
  journal={Journal of Machine Learning Research},
  volume={26},
  year={2025}
}

@inproceedings{ildiz2024highdimensional,
title={High-dimensional Analysis of Knowledge Distillation: Weak-to-Strong Generalization and Scaling Laws}, 
author={M. Emrullah Ildiz and Halil Alperen Gozeten and Ege Onur Taga and Marco Mondelli and Samet Oymak},
year={2025},
booktitle={International Conference on Learning Representations}
}

@inproceedings{kolossovtowards,
  title={Towards a statistical theory of data selection under weak supervision},
  author={Kolossov, Germain and Montanari, Andrea and Tandon, Pulkit},
  booktitle={International Conference on Learning Representations},
year={2024}
}

@article{jain2024scaling,
  title={Scaling laws for learning with real and surrogate data},
  author={Jain, Ayush and Montanari, Andrea and Sasoglu, Eren},
  journal={Advances in Neural Information Processing Systems},
  volume={37},
  year={2024}
}

@inproceedings{rezaei2025high,
  title={High-dimensional Analysis of Synthetic Data Selection},
  author={Rezaei, Parham and Kovacevic, Filip and Locatello, Francesco and Mondelli, Marco},
  booktitle={International Conference on Learning Representations},
  year={2026}
}

@article{song2024generalization,
  title={Generalization error of min-norm interpolators in transfer learning},
  author={Song, Yanke and Bhattacharya, Sohom and Sur, Pragya},
  journal={arXiv:2406.13944},
  year={2024}
}

@InProceedings{cyffers2025optimal,
  title = 	 {Optimal Regularization for Performative Learning},
  author =       {Cyffers, Edwige and Mirrokni, Alireza and Mondelli, Marco},
  booktitle = 	 {Proceedings of the 43rd International Conference on Machine Learning},
  year = 	 {2026},
}

@article{han2023distribution,
  author  = {Qiyang Han and Xiaocong Xu},
  title   = {The Distribution of Ridgeless Least Squares Interpolators},
  journal = {Journal of Machine Learning Research},
  year    = {2026},
  volume  = {27},
}

@inproceedings{aghajanyan2021Intrinsic,
    title = {Intrinsic {Dimensionality} {Explains} the {Effectiveness} of {Language} {Model} {Fine}-{Tuning}},
    booktitle = {Proceedings of the 59th {Annual} {Meeting} of the {Association} for {Computational} {Linguistics} and the 11th {International} {Joint} {Conference} on {Natural} {Language} {Processing} ({Volume} 1: {Long} {Papers})},
    author = {Aghajanyan, Armen and Gupta, Sonal and Zettlemoyer, Luke},
    year = {2021},
}

@inproceedings{
braun2022exact,
title={Exact learning dynamics of deep linear networks with prior knowledge},
author={Lukas Braun and Cl{\'e}mentine Carla Juliette Domin{\'e} and James E Fitzgerald and Andrew M Saxe},
booktitle={Advances in Neural Information Processing Systems},
year={2022},
}

@inproceedings{ouyang2022Training,
  title = {Training Language Models to Follow Instructions with Human Feedback},
  booktitle = {Advances in Neural Information Processing Systems},
  author = {Ouyang, Long and Wu, Jeffrey and Jiang, Xu and Almeida, Diogo and Wainwright, Carroll and Mishkin, Pamela and Zhang, Chong and Agarwal, Sandhini and Slama, Katarina and Ray, Alex and Schulman, John and Hilton, Jacob and Kelton, Fraser and Miller, Luke and Simens, Maddie and Askell, Amanda and Welinder, Peter and Christiano, Paul F and Leike, Jan and Lowe, Ryan},
  year = 2022,
  volume = {35},
}

@article{openai2024GPT4,
    title = {{GPT}-4 {Technical} {Report}},
    author = {OpenAI and others},
    year = {2024},
    journal = {arXiv:2303.08774},
}

@inproceedings{hendrycks2021Measuring,
 author = {Hendrycks, Dan and Burns, Collin and Kadavath, Saurav and Arora, Akul and Basart, Steven and Tang, Eric and Song, Dawn and Steinhardt, Jacob},
 booktitle = {Proceedings of the Neural Information Processing Systems Track on Datasets and Benchmarks},
 title = {Measuring Mathematical Problem Solving With the MATH Dataset},
 volume = {1},
 year = {2021}
}

@inproceedings{christiano2023Deep,
  title = {Deep Reinforcement Learning from Human Preferences},
  booktitle = {Advances in Neural Information Processing Systems},
  author = {Christiano, Paul F and Leike, Jan and Brown, Tom and Martic, Miljan and Legg, Shane and Amodei, Dario},
  year = 2017,
  volume = {30},
}

@inproceedings{hu2022LoRA,
title={Lo{RA}: Low-Rank Adaptation of Large Language Models},
author={Edward J Hu and yelong shen and Phillip Wallis and Zeyuan Allen-Zhu and Yuanzhi Li and Shean Wang and Lu Wang and Weizhu Chen},
booktitle={International Conference on Learning Representations},
year={2022},
}

@inproceedings{bengio2007Greedy,
  title = {Greedy Layer-Wise Training of Deep Networks},
  booktitle = {Advances in Neural Information Processing Systems},
  author = {Bengio, Yoshua and Lamblin, Pascal and Popovici, Dan and Larochelle, Hugo},
  year = 2006,
  volume = {19},
}

@article{erhanWhy,
  author  = {Dumitru Erhan and Yoshua Bengio and Aaron Courville and Pierre-Antoine Manzagol and Pascal Vincent and Samy Bengio},
  title   = {Why Does Unsupervised Pre-training Help Deep Learning?},
  journal = {Journal of Machine Learning Research},
  year    = {2010},
  volume  = {11},
}

@article{vincentStacked,
  author  = {Pascal Vincent and Hugo Larochelle and Isabelle Lajoie and Yoshua Bengio and Pierre-Antoine Manzagol},
  title   = {Stacked Denoising Autoencoders: Learning Useful Representations in a Deep Network with a Local Denoising Criterion},
  journal = {Journal of Machine Learning Research},
  year    = {2010},
  volume  = {11},
}

@article{boyle2024Tuned,
  title = {Tuned Geometries of Hippocampal Representations Meet the Computational Demands of Social Memory},
  author = {Boyle, Lara M. and Posani, Lorenzo and Irfan, Sarah and Siegelbaum, Steven A. and Fusi, Stefano},
  year = 2024,
  journal = {Neuron},
  volume = {112},
}

@article{courellis2024Abstract,
  title = {Abstract Representations Emerge in Human Hippocampal Neurons during Inference},
  author = {Courellis, Hristos S. and Minxha, Juri and Cardenas, Araceli R. and Kimmel, Daniel L. and Reed, Chrystal M. and Valiante, Taufik A. and Salzman, C. Daniel and Mamelak, Adam N. and Fusi, Stefano and Rutishauser, Ueli},
  year = 2024,
  journal = {Nature},
  volume = {632},
}

@article{mishchanchuk2024hidden,
  title = {Hidden State Inference Requires Abstract Contextual Representations in the Ventral Hippocampus},
  author = {Mishchanchuk, Karyna and Gregoriou, Gabrielle and Q{\"u}, Albert and Kastler, Aliz{\'e}e and Huys, Quentin J. M. and Wilbrecht, Linda and MacAskill, Andrew F.},
  year = 2024,
  journal = {Science},
  volume = {386},
}

@article{nogueira2023Geometry,
  title = {The Geometry of Cortical Representations of Touch in Rodents},
  author = {Nogueira, Ramon and Rodgers, Chris C. and Bruno, Randy M. and Fusi, Stefano},
  year = 2023,
  journal = {Nature Neuroscience},
  volume = {26},
}

@article{wang2026Mathematical,
  title = {A Mathematical Theory for Understanding When Abstract Representations Emerge in Neural Networks},
  author = {Wang, Bin and Johnston, W. Jeffrey and Fusi, Stefano},
  year = 2026,
  journal = {arXiv:2510.09816},
}

@article{alain2018Understanding,
  title = {Understanding Intermediate Layers Using Linear Classifier Probes},
  author = {Alain, Guillaume and Bengio, Yoshua},
  year = 2018,
  journal={arXiv:1610.01644},
}

@inproceedings{valeriani2023Geometry,
  title = {The Geometry of Hidden Representations of Large Transformer Models},
  booktitle = {Advances in {{Neural Information Processing Systems}}},
  author = {Valeriani, Lucrezia and Doimo, Diego and Cuturello, Francesca and Laio, Alessandro and Ansuini, Alessio and Cazzaniga, Alberto},
  year = 2023,
  volume = {36},
}

@article{anguita2026theory,
      title={A Theory of How Pretraining Shapes Inductive Bias in Fine-Tuning}, 
      author={Nicolas Anguita and Francesco Locatello and Andrew M. Saxe and Marco Mondelli and Flavia Mancini and Samuel Lippl and Clementine Domine},
      year={2026},
      journal={arXiv:2602.20062},
}

@article{fusi_why_2016,
    title = {Why neurons mix: high dimensionality for higher cognition},
    volume = {37},
    shorttitle = {Why neurons mix},
    language = {en},
    urldate = {2024-01-23},
    journal = {Current Opinion in Neurobiology},
    author = {Fusi, Stefano and Miller, Earl K and Rigotti, Mattia},
    year = {2016},
}

@article{advani2020Highdimensional,
    title = {High-dimensional dynamics of generalization error in neural networks},
    volume = {132},
    journal = {Neural Networks},
    author = {Advani, Madhu S. and Saxe, Andrew M. and Sompolinsky, Haim},
    year = {2020},
}

@InProceedings{jonesmccormick2025provable,
  title = 	 {Provable Benefits of Unsupervised Pre-training and Transfer Learning via Single-Index Models},
  author =       {Jones-Mccormick, Taj and Jagannath, Aukosh and Sen, Subhabrata},
  booktitle = 	 {Proceedings of the 42nd International Conference on Machine Learning},
  year = 	 {2025},
  volume = 	 {267},
  publisher = {PMLR}
}

@article{arora2019theoretical,
      title={A Theoretical Analysis of Contrastive Unsupervised Representation Learning}, 
      author={Sanjeev Arora and Hrishikesh Khandeparkar and Mikhail Khodak and Orestis Plevrakis and Nikunj Saunshi},
      year={2019},
      journal={arXiv:1902.09229},
}

@article{pan2010survey,
  author={Pan, Sinno Jialin and Yang, Qiang},
  journal={IEEE Transactions on Knowledge and Data Engineering}, 
  title={A Survey on Transfer Learning}, 
  year={2010},
  volume={22},
}

@inproceedings{tahir2025features,
title={Features are fate: a theory of transfer learning in high-dimensional regression},
author={Javan Tahir and Surya Ganguli and Grant M. Rotskoff},
booktitle={Forty-second International Conference on Machine Learning},
year={2025},
}

@article{zhuang2020comprehensive,
author = {Zhuang, Fuzhen and Qi, Zhiyuan and Duan, Keyu and Xi, Dongbo and Zhu, Yongchun and Zhu, Hengshu and Xiong, Hui and He, Qing},
year = {2020},
title = {A Comprehensive Survey on Transfer Learning},
volume = {PP},
journal = {Proceedings of the IEEE},
}

@inproceedings{lee2021predicting,
title={Predicting What You Already Know Helps: Provable Self-Supervised Learning},
author={Jason D. Lee and Qi Lei and Nikunj Saunshi and Jiacheng Zhuo},
booktitle={Advances in Neural Information Processing Systems},
editor={A. Beygelzimer and Y. Dauphin and P. Liang and J. Wortman Vaughan},
year={2021},
}

@inproceedings{zhang-hashimoto-2021-inductive,
    title = "On the Inductive Bias of Masked Language Modeling: From Statistical to Syntactic Dependencies",
    author = "Zhang, Tianyi  and
      Hashimoto, Tatsunori B.",
    booktitle = "Proceedings of the 2021 Conference of the North American Chapter of the Association for Computational Linguistics: Human Language Technologies",
    year = "2021",
}

@inproceedings{azar2024semisupervised,
title={Semi-Supervised Sparse Gaussian Classification: Provable Benefits of Unlabeled Data},
author={Eyar Azar and Boaz Nadler},
booktitle={The Thirty-eighth Annual Conference on Neural Information Processing Systems},
year={2024},
}

@article{badre_dimensionality_2021,
    title = {The dimensionality of neural representations for control},
    volume = {38},
    journal = {Current Opinion in Behavioral Sciences},
    author = {Badre, David and Bhandari, Apoorva and Keglovits, Haley and Kikumoto, Atsushi},
    year = {2021},
}

@InProceedings{nam2025position,
  title = 	 {Position: Solve Layerwise Linear Models First to Understand Neural Dynamical Phenomena ({N}eural Collapse, Emergence, {L}azy/{R}ich Regime, and Grokking)},
  author =       {Nam, Yoonsoo and Lee, Seok Hyeong and Domin\'{e}, Cl\'{e}mentine Carla Juliette and Park, Yeachan and London, Charles and Choi, Wonyl and G\"{o}ring, Niclas Alexander and Lee, Seungjai},
  booktitle = 	 {Proceedings of the 42nd International Conference on Machine Learning},
  year = 	 {2025},
}

@article{domine2024lazy,
  title={From Lazy to Rich: Exact Learning Dynamics in Deep Linear Networks},
  author={Domin{\'e}, Cl{\'e}mentine CJ and Anguita, Nicolas and Proca, Alexandra M and Braun, Lukas and Kunin, Daniel and Mediano, Pedro AM and Saxe, Andrew M},
  journal={arXiv:2409.14623},
  year={2024}
}

@article{kunin2024get,
  title={Get rich quick: exact solutions reveal how unbalanced initializations promote rapid feature learning},
  author={Kunin, Daniel and Ravent{\'o}s, Allan and Domin{\'e}, Cl{\'e}mentine and Chen, Feng and Klindt, David and Saxe, Andrew and Ganguli, Surya},
  journal={Advances in Neural Information Processing Systems},
  volume={37},
  year={2024}
}

@inproceedings{wang_superglue_2019a,
  title = {{{SuperGLUE}}: A Stickier Benchmark for General-Purpose Language Understanding Systems},
  booktitle = {Advances in Neural Information Processing Systems},
  author = {Wang, Alex and Pruksachatkun, Yada and Nangia, Nikita and Singh, Amanpreet and Michael, Julian and Hill, Felix and Levy, Omer and Bowman, Samuel},
  year = 2019,
}

@inproceedings{agrawal_$alpha$req_2022,
    title = {\${\textbackslash}alpha\$-{ReQ} : {Assessing} {Representation} {Quality} in {Self}-{Supervised} {Learning} by measuring eigenspectrum decay},
    author = {Agrawal, Kumar Krishna and Mondal, Arnab Kumar and Ghosh, Arna and Richards, Blake Aaron},
    booktitle = {Advances in Neural Information Processing Systems},
    year = {2022},
}

@inproceedings{
li_tracing_2025,
title={Tracing the Representation Geometry of Language Models from Pretraining to Post-training},
author={Melody Zixuan Li and Kumar Krishna Agrawal and Arna Ghosh and Komal Kumar Teru and Adam Santoro and Guillaume Lajoie and Blake Aaron Richards},
booktitle={The Thirty-ninth Annual Conference on Neural Information Processing Systems},
year={2025},
}

\newpage

\renewcommand{\thesubsection}{\Alph{subsection}}

\startcontents[appendix]

\section*{Appendix}

\begin{center}
    \large\textbf{Appendix contents}
\end{center}
\vspace{-0.5em}

\printcontents[appendix]{}{2}{
    \setcounter{tocdepth}{3}
}

\subsection{Assumptions and proofs} \label{app:theory}

\subsubsection{Assumptions}
\label{app:assumptions}

\begin{assumption}[High-dimensional asymptotic regime]
\label{ass:high_dim}
Let the number of features be $p$, and \( n_l , n_u, m \) be such that
\[
\frac{p}{n_l} \to \gamma_l \in (0,\infty),
\qquad
\frac{p}{n_u} \to \gamma_u \in (0,\infty),
\qquad
\frac{m}{p} \to \alpha \in [0,1].
\]
Furthermore, we define $\gamma_{\mathrm{eff}} := \frac{m}{n_l}
\;\to\;
\alpha \gamma_l\in [0, \infty)$.  We assume $1\le m\le p$. In particular, $\alpha=0$ means $m/p\to0$ while retaining at least the leading empirical principal component; it does not mean $m=0$. This includes both fixed $m$ and sequences $m\to\infty$ with $m=o(p)$. For $\alpha=0$, $\gamma_{\mathrm{eff}}=0$.
\end{assumption}
In words, $\alpha$ is the fraction of dimensions to keep, $\gamma_u$ and $\gamma_l$ are the aspect ratios the prior and downstream datasets, and $\gamma_{\mathrm{eff}}$ is the effective aspect ratio of the compressed downstream dataset.

\begin{assumption}[The spiked covariance models]
\label{ass:spike_cov}
For each $p$, the pretraining and downstream population covariance matrices are
\[
\bSigma_u = \mathbf I_p + (\lambda_u - 1) \bv \bv^\top,
\qquad
\bSigma_l = \mathbf I_p + (\lambda_l - 1) \bq \bq^\top,
\]
where $\lambda_u,\lambda_l>1$, $\bv,\bq\in\mathbb R^p$ are deterministic unit vectors.
\end{assumption}

Note that the eigendecomposition of $\bSigma_u$ is 
\[ 
\bSigma_u = \V \mathrm{diag} (\lambda_u, 1, \dots ,1) \V^\top,
\]
with $\V = [\bv_{1}, \bv_{2}, \dots, \bv_{p}] \in \mathbb R^{p \times p}$. As we often refer directly to the spike eigenvector, we keep the simpler notation $\bv_{1} = \bv$. Similarly for 
$\bSigma_l$ where we use  $\bq_{1} = \bq$.

\begin{assumption}[Data generation]
    \label{ass:data_generation}
    The matrices $\X_l\in\mathbb R^{n_l\times p}$ and $\X_u\in\mathbb R^{n_u\times p}$ are independent of each other. Each matrix has independent, identically distributed rows, with
\[
\x_l \sim \mathcal{N}(0, \bSigma_l),
\qquad
\x_u \sim \mathcal{N}(0, \bSigma_u).
\]
\end{assumption}
Equivalently, let $\Z_u\in\mathbb R^{n_u\times p}$ and $\Z_l\in\mathbb R^{n_l\times p}$ be mutually independent matrices
with independent $\mathcal N(0,1)$ entries. Then $\X_u=\Z_u\bSigma_u^{1/2}$ and $\X_l=\Z_l\bSigma_l^{1/2}$.
The label-noise vector $\bxi\sim\mathcal N(0,\sigma^2\I_{n_l})$ is independent of both datasets.

\begin{assumption}[Task weight vector]
    The deterministic task vector satisfies
    \[
    \|\wstar\|^2\longrightarrow\bar w^\ast\in(0,\infty).
    \]
    The deterministic unit vectors $\bv,\bq$ and the task vector have convergent signed normalized overlaps:
    \[
    \frac{\wstar^\top\bv}{\|\wstar\|}\to r_v,
    \qquad
    \frac{\wstar^\top\bq}{\|\wstar\|}\to r_q,
    \qquad
    \bv^\top\bq\to r.
    \]
    
    Furthermore, let the alignment between the true signal $\wstar$ and the pretraining covariance structure $\bSigma_u$ be given by the spectral measure
    \[
    \hat G_\wstar(\tau) = \frac{1}{\|\wstar\|^2} \sum_{i=1}^p (\wstar^\top \bv_{i})^2 \mathbf 1\{ \lambda_i(\bSigma_u) \leq \tau\}.
    \]
Then, there exists a limiting spectrum $G_\wstar$ such that $\hat G_\wstar \Rightarrow G_\wstar$ almost surely. Similarly, we define the alignment between the downstream spike $\bq$ and the pretraining covariance $\bSigma_u$ 
\[
    \hat G_\bq(\tau) = \sum_{i=1}^p (\bq^\top \bv_{i})^2 \mathbf 1\{ \lambda_i(\bSigma_u) \leq \tau\}.
    \]
with a limiting spectrum $\hat G_\bq \Rightarrow G_\bq$.

\end{assumption}

\paragraph{Sample covariance eigenspectrum.} We further note that sample covariance matrices $\mathbf S_u = \frac{1}{n_u} \X_u^\T \X_u$ and $\Ss_l =  \frac{1}{n_l} \X_l^\T \X_l$ have empirical spectral measures

\begin{align*}
    &\hat F^{\Ss_u}(\theta) = \frac{1}{p} \sum_{i=1}^p \mathbf 1 \{ \lambda_i(\Ss_u) \leq \theta \}, & \hat F^{\Ss_l}(\theta) = \frac{1}{p} \sum_{i=1}^p \mathbf 1 \{ \lambda_i(\Ss_l) \leq \theta \},
\end{align*}
which converge weakly almost surely to Marcenko-Pastur laws \cite{marcenko_distribution_1967, bai_spectral_2010} with aspect ratios $\gamma_u$ and $\gamma_l$ respectively:

\begin{align*}
    &\hat F^{\Ss_u} \Rightarrow F_{\gamma_u}, &\hat F^{\Ss_l} \Rightarrow F_{\gamma_l}.
\end{align*}
For completeness, we state the Marcenko-Pastur density with an aspect ratio $\gamma$:

\begin{equation} \label{eq:mp_definition}
\begin{aligned}
    f_{\gamma}(\lambda) &= \frac{\sqrt{(\lambda_+ - \lambda)(\lambda - \lambda_-)}}{2\pi\lambda \gamma} +  
     \mathbf{1}_{\gamma>1}\left(1-\frac{1}{\gamma}\right)\delta(\lambda), \\
     \lambda_+ &= (1+\sqrt{\gamma})^2, \\
     \lambda_- &= (1 - \sqrt{\gamma})^2, 
\end{aligned}
\end{equation}
with $\lambda_+, \lambda_-$ denoting the edges of the support of the distribution.

\subsubsection{Main ideas of the proofs}\label{app:proofideas}

The key idea of each proof is to decompose the model's estimator $\what$ into a signal term, a bias term arising from the component of $\wstar$ outside the relevant subspace, and two variance terms arising from noise and model mismatch. The Gaussian structure of the data allows us to express the component of $\X_l$ orthogonal to the prior subspace $\U_m$ as a linear function of $\X_l\U_m$ plus independent noise $\E$. This reduces the analysis to understanding various projections onto the row space of spiked Gaussian matrices. 

\paragraph{Selection of the retained PCA subspace.}
The columns of $\U_m$ are eigenvectors associated with the $m$
largest eigenvalues of $\Ss_u$. If $m>\operatorname{rank}(\Ss_u)$,
we retain all positive-eigenvalue directions and complete the
subspace with a Haar-uniform
$(m-\operatorname{rank}(\Ss_u))$-dimensional subspace of
$\ker(\Ss_u)$, conditional on $\X_u$.
The auxiliary randomness used for this completion is independent
of the downstream data and label noise.

\paragraph{Conditional Gaussian decomposition under prior subspace projections.}
\label{sec:gaussian_decomposition}
Let $\X_l \in \mathbb{R}^{n_l \times p}$ be a random matrix whose rows $\x_i \sim \mathcal{N}(0,\bSigma_l)$ are i.i.d., and let $\U_m \in \mathbb{R}^{p \times m}$ be a random matrix with orthonormal columns. Let $\U_\perp \in \mathbb{R}^{p \times (p-m)}$ be such that $[\U_m \ \U_\perp]$ is an orthogonal matrix. Define
\[
\Xm = \X_l \U_m \in \mathbb{R}^{n_l \times m}, 
\qquad
\Xperp = \X_l \U_\perp \in \mathbb{R}^{n_l \times (p-m)}.
\]
Then, conditioning on $\U_m$, the rows of $(\Xm, \Xperp)$ are jointly Gaussian with covariance
\[
\begin{bmatrix}
\bSigma_{mm} & \bSigma_{m\perp} \\
\bSigma_{\perp m} & \bSigma_{\perp\perp}
\end{bmatrix}
=
\begin{bmatrix}
\U_m^\top \bSigma_l \U_m & \U_m^\top \bSigma_l \U_\perp \\
\U_\perp^\top \bSigma_l \U_m & \U_\perp^\top \bSigma_l \U_\perp
\end{bmatrix}.
\]
Moreover, due to the conditional distribution, $\Xperp$ given $\Xm$ satisfies
\begin{equation} \label{eq:xperp_decomp}
    \Xperp = \Xm \bSigma_{mm}^{-1}\bSigma_{m\perp} + \E,
\end{equation}
where $\E$ is independent of $\Xm$ and its rows are i.i.d.\ with distribution  
$\mathbf e_i \sim \mathcal{N}\!\left(0,\; \bar \bSigma\right)$, with $\bar \bSigma :=\bSigma_{\perp\perp} - \bSigma_{\perp m}\bSigma_{mm}^{-1}\bSigma_{m\perp}$.
Under the spiked covariance model of Assumption \ref{ass:spike_cov}, the population covariance $\bSigma_{mm}$ of $\X_m$  equals
\begin{equation}
    \bSigma_{mm} = \I_m + (\lamtilde_l-1)\tilde \bq \tilde \bq^\top,
\end{equation}
where $\lamtilde_l = 1 + (\lambda_l-1) \| \Pm\bq\|^2$ and $\tilde \bq = \frac{\U_m^\top \bq}{\| \Pm \bq\|}$. Furthermore, after some algebraic manipulations using the Sherman-Morrison formula, it follows that
\begin{align}
    \bSigma_{mm}^{-1} &= \mathbf I_m - \frac{\lamtilde_l - 1}{\lamtilde_l} \tilde \bq\tilde \bq^\top, \\
         \bSigma_{mm}^{-1}\bSigma_{m\perp} &=  \frac{(\lambda_l-1) \| \Pm \bq\|}{\lamtilde_l} \qtilde \bq^\top \U_\perp, \label{eq:sigma_mm_inv_sigma_mperp}\\
                  \bar \bSigma &=  \mathbf I_{p-m} + \frac{\lambda_l-1}{\lamtilde_l} \U_\perp^\top \bq\bq^\top \U_\perp. \label{eq:barsigma}
\end{align}
\paragraph{Decomposing the model's weight estimate.} From the definition of $\what$ in Equation~\eqref{eq:weight_estimate} and the decomposition in Equation \eqref{eq:xperp_decomp}, we can write:
\begin{equation*}
    \begin{aligned}
       \what &= (\X_l \Pm)^\dagger \X_l\Pm \wstar +(\X_l \Pm)^\dagger \X_l\Pperp \wstar  + (\X_l\Pm)^\dagger \xi \\
       &=  (\X_m \U_m^\top)^\dagger \X_m \U_m^\top\wstar + (\X_m \U_m^\top)^\dagger \X_m \bSigma_{mm}^{-1} \bSigma_{m\perp} \U_\perp^\top \wstar +   (\X_m \U_m^\top)^\dagger \E\U_\perp^\top \wstar +  (\X_m \U_m^\top)^\dagger \xi.
    \end{aligned}
\end{equation*}
Then, using Equation~\eqref{eq:sigma_mm_inv_sigma_mperp} leads to 
\begin{equation} \label{eq:transformed_wstar}
\bSigma_{mm}^{-1}\bSigma_{m\perp} \U_\perp^\top \wstar =   \underbrace{\frac{(\lambda_l-1)(\wstar^\top \Pperp \bq)}{\lamtilde_l}}_{A}\; \U_m^\top \bq 
\end{equation}
and by defining 
\[\Pi_m := \Xm^\dagger \Xm \in \mathbb{R}^{m \times m} \qquad  \Pi := (\X_l \Pm)^\dagger \X_l \Pm = \U_m \Pi_m \U_m^\top \in \mathbb R^{p \times p},\]
we obtain
\begin{equation}\label{eq:what_expression}
\begin{aligned}
    \what 
    &=\underbrace{\Pi \wstar}_{\text{signal from $\Pm$ subspace}} + \underbrace{A \cdot \Pi \bq}_{\text{signal from $\Pperp$ subspace that leaked into $\Pm$}}\\
    &+\underbrace{(\X_l \Pm)^\dagger \E \wstar_\perp}_{\text{variance from model mismatch}} + \underbrace{(\X_l\Pm)^\dagger \xi}_{\text{variance from label noise}},
\end{aligned}
\end{equation}
where $\wstar_\perp:= \U_\perp^\top \wstar$.
 Note that the component of $\wstar$ not involving label noise is 
\begin{equation}\label{eq:noiseless_what_decomp}
    (\X_l\Pm)^\dagger \X_l \wstar = \Pi \wstar + A \Pi \bq + (\X_l\Pm)^\dagger \E \wstar_{\perp}.
\end{equation}
The model's weight estimate decomposes into the four parts discussed below.
\begin{itemize}
    \item $\Pi \wstar$: part of $\wstar$ that lies in the subspace formed by $\Pm$ and is identifiable from the available data. It is equal to $\Pm\wstar$ if $m < n_l$ (i.e., $\gamma_\mathrm{eff} < 1$), since in that case $\Pi_m = \mathbf I_m$  (see Lemma~\ref{lem:projection_effective}); if $m > n_l$ (i.e., $\gamma_\mathrm{eff} > 1$), this term causes bias due to the finite sample size.
    \item $A \Pi \bq$: a part of $\wstar$ that lies outside the subspace $\Pm$ can leak into $\Pm$, causing the model to falsely reconstruct it. The reason is the existence of the population covariance spike which makes the rows of $\X_l \Pperp$ and $\X_l\Pm$ correlated, so the part outside the subspace leaks back into the subspace through this covariance structure. 
    \item $(\X_l \Pm)^\dagger \E \wstar_\perp$:  the part of the $\wstar$ that lies outside the subspace $\Pm$ and does not leak into $\Pm$ behaves like noise, eventually causing a double-descent-like curve that scales with $\|\wstar\|^2$.
    \item $(\X_l\Pm)^\dagger \xi$: label noise variance, causing another double-descent-like curve that scales with $\sigma^2$.
\end{itemize}

\subsubsection{Proof for the estimation error}
\label{app:proof_estimation_error}
\begin{proof}[Proof of Theorem~\ref{thm:estimation_error}]

The proof proceeds in three stages, exploiting the conditional independence structure of the data. First, we use the joint Gaussianity of $\X_m $ and $\X_\perp$ to rewrite $\X_\perp$ in terms of $\X_m$ and an independent noise matrix $\E \in \mathbb{R}^{n_l \times (p-m)}$ (see Section~\ref{sec:gaussian_decomposition}). We then average over $\E$ conditioned on $\X_m$ and $\U_m$, then average over $\X_m$ conditioned on $\U_m$, and finally over $\U_m$.

\paragraph{Step 1: Decomposition of the error.}

The estimation error splits into bias and variance:
\begin{equation*}
    \begin{aligned}
        \mathbb E_\xi \left[\| \wstar - \what\|^2 \right]&=   \left\|\left(\mathbf{I} - (\X_l\Pm)^\dagger\X_l\right)\wstar\right\|^2  + \mathbb E_\xi \left[ \xi^\top (\X_l\Pm)^{\dagger\top}(\X_l\Pm)^\dagger \xi\right] \\
        &- 2 \mathbb E_\xi \left[\xi^\top (\X_l\Pm)^{\dagger\top} \left(\mathbf{I} - (\X_l\Pm)^\dagger\X_l\right)\wstar\right] \\
        &=  \left\|\left(\mathbf{I} - (\X_l\Pm)^\dagger\X_l\right)\wstar\right\|^2  + \mathbb E_\xi \left[ \mathrm{Tr}\left[\xi\xi^\top (\X_l\Pm)^{\dagger\top}(\X_l\Pm)^\dagger\right]\right] \\
        &=  \left\|\left(\mathbf{I} - (\X_l\Pm)^\dagger\X_l\right)\wstar\right\|^2  + \sigma^2 \mathrm{Tr}\left[(\Pm \X_l^\top\X_l\Pm)^\dagger\right].
    \end{aligned}
\end{equation*}
Substituting the result in \eqref{eq:noiseless_what_decomp}  further decomposes the bias:
\begin{equation*}
\begin{aligned}
    \frac{1}{\|\wstar \|^2}\left\|\left(\mathbf{I} - (\X_l\Pm)^\dagger\X_l\right)\wstar\right\|^2 &=  \frac{1}{\|\wstar \|^2} \|\wstar - \Pi\wstar - A\Pi\bq - (\X_l\Pm)^\dagger \E \U_\perp^\top\wstar \|^2 \\
    &=  \frac{1}{\|\wstar \|^2}\|\wstar - \Pi\wstar - A\Pi\bq\|^2 +  \frac{1}{\bar w^\ast}\left\|(\X_l\Pm)^\dagger \E \U_\perp^\top\wstar\right\|^2 \\
    &\quad - 2\underbrace{\frac{1}{\|\wstar \|^2}(\wstar - \Pi\wstar - A\Pi\bq)^\top 
      (\X_l\Pm)^\dagger \E \U_\perp^\top\wstar}_{\text{cross term } T}.
\end{aligned}
\end{equation*}
\paragraph{Step 2: $\E$-independent term.}

The first term $\|\wstar - \Pi\wstar - A\Pi\bq\|^2$ does not depend on $\E$, so its conditional expectation on $\E$ is itself. We expand using the property of the projection matrix $\Pi$ that $\Pi^2 = \Pi$:
\begin{equation}
     \frac{1}{\|\wstar \|^2}\|\wstar - \Pi\wstar - A\Pi\bq\|^2 
    = 1 -  \frac{\|\Pi\wstar\|^2 }{\|\wstar \|^2}
    + \frac{A^2}{\|\wstar \|^2}\|\Pi\bq\|^2,
\end{equation}
where the cross term $-2A\wstar^\top\Pi\bq + 2A\wstar^\top \Pi\Pi\bq = 0$. Substituting $A = \frac{\lambda_l-1}{\tilde\lambda_l}(\bq^\top\Pperp\wstar)$ gives
\begin{equation}\label{eq:temp_theorem1_1}
    1 - \frac{\|\Pi\wstar\|^2 }{\|\wstar\|^2}
    + \frac{(\lambda_l-1)^2}{\tilde\lambda_l^2}\frac{(\bq^\top\Pperp\wstar)^2}{\|\wstar\|^2} \|\Pi\bq\|^2 \Rightarrow 1 - \bar\pi_{\wstar, \wstar} + \frac{(\lambda_l-1)^2}{\bar \lambda_l^2} \bar p_{\perp, \wstar, \bq} ^2 \bar \pi_{\bq, \bq} \,\, \text{almost surely},
\end{equation}
which matches the first two terms of the bias. The projections $\frac{\|\Pi\wstar\|^2 }{\|\wstar\|^2}$ and $\|\Pi\bq\|^2$ converge almost surely to $\bar \pi_{\wstar, \wstar}, \bar \pi_{\bq, \bq}$ in Equation \eqref{eq:limit_pitilde_projections} according to Lemma~\ref{lem:projection_effective}.  The quantity $\frac{\bq^\top\Pperp\wstar}{\|\wstar\|} $ converges almost surely to $\bar p_{\perp, \wstar, \bq}$ in Equation \eqref{eq:bilinear_proj} according to Lemma~\ref{lem:bilinear_projection}, and $\lamtilde_l = 1 + (\lambda_l - 1)\|\Pm \bq\|^2$ converges to $\bar \lambda_l$ almost surely by the continuous mapping theorem and the almost sure convergence of $\|\Pm \bq\|^2 \to \bar p_{\bq, \bq}$ from Equation \eqref{eq:pperpv} (see Lemma~\ref{lem:pperpv}). By the continuous mapping theorem, the combined expression in \eqref{eq:temp_theorem1_1} converges almost surely.

\paragraph{Step 3: Cross term linear in $\E$.}

We write the cross term as
\[
T = \mathbf c^\top \E \mathbf d,
\]
with
\begin{align*}
\mathbf c &:= \frac{1}{\|\wstar \|} (\X_l\Pm)^{\dagger\top} (\wstar - \Pi\wstar - A\Pi\bq) \in \mathbb R^{n_l}, \\
\mathbf d &:= \frac{1}{\|\wstar \|} \U_\perp^\top \wstar \in \mathbb R^{p-m}.
\end{align*}
Conditioned on $(\X_m,\U_m)$, the vectors $\mathbf c$ and $\mathbf d$ are deterministic and independent of $\E$. Since the rows of $\E$ are i.i.d.\ Gaussian with covariance $\bar\bSigma$, it follows that
\[
T \mid (\X_m,\U_m) \sim \mathcal N\!\left(0,\, (\mathbf d^\top \bar\bSigma \mathbf d)\,\|\mathbf c\|^2 \right).
\]
We now bound the variance. First,
\[
\mathbf d^\top \bar\bSigma \mathbf d
=
\frac{\| \Pperp \wstar\|^2}{\|\wstar\|^2}
+
\frac{\lambda_l-1}{\tilde\lambda_l}
\frac{(\wstar^\top \Pperp \bq)^2}{\|\wstar\|^2},
\]
which by the almost sure limits of $\|\Pperp \wstar\|^2$, $\tilde\lambda_l$, and $\wstar^\top \Pperp \bq$, converges almost surely to
\[
\bar\sigma^2_{\mathrm{eff}} = \bar p_{\perp, \wstar, \wstar} + \frac{\lambda_l-1}{\bar\lambda_l}\bar p_{\perp,\wstar,\bq}^2.
\] 
In particular, there exists a deterministic constant $C_1$ such that for all sufficiently large $p$,
\[
\mathbf d^\top \bar\bSigma \mathbf d \le C_1 \quad \text{almost surely.}
\]
It therefore remains to control $\|\mathbf c\|^2$. We have
\[
\|\mathbf c\|^2 \le \frac{1}{\|\wstar\|^2} \frac{\|\wstar - \Pi\wstar - A\Pi\bq\|^2}{\sigma_{\min}(\X_m)^2}.
\]
By Step~2, the numerator is bounded almost surely for all sufficiently large $p$. For the denominator, writing \( \X_m = \mathbf G \bSigma_{mm}^{1/2} \) (see Lemma~\ref{lem:esd_projected_x}) gives
\[
\sigma_{\min}(\X_m) \ge \sigma_{\min}(\mathbf G) \sigma_{\min}(\bSigma_{mm}^{1/2}).
\]
Note that $\sigma_{\min}(\bSigma_{mm}^{1/2}) = \min \{1, \sqrt{\tilde \lambda_l}\}$ for any $p$. By the singular value inequality for Gaussian random matrices (see e.g. Exercise 7.13 in \cite{Vershynin_2026}), for all $t>0$,
\[
\mathbb P\!\left( \sigma_{\min}(\mathbf G) \le |\sqrt{n_l}-\sqrt m|-t \right) \le 2e^{-t^2/2}.
\]
Since $\gamma_{\mathrm{eff}} = m/n_l \neq 1$, we have
\[
|\sqrt{n_l}-\sqrt m| = \sqrt{p}\frac{|1-\sqrt{\gamma_{\mathrm{eff}}}|}{\sqrt{\gamma_l}} = \sqrt{p}  \,c_\gamma.
\]
Choosing $t = \tfrac12 c_\gamma \sqrt p$ and setting $c=c_\gamma^2/8$, we obtain
\[
\mathbb P\!\left( \sigma_{\min}(\mathbf G)
\le \frac12 c_\gamma \sqrt p \right)
\le
2\exp\!\left(-\frac{c_\gamma^2}{8}p\right).
\]
Noting that $\tilde \lambda_l \le \lambda_l$, and setting \( a_\gamma := \frac14 c_\gamma^2 \min\{1,\lambda_l\} \), \(b_\gamma := \frac18 c_\gamma^2\),
we have
\[
\mathbb P\!\left(
\sigma_{\min}(\X_m)^2
\le
a_\gamma p
\right)
\le
2e^{-b_\gamma p}.
\]
Therefore, there exists $C_2>0$ such that
\[
\mathbb P\!\left( \|\mathbf c\|^2 \le \frac{C_2}{p} \right) \ge 1-2e^{-b_\gamma p}.
\]
Define the event
\[
\mathcal E_p := \left\{ \|\mathbf c\|^2 \le \frac{C_2}{p} \right\}.
\]
Then
\[
\mathbb P(\mathcal E_p^c) \le 2e^{-b_\gamma p}.
\]

For any $\delta>0$,
\[
\mathbb P(|T|\ge \delta) = \mathbb P(|T|\ge \delta, \mathcal E_p^c) + \mathbb P(|T|\ge \delta,\mathcal E_p) \le \mathbb P( \mathcal E_p^c) + \mathbb P(|T|\ge \delta,\mathcal E_p).
\]
Noting that there exists $C_1$ s.t. $\mathbf d^\top  \bar \bSigma \mathbf d \le C_1$ for any large $p$, on $\mathcal E_p$,
\[
\mathrm{Var}(T \mid \X_m,\U_m) \le \frac{C_1 C_2}{p}.
\]
Thus, on $\mathcal E_p$, we have
\[
\mathbb P\!\left( |T|\ge \delta \mid \X_m,\U_m \right) \le 2\exp(-c_\delta p),
\]
and since this holds for any $\Xm, \U_m$, by the law of total probability, we have
\[
\mathbb P(|T|\ge \delta,\mathcal E_p) \le 2e^{-c_\delta p}.
\]

Therefore,
\[
\mathbb P(|T|\ge \delta) \le
2e^{-b_\gamma p} + 2e^{-c_\delta p} = C e^{-c p}.
\]
Since $\sum_p  C e^{-c p} < \infty$, the Borel--Cantelli lemma implies
\[
T \xrightarrow{a.s.} 0.
\]

\paragraph{Step 4: Average of the trace terms.}
The variance term contains
\[
\mathrm {Tr} [ (\Pm \X_l^\top \X_l \Pm)^\dagger] = \Tr[(\X_m^\top \X_m)^\dagger].
\]
We note that the term $\mathrm{Tr}[(\Xm \Xm^\top)^\dagger]$ is equivalent to the above, since the non-zero eigenvalues of $\Xm^\top \Xm$ and $\Xm \Xm^\top$ are the same. By Lemma~\ref{lem:trace_pinv_det}, the trace term has a limiting value $\frac{\min\{ 1, \gamma_\mathrm{eff}\}}{|\gamma_\mathrm{eff} - 1|}$ almost surely. 

\paragraph{Step 5: Quadratic term in $\E$.}

Define
\[
Q
:=
\wstar^{\top} \U_\perp \E^\top
(\X_m\X_m^\top)^\dagger
\E \U_\perp^\top \wstar .
\]
Let
\[
\ba := \U_\perp^\top \wstar,
\qquad
\mathbf A := (\X_m\X_m^\top)^\dagger .
\]
Conditioned on $(\X_m,\U_m)$, the vector $\ba$ and the matrix $\mathbf A$
are fixed and independent of $\E$. Since the rows of $\E$ are i.i.d.\
Gaussian with covariance $\bar\bSigma$, the vector
\[
\mathbf c := \E \ba
\]
is Gaussian with
\[
\mathbf c \mid (\X_m,\U_m)
\sim
\mathcal N\!\left(
0,\,
(\ba^\top \bar\bSigma \ba)\mathbf I_{n_l}
\right).
\]
Therefore
\[
Q
=
\mathbf c^\top \mathbf A \mathbf c .
\]

First, the conditional expectation is
\[
\mathbb E_\E[Q\mid \X_m,\U_m]
=
(\ba^\top \bar\bSigma \ba)\Tr(\mathbf A).
\]
Here
\[
\ba^\top \bar\bSigma \ba
=
\wstar^\top \U_\perp \bar\bSigma \U_\perp^\top \wstar
=
\|\wstar\|^2 \sigma_{\mathrm{eff}}^2 .
\]
By the deterministic equivalents used above,
\[
\|\wstar\|^2\sigma_{\mathrm{eff}}^2
\xrightarrow{a.s.} \bar w^\ast
\bar\sigma_{\mathrm{eff}}^2 .
\]
Moreover, by the pseudo-inverse trace computation in Step~4,
\[
\Tr\!\left[(\X_m\X_m^\top)^\dagger\right]
\xrightarrow{a.s.}
\frac{\min\{1,\gamma_{\mathrm{eff}}\}}
{|\gamma_{\mathrm{eff}}-1|}.
\]
Thus
\[
\mathbb E_\E[Q\mid \X_m,\U_m]
\xrightarrow{a.s.}
\bar w^\ast
\bar\sigma_{\mathrm{eff}}^2
\frac{\min\{1,\gamma_{\mathrm{eff}}\}}
{|\gamma_{\mathrm{eff}}-1|}.
\]

It remains to show concentration around this conditional expectation. Write
\[
\mathbf c
=
\sqrt{\ba^\top \bar\bSigma \ba}\;\mathbf z,
\qquad
\mathbf z\sim \mathcal N(0,\mathbf I_{n_l}),
\]
conditionally on $(\X_m,\U_m)$. Then
\[
Q-\mathbb E_\E[Q\mid \X_m,\U_m]
=
(\ba^\top \bar\bSigma \ba)
\left(
\mathbf z^\top \mathbf A \mathbf z-\Tr(\mathbf A)
\right).
\]
By the Hanson--Wright inequality, for every fixed $\delta>0$,
\[
\mathbb P_\E\!\left(
\left|Q-\mathbb E_\E[Q\mid \X_m,\U_m]\right|
\ge \delta
\mid \X_m,\U_m
\right)
\le
2\exp\!\left[
-c\min\left(
\frac{\delta^2}
{(\ba^\top\bar\bSigma\ba)^2\|\mathbf A\|_F^2},
\frac{\delta}
{(\ba^\top\bar\bSigma\ba)\|\mathbf A\|}
\right)
\right].
\]
We now bound the quantities appearing in this inequality. As above,
\[
\ba^\top\bar\bSigma\ba
=
\|\wstar\|^2\sigma_{\mathrm{eff}}^2
\]
is bounded for all sufficiently large $p$ almost surely. Furthermore, since
$\gamma_{\mathrm{eff}}\neq 1$, the nonzero spectrum of \( \frac{1}{n_l}\X_m^\top \X_m \) is bounded away from zero almost surely. Equivalently, the nonzero singular values of $\X_m$ are of order $\sqrt p$ (see argument in Step 3). Hence
\[
\|\mathbf A\|
=
\|(\X_m\X_m^\top)^\dagger\|
=
O\!\left(\frac1p\right),
\]
and
\[
\|\mathbf A\|_F^2
=
\sum_{i:\sigma_i(\X_m)>0}
\frac{1}{\sigma_i(\X_m)^4}
=
O\!\left(\frac1p\right).
\]
Consequently, for all sufficiently large $p$,
\[
\mathbb P_\E\!\left(
\left|Q-\mathbb E_\E[Q\mid \X_m,\U_m]\right|
\ge \delta
\mid \X_m,\U_m
\right)
\le
2e^{-c_\delta p}.
\]
The right-hand side is summable in $p$, so Borel--Cantelli gives
\[
Q-\mathbb E_\E[Q\mid \X_m,\U_m]
\xrightarrow{a.s.} 0.
\]
Combining this concentration statement with the limit of the conditional
expectation yields
\[
Q
\xrightarrow{a.s.}
\bar w^\ast
\bar\sigma_{\mathrm{eff}}^2
\frac{\min\{1,\gamma_{\mathrm{eff}}\}}
{|\gamma_{\mathrm{eff}}-1|}.
\]
\paragraph{Combining.}
Summing the contributions from Steps 1--5 and applying the continuity mapping theorem concludes the proof.

\end{proof}

\subsubsection{Proof for the  generalisation error}
\label{app:proof_generalisation_error}

\begin{proof}[Proof of Theorem~\ref{thm:generalisation_error}]
We decompose the error and take averages over $\E$, $\Xm$, $\U_m$.
\paragraph{Step 1: Decomposition of the generalisation error.}
We start from the definition of the generalisation error
\begin{equation}
E^\mathrm{gen} = \mathbb{E}_{\x,y,\xi}\!\left[\|y - \x^\T \hat{\w}\|^2 \mid \wstar, \X_l, \X_u\right].
\end{equation}
Using the data model $y = \x^\T \wstar + \xi$ and independence of the test noise, this becomes
\begin{equation}
\begin{aligned}
E^\mathrm{gen}
&= \mathbb{E}_{\x,\xi}\!\left[(\wstar - \hat{\w})^\T \x \x^\T (\wstar - \hat{\w}) \mid \wstar, \X_l, \X_u\right] + \sigma^2 \\
&= \mathbb{E}_{\xi}\!\left[(\wstar - \hat{\w})^\T \bSigma_l (\wstar - \hat{\w}) \right] + \sigma^2.
\end{aligned}
\end{equation}
Using the spiked covariance model $\bSigma_l= \mathbf I_p +  (\lambda_l - 1)\bq \bq^\T$, we obtain
\begin{equation}
E^\mathrm{gen}
= E^\mathrm{est} + (\lambda_l - 1)\,\mathbb{E}_\xi\!\left[\left((\wstar - \hat{\w})^\T \bq\right)^2\right] + \sigma^2.
\end{equation}
Substituting $\hat{\w} = (\X_l \Pm)^\dagger \X_l \wstar + (\X_l \Pm)^\dagger \xi$ and expanding the quadratic form, the cross term vanishes in $\xi$ expectation and we obtain a bias contribution
\begin{equation}
\left((\wstar  - (\X_l \Pm)^\dagger \X_l \wstar)^\T \bq\right)^2
\end{equation}
and a variance contribution proportional to
\begin{equation}\label{eq:vart}
\bq^\T (\Pm \X_l^\T \X_l \Pm)^\dagger \bq.
\end{equation}
Using the same decomposition and conditional Gaussian argument as in Equation~\eqref{eq:noiseless_what_decomp}, $(\X_l\Pm)^\dagger \X_l \wstar = \Pi \wstar + A \Pi \bq + (\X_l\Pm)^\dagger \E \wstar_{\perp}$, we further decompose the bias 
\begin{equation}
(\wstar  - (\X_l \Pm)^\dagger \X_l \wstar)^\T \bq
=
\wstar^\T \bq
-
\bq^\top \Pi \wstar -  A \bq^\top \Pi \bq -  \bq^\top(\X_l\Pm)^\dagger \E \wstar_{\perp},
\end{equation}
with the squared term being:

\begin{equation*}
\begin{aligned}
    &\underbrace{(\wstar^\T \bq -
\bq^\top \Pi \wstar -  \frac{\lambda_l-1}{\tilde\lambda_l} (\wstar^\top \Pperp \bq) (\bq^\top \Pi \bq))^2}_{\E \text{--independent}} + \underbrace{(\bq^\top(\X_l\Pm)^\dagger \E \wstar_{\perp})^2 }_{\text{quadratic in } \E}\\
& -2\underbrace{(\wstar^\T \bq -
\bq^\top \Pi \wstar -  \frac{\lambda_l-1}{\tilde\lambda_l} (\wstar^\top \Pperp \bq) (\bq^\top \Pi \bq)) (\bq^\top(\X_l\Pm)^\dagger \E \wstar_{\perp}) }_{\text{cross term}}.
\end{aligned}
\end{equation*}

\paragraph{Step 2: $\E$-independent term.}
This component is made up of terms that have almost sure limits. In particular:
\begin{align*}
    \bq^\top \Pi \wstar \xrightarrow{a.s.} \sqrt{\bar w^\ast} \bar \pi_{\wstar, \bq} \qquad &\text{by Lemma~\ref{lem:projection_effective_bilinear}} \\
    \bq^\top \Pi \bq \xrightarrow{a.s.}  \bar \pi_{\bq, \bq} \qquad &\text{by Lemma~\ref{lem:projection_effective_bilinear}} \\
    \wstar^\top \Pperp \bq \xrightarrow{a.s.} \sqrt{\bar w^\ast}\bar p_{\perp, \wstar, \bq} \qquad &\text{by Lemma~\ref{lem:bilinear_projection}} \\
     \tilde \lambda_l \xrightarrow{a.s.} \bar \lambda_l  \qquad &\text{by Lemma~\ref{lem:pperpv}} \\
     \wstar^\top \bq \xrightarrow{a.s.} \sqrt{\bar w^\ast} r_q  \qquad &\text{by model assumptions}
\end{align*}
Thus, by the continuous mapping theorem, we have the almost sure convergence of the entire $\E$-independent term to 
\[
\bar w^\ast \left(r_q - \bar \pi_{\wstar, \bq} - \frac{\lambda_l-1}{\bar \lambda_l}\bar p_{\perp, \wstar, \bq} \bar \pi_{\bq, \bq}  \right)^2.
\]

\paragraph{Step 3: Cross term linear in $\E$.}

The linear $\E$ term has the form
\[
T = \mathbf c^\top \E \mathbf d,
\]
where
\[
\mathbf c = C\,(\X_l\Pm)^{\dagger\top}\bq,
\qquad
\mathbf d = \wstar_\perp,
\]
and $C$ is the deterministic prefactor defined in Step~2, which converges almost surely to an $O(1)$ limit.

This term has the same structure as the term linear in $\E$ in
Theorem~\ref{thm:estimation_error}, Step~3, with the vector
\[
\wstar - \Pi\wstar - A\Pi\bq
\]
replaced by the unit vector $\bq$ and multiplied by the scalar $C = O(1)$.
The same arguments apply and we conclude that
\[
T \xrightarrow{a.s.} 0.
\]
\paragraph{Step 4: Quadratic term in $\E$.}

Define
\[
Q
:=
\wstar_\perp^\top \E^\top
(\X_l\Pm)^{\dagger\top}
\bq \bq^\top
(\X_l\Pm)^\dagger
\E \wstar_{\perp}.
\]
Let
\[
\ba := \wstar_\perp,
\qquad
\A
:=
(\X_l\Pm)^{\dagger\top}
\bq \bq^\top
(\X_l\Pm)^\dagger.
\]
Conditioned on $(\X_m,\U_m)$, the vector $\ba$ and matrix $\A$ are independent of $\E$, so $Q = \ba^\top \E^\top \A \E \ba$ is a quadratic Gaussian form.

This is exactly the same object as in Theorem~\ref{thm:estimation_error}, Step~5, with a different choice of $\A$. The concentration argument therefore carries over verbatim once we verify that the same norm bounds hold.

First, as in Theorem~\ref{thm:estimation_error}, Step~3,
\[
\ba^\top \bar\bSigma \ba = O(1)
\qquad \text{almost surely.}
\]

Next, since $\A$ is rank one,
\[
\|\A\|
=
\|(\X_l\Pm)^{\top\dagger} \bq\|^2,
\qquad
\|\A\|_F^2
=
\|\A\|^2.
\]
Using the bounds on $\|(\X_l\Pm)^\dagger\|_{\mathrm{op}}$ established in Theorem~\ref{thm:estimation_error}, Step~3,
together with $\|\bq\| = 1$, we obtain
\[
\|\A\| = O\!\left(\frac{1}{p}\right),
\qquad
\|\A\|_F^2 = O\!\left(\frac{1}{p^2}\right),
\qquad \text{almost surely.}
\]
The conditions required for the Hanson--Wright concentration used in Theorem~\ref{thm:estimation_error}, Step~5 are satisfied. Therefore, by the same argument as in Theorem~\ref{thm:estimation_error}, Step~5,
\[
Q - \mathbb E_\E[Q \mid \X_m,\U_m]
\xrightarrow{a.s.} 0.
\]

Finally, the conditional expectation is
\[
\mathbb E_\E[Q \mid \X_m,\U_m] = (\ba^\top \bar\bSigma \ba)\Tr(\A) = O\left(\frac{1}{p}\right) \xrightarrow{a.s.} 0
\]

\paragraph{Step 5: Variance contribution.}
Finally, we consider the term from Equation \eqref{eq:vart} and note it is the same as the norm of the matrix $\mathbf A$ from Step 4:
\[
\bq^\T (\Pm \X_l^\T \X_l \Pm)^\dagger \bq = \|(\X_l\Pm)^{\dagger\top} \bq\|^2 = \| \mathbf A\| = O \left( \frac{1}{p}\right).
\]
Thus, this term converges to 0:
\[
\bq^\T (\Pm \X_l^\T \X_l \Pm)^\dagger \bq \xrightarrow{a.s.} 0.
\]
\paragraph{Combining.}
Thus, in the high-dimensional limit the only additional contribution relative to the estimation error is the spike bias term, and the generalisation error converges to
\begin{equation*}
E^\mathrm{gen}_\infty
=
E^\mathrm{est}_\infty
+
\bar w^\ast (\lambda_l- 1)
\left( r_q - \bar \pi_{\wstar, \bq} - \frac{\lambda_l-1}{\bar \lambda_l}\bar p_{\perp, \wstar, \bq} \bar \pi_{\bq, \bq}  \right)^2
+ \sigma^2.
\end{equation*}
\end{proof}


\subsubsection{Proof for the training error}
\label{app:proof_training_error}
\begin{proof}[Proof of Theorem~\ref{thm:training_error}]
Let
\[
\mathbf{H} = (\X_l \Pm)(\X_l \Pm)^\dagger
\]
and recall that the target vector is $\y_d = \X_l \wstar + \bxi$. Then, the residual vector is:
\begin{equation}
\begin{aligned}
    \y_d - \X_l \what &= \y_d - (\X_l \Pm) (\X_l \Pm)^\dagger \y_d - (\X_l \Pperp) (\X_l \Pm)^\dagger \y_d\\
&= (\mathbf{I}_{n_l} - \mathbf{H}) \y_d \\
&= (\mathbf{I}_{n_l} - \mathbf{H})(\X_l \wstar + \bxi),
\end{aligned}
\end{equation}
where we note that the term multiplying two projections is 0:
\[
(\X_l \Pperp) (\X_l \Pm)^\dagger = \X_l \U_\perp \underbrace{\U_\perp^\top \U_m}_0 (\X_l \U_m)^\dagger  = 0.
\]
The training error splits into two terms since the noise $\bxi$ is independent and zero-mean:
\begin{equation}
\begin{aligned}
E^{\mathrm{train}}
&= \frac{1}{n_l} \mathbb{E}_{\bxi} \left[
\left\| (\mathbf{I}_{n_l} - \mathbf{H}) \X_l \wstar + (\mathbf{I}_{n_l} - \mathbf{H}) \bxi \right\|^2
\right] \\
&= \underbrace{\frac{1}{n_l} \left\| (\mathbf{I}_{n_l} - \mathbf{H}) \X_l \wstar \right\|^2}_{\text{Bias}^2}
+ \underbrace{\frac{1}{n_l} \mathbb{E}_{\bxi} \left[
\left\| (\mathbf{I}_{n_l} - \mathbf{H}) \bxi \right\|^2
\right]}_{\text{Variance}}.
\end{aligned}
\end{equation}
We start by computing the bias. 
Using the decomposition $\X_l = \X_l \Pm + \X_l \Pperp$ and the fact that $(\mathbf{I} - \mathbf{H}) \X_l \Pm = 0$, the bias simplifies to:
\begin{equation}
\text{Bias}^2
= \frac{1}{n_l} \left\| (\mathbf{I}_{n_l} - \mathbf{H}) \X_l \Pperp \wstar \right\|^2.
\end{equation}

Substituting the conditional Gaussian decomposition, we have
\[
\X_l \Pperp
= \Xm \bSigma_{mm}^{-1} \bSigma_{m\perp} \U_\perp^\top + \E \U_\perp^\top,
\]
where $\E$ is independent of $\Xm$. Thus, we obtain
\begin{equation}
\begin{aligned}
(\mathbf{I}_{n_l} - \mathbf{H}) \X_l \Pperp \wstar
&= (\mathbf{I}_{n_l} - \mathbf{H}) \Xm \bSigma_{mm}^{-1} \bSigma_{m\perp} \U_\perp^\top \wstar
+ (\mathbf{I}_{n_l} - \mathbf{H}) \E \U_\perp^\top \wstar \\
&= (\mathbf{I}_{n_l} - \mathbf{H}) \E \wstar_\perp,
\end{aligned}
\end{equation}
where we used $(\mathbf{I}_{n_l} - \mathbf{H}) \Xm = 0$ and defined $\wstar_\perp = \U_\perp^\top \wstar$. Note that the entries of $\E \wstar_\perp$ are i.i.d.\ $\mathcal{N}(0, \|\wstar\|^2 \sigma_{\mathrm{eff}}^2)$, with 
\begin{equation}
\sigma_{\mathrm{eff}}^2
= \frac{\|\Pperp \wstar\|^2}{\|\wstar\|^2}
+ \frac{\lambda_l - 1}{\lamtilde_l} \frac{(\bq^\top \Pperp \wstar)^2}{\|\wstar\|^2}.
\end{equation}

Conditioned on $\Xm$, $\mathbf{H}$ is deterministic, and the term
\(\|(\mathbf{I}_{n_l} - \mathbf{H}) \E \wstar_\perp\|^2\)
is a quadratic form $z^\top (\mathbf{I}_{n_l} - \mathbf{H}) z$ where
$z \sim \mathcal{N}(0, \|\wstar\|^2\sigma_{\mathrm{eff}}^2 \mathbf{I}_{n_l})$. Hence, we have
\begin{equation}
\mathbb{E}_{\E} \left[ \text{Bias}^2 \mid \Xm \right]
= \|\wstar\|^2\frac{\sigma_{\mathrm{eff}}^2}{n_l} \operatorname{Tr}(\mathbf{I}_{n_l} - \mathbf{H})
= \|\wstar\|^2\sigma_{\mathrm{eff}}^2 \left( 1 - \frac{\min(m, n_l)}{n_l} \right),
\end{equation}
where the last equality holds almost surely (w.r.t.\ the randomness in $\X_m$). 

Next, to compute the variance, recall that the label noise $\bxi$ is isotropic $\mathcal{N}(0, \sigma^2 \mathbf{I}_{n_l})$. Thus, we have
\begin{equation}
\frac{1}{n_l} \mathbb{E}_{\bxi} \left[
\bxi^\top (\mathbf{I}_{n_l} - \mathbf{H}) \bxi
\right]
= \frac{\sigma^2}{n_l} \operatorname{Tr}(\mathbf{I}_{n_l} - \mathbf{H})
= \sigma^2 \left( 1 - \frac{\min(m, n_l)}{n_l} \right),
\end{equation}
where the last equality again holds almost surely (w.r.t.\ the randomness in $\X_m$).
Note that $$1 - \frac{\min(m, n_l)}{n_l}=\max(0, 1 - \gamma_{\mathrm{eff}}).$$ To conclude, it suffices to recall from the proof of Theorem \ref{thm:estimation_error} that $\sigma_{\mathrm{eff}}^2 \xrightarrow{\mathrm{a.s.}} \bar{\sigma}_{\mathrm{eff}}^2$, and the argument is complete.
\end{proof}

\subsubsection{Auxiliary lemmas}\label{app:projections}

\begin{lemma}[Limiting spectral distribution of $\Xm$]\label{lem:esd_projected_x}
Consider the setup of Section~\ref{sec:setup}, and  define
\[
\X_{m}:=\X_l\U_{m}.
\]
Then the empirical spectral measure of \(\frac1{n_l}\X_{m}^\top \X_{m}\) converges weakly almost surely to the Marchenko--Pastur law with aspect ratio $\gamma_{\mathrm{eff}}$, i.e.,
\[
\hat F^{\frac1{n_l}\X_{m}^\top \X_{m}}
\Rightarrow
F_{\gamma_{\mathrm{eff}}}
\qquad\text{almost surely.}
\]

The bounded-$m$ case, including $m=1$, is discussed in the remark following the proof.
\end{lemma}

\begin{proof}
Fix a realization of the full sequence $(\U_{m})_{p\ge 1}$. Conditional on this realization, the rows of \(
\X_{m}=\X_l\U_{m}
\) are i.i.d.\ Gaussian in $\mathbb R^{m}$ with covariance
\[
\bSigma_{mm}:=\U_{m}^\top \bSigma_l \U_{m}.
\]
Using the form of $\bSigma_l$ from Assumption \ref{ass:spike_cov}, we obtain
\[
\bSigma_{mm}
=
\U_{m}^\top\Bigl(\mathbf I_{p}+(\lambda_l-1)\bq\bq^\top\Bigr)\U_{m}
=
\mathbf I_{m}+(\lambda_l-1)(\U_{m}^\top\bq)(\U_{m}^\top\bq)^\top.
\]
Hence, $\bSigma_{mm}-\I_{m}$ has rank at most one. Since $\bSigma_{mm} \succ 0$, define $\mathbf G\in\mathbb R^{n_l\times m}$ and $\mathbf W\in\mathbb R^{m \times m}$ by
\[
\mathbf G := \X_{m} \bSigma_{mm}^{-1/2} \qquad \mathbf W:=\frac1{n_l}\mathbf G^\top\mathbf G.
\]
Conditional on $\U_{m}$, $\mathbf G$ has i.i.d.\ $\mathcal N(0,1)$ entries. The sample covariance can be rewritten:
\[
\frac1{n_l}\X_{m}^\top \X_{m}
=
\bSigma_{mm}^{1/2}
\left(\frac1{n_l}\mathbf G^\top\mathbf G\right)
\bSigma_{mm}^{1/2} = \bSigma_{mm}^{1/2}
\mathbf W
\bSigma_{mm}^{1/2}.
\]
Since $\bSigma_{mm}=\I_{m}+\mathbf R$ with $\operatorname{rank}(\mathbf R)\le 1$, the spectral decomposition shows that $\bSigma_{mm}^{1/2}-\I_{m}$ also has rank at most one. Consequently,
\[
\bSigma_{mm}^{1/2}\mathbf W\bSigma_{mm}^{1/2}-\mathbf W
=
(\bSigma_{mm}^{1/2}-\I_{m})\mathbf W\bSigma_{mm}^{1/2}
+
\mathbf W(\bSigma_{mm}^{1/2}-\I_{m}),
\]
and thus
\[
\operatorname{rank}\!\left(
\bSigma_{mm}^{1/2}\mathbf W\bSigma_{mm}^{1/2}-\mathbf W
\right)\le 2.
\]
Let $\hat F^\A$ denote the empirical spectral distribution function of a real symmetric matrix $\A$. By the rank inequality (e.g., see Theorem A.43 in \citep{bai_spectral_2010}),
\[
\sup_{\theta\in\mathbb R}
\left|
\hat F^{\frac1{n_l}\X_{m}^\top\X_{m}}(\theta)-\hat F^{\mathbf W}(\theta)
\right|
\le \frac{2}{m} \underset{p\to \infty}{\to} 0.
\]
Since $\mathbf W$ is a standard Wishart matrix of size $m\times m$ with aspect ratio  \(\frac{m}{n_l}\to \gamma_{\mathrm{eff}}\), by the Marchenko--Pastur theorem \citep{marcenko_distribution_1967},
\[
\hat F^{\mathbf W}\Rightarrow F_{\gamma_{\mathrm{eff}}}
\qquad\text{almost surely, conditional on } (\U_{m})_{p\ge 1}.
\]
 Let \(
A:=
\left\{
F^{\frac1{n_l}\X_{m}^\top\X_{m}}
\Rightarrow
F_{\gamma_{\mathrm{eff}}}
\right\}.
\) Taken together, the arguments above shows that:
\[
\mathbb P\!\left(A\,\middle|\,(\U_{m})_{p\ge 1}\right)=1 .
\]
We now remove the conditioning on $(\U_{m})_{p\ge 1}$. By the law of total probability,
\[
\mathbb P(A)
=
\mathbb E\!\left[
\mathbb P\!\left(A\,\middle|\,(\U_{m})_{p\ge 1}\right)
\right]
=
\mathbb E[1]
=1.
\]
This proves that
\[
\hat F^{\frac1{n_l}\X_{m}^\top\X_{m}}
\Rightarrow
F_{\gamma_{\mathrm{eff}}}
\qquad\text{almost surely.}
\]
\end{proof}
\begin{remark}[Fixed projected dimension]
The assumption $m\to\infty$ ensures that the population spike
has vanishing weight in the empirical spectral distribution.
For fixed $m$ and $n_l\to\infty$, one instead uses consistency
of the sample covariance:
\[
\left\|
\frac1{n_l}\X_m^\top\X_m-\bSigma_{mm}
\right\|_{\mathrm{op}}\xrightarrow{a.s.}0.
\]
Consequently, if
\[
\tilde\lambda_l
:=1+(\lambda_l-1)\|\U_m^\top\bq\|^2
\xrightarrow{\mathrm{a.s.}}\bar\lambda_l
\]
then
\[
\hat F^{\frac1{n_l}\X_m^\top\X_m}
\Rightarrow
\frac{m-1}{m}\delta_1+\frac1m\delta_{\bar\lambda_l}.
\]
Thus, for $m>1$, the non-spike part of the spectrum
converges to $1$, consistently with $F_0=\delta_1$.
For $m=1$, there is only the spiked eigenvalue
and no bulk.
\end{remark}

\begin{lemma}[Deterministic equivalent for the pseudoinverse trace]
\label{lem:trace_pinv_det}
Consider the setup of Section~\ref{sec:setup}, and define
\[
\X_{m}:=\X_l\U_{m}.
\]
Then
\[
\mathrm{Tr}\!\big[(\X_{m}^\top\X_{m})^\dagger\big]
\xrightarrow{a.s.}
\begin{cases}
\displaystyle \frac{\gamma_{\mathrm{eff}}}{1-\gamma_{\mathrm{eff}}}, & \gamma_{\mathrm{eff}}<1,\\[8pt]
\displaystyle \frac{1}{\gamma_{\mathrm{eff}}-1}, & \gamma_{\mathrm{eff}}>1.
\end{cases}
\]
\end{lemma}

\begin{proof}
We have
\begin{equation}\label{eq:trace_rescale}
\Tr\!\big[(\X_{m}^\top \X_{m})^\dagger\big]
= \frac{1}{n_l}\Tr\Bigl((\frac{1}{n_l}\X_{m}^\top \X_{m})^\dagger\Bigr)
= \frac{m}{n_l}\cdot \frac{1}{m}\Tr\Bigl((\frac{1}{n_l}\X_{m}^\top \X_{m})^\dagger\Bigr).
\end{equation}

By Lemma~\ref{lem:esd_projected_x}, the empirical spectral distribution of \(\frac{1}{n_l}\X_{m}^\top \X_{m}\) converges weakly almost surely to the Marcenko--Pastur law with aspect ratio \(\gamma_{\mathrm{eff}}\). If \(\gamma_{\mathrm{eff}}< 1\), its support is bounded away from zero, and the claim follows by continuity of \(x\mapsto 1/x\). If \(\gamma_{\mathrm{eff}} > 1\), there is an additional atom at zero; however, the pseudo-inverse removes this null space contribution, and the nonzero spectrum remains bounded away from zero, so the same argument applies on the effective support. Therefore, 
\[
\frac{1}{m}\Tr\Bigl(\Bigl(\frac{1}{n_l}\X_{m}^\top \X_{m}\Bigr)^\dagger\Bigr)
= \int_{(0,\infty)} \frac{1}{x}\, dF^{\frac{1}{n_l}\X_{m}^\top \X_{m}}(x)
\xrightarrow{a.s.}
\int_{(0,\infty)} \frac{1}{x}\, dF_{\gamma_{\mathrm{eff}}}(x).
\]
Recall that
\[
\int \frac{1}{x}\, dF_{\gamma}
=
\begin{cases}
\displaystyle \frac{1}{1-\gamma}, & \gamma<1,\\[6pt]
\displaystyle \frac{1}{\gamma(\gamma-1)}, & \gamma>1.
\end{cases}
\]
Substituting into \eqref{eq:trace_rescale} and using \(\gamma =\gamma_{\mathrm{eff}}\) gives
\[
\Tr\!\big[(\X_{m}^\top \X_{m})^\dagger\big]
\xrightarrow{a.s.}
\begin{cases}
\displaystyle \frac{\gamma_{\mathrm{eff}}}{1-\gamma_{\mathrm{eff}}}, & \gamma_{\mathrm{eff}}<1,\\[8pt]
\displaystyle \frac{1}{\gamma_{\mathrm{eff}}-1}, & \gamma_{\mathrm{eff}}>1,
\end{cases}
\]
which concludes the proof.
\end{proof}

\begin{lemma}[Spectral measure of the overlap]
\label{lem:overlap_measure}
Consider the setup of Section~\ref{sec:setup}, and let $(\bu_i)_{i=1}^p$ be the eigenvectors of $\Ss_u = \frac{1}{n_u} \X_u^\top \X_u$. Let $\ba \in \mathbb R^p$ be a fixed vector with population mass distribution $G_\ba$ on the eigenvectors of $\bSigma_u$, i.e., $G_\ba = \frac{(\ba^\top \bv)^2}{\|\ba\|^2}\delta_{\lambda_u} + \big(1 - \frac{(\ba^\top \bv)^2}{\|\ba\|^2}\big) \delta_1$. Define the empirical CDF 
\[
  \hat G^\mathrm{sample}_\ba (\theta) = \frac{1}{\| \ba\|^2} \sum_{i=1}^p (\ba^\top \bu_i)^2 \;\mathbf 1 \{ \lambda_i(\Ss_u) \leq\theta \}.
\]
Then, $\hat G^\mathrm{sample}_\ba  \Rightarrow G^\mathrm{sample}_\ba$ almost surely, where $G^\mathrm{sample}_\ba$ is such that \emph{(i)}
    its support satisfies $\mathrm{supp}(G^\mathrm{sample}_\ba) \subseteq \mathcal S_{\gamma_u} \cup \{\theta^* \} \cup \{0\}$ with $\mathcal S_{\gamma_u}$ being the support of the Marcenko-Pastur density with  aspect ratio $\gamma_u$;
 \emph{(ii)} it has a continuous density $g^{\mathrm{sample}}_\ba(\theta)$ on the interior of $\mathcal S_{\gamma_u}$
\[
g^{\mathrm{sample}}_\ba(\theta) =\left(\int\frac{\tau\,\gamma_u}
         {\tau^2-\tau(1-\gamma_u+\theta)+\theta}
    \,dG_\ba(\tau)
    \right) f_{\gamma_u}(\theta);
\]
\emph{(iii)} it may have an outlier atom at $\{\theta^*\}$ where \( \displaystyle \theta^* = \frac{\lambda_u(\lambda_u - 1 + \gamma_u)}{\lambda_u - 1}
\) is the location of the spike $\lambda$ in the sample covariance spectrum
\[
 G^\mathrm{sample}_\ba(\{\theta^*\}) = \begin{cases} \frac{1 - \frac{\gamma_u}{(\lambda_u - 1)^2}}{1 + \frac{\gamma_u}{(\lambda_u - 1)}} G_\ba(\{\lambda_u\}) &\lambda_u > 1 + \sqrt{\gamma_u} , \\
 0 &\lambda_u \leq 1 + \sqrt{\gamma_u}
 \end{cases};
\]
and \emph{(iv)} it may have an atom at 0 of weight 
\[
G_\ba^{\mathrm{sample}}(\{0\})
=
\begin{cases}
0,&\gamma_u\le1,\\[4pt]
\displaystyle
\int\frac{\gamma_u-1}{\gamma_u-1+\tau}\,dG_\ba(\tau),
&\gamma_u>1.
\end{cases}
\]
\end{lemma}

\begin{proof}
    The result follows by specializing Lemma 3 in \cite{green_high-dimensional_2025} to the spiked covariance model of Assumption \ref{ass:spike_cov}.
\end{proof}

\begin{lemma}[Deterministic limit of $\|\Pm \ba \|^2, \| \Pperp \ba\|^2$]
\label{lem:pperpv}
Consider the setup of Section~\ref{sec:setup}, and let $\ba \in \mathbb{R}^p$ be a fixed vector with population mass distribution $G_\ba$ on the eigenvectors of $\bSigma_u$. Let $\Pm$ be the projection onto the top $m$ eigenvectors of the prior sample covariance $\Ss_u$, with $m/p \to \alpha$.
Let $G^\mathrm{sample}_\ba$ denote the deterministic limiting weighted spectral measure defined in Lemma~\ref{lem:overlap_measure}. When $0<\alpha<\min\{1,1/\gamma_u\}$, let $\lambda_t$ be the
unique bulk threshold satisfying
\[
F_{\gamma_u}(\lambda_t)=1-\alpha.
\]
For $\gamma_u\le1$ and $\alpha=1$, set $\lambda_t=\lambda_-$. For $\alpha=0, \lambda_t = \lambda_+$
When $\gamma_u>1$ and $\alpha\ge1/\gamma_u$, no bulk threshold
is needed: all positive-eigenvalue directions are retained,
and the nullspace-completion rule determines the discarded mass.
Then, almost surely,
\begin{equation} \label{eq:pperpv}
\frac{\|\Pperp \ba\|^2}{\|\ba\|^2}
\;\xrightarrow{a.s.}\;
\begin{cases}
\displaystyle
\int_{\lambda_-}^{\lambda_t} g_\ba^\mathrm{sample}(\lambda) d\lambda,
& \gamma_u \le 1, \\[10pt]
\displaystyle
G_\ba^\mathrm{sample}(\{0\}) + \int_{\lambda_-}^{\lambda_t} g_\ba^\mathrm{sample}(\lambda) d\lambda,
& \gamma_u > 1,\;\alpha < \tfrac{1}{\gamma_u}, \\[10pt]
\displaystyle
\frac{1-\alpha}{1-\frac{1}{\gamma_u}}\, G_\ba^\mathrm{sample}(\{0\}),
& \gamma_u > 1,\;\alpha \ge \tfrac{1}{\gamma_u},
\end{cases}
\end{equation}
where $\lambda_- = (1-\sqrt{\gamma_u})^2$ is the infimum of the support of the Marcenko-Pastur distribution as defined in Equation \eqref{eq:mp_definition} and
\begin{equation}
\|\Pm \ba\|^2 = \|\ba\|^2 - \|\Pperp \ba\|^2.
\end{equation}
\end{lemma}

\begin{proof}
Let
\[
\hat G_\ba^\mathrm{sample} (\theta)
=
\frac{1}{\|\ba\|^2}\sum_{i=1}^p (\ba^\top  \bu_i)^2 \,\mathbf 1 \{ \lambda_i (\Ss_u) \leq \theta\}
\]
be the weighted empirical spectral measure of $\Ss_u$. By Lemma~\ref{lem:overlap_measure}, $\hat G_\ba^\mathrm{sample}$ converges almost surely to the deterministic measure $G_\ba^\mathrm{sample}$.

Since $\|\Pperp \ba\|^2/\|\ba\|^2$ is exactly the mass of $\hat G_\ba^\mathrm{sample}$ on the discarded eigenspaces, its limit is obtained by integrating $G_\ba^\mathrm{sample}$ over the discarded part of the spectrum.

If $\gamma_u\le1$, the limiting weighted spectral measure has no atom at zero, so the discarded mass is given by the bulk integral. This gives
\[
\frac{\|\Pperp \ba\|^2}{\|\ba\|^2}
\xrightarrow{a.s.}
\int_{\lambda_-}^{\lambda_t} g_\ba^\mathrm{sample}(\lambda) d\lambda.
\]

If $\gamma_u>1$ and $\alpha<1/\gamma_u$, fewer than all positive-eigenvalue directions are retained, so the entire nullspace is discarded as well. Hence
\[
\frac{\|\Pperp \ba\|^2}{\|\ba\|^2}
\xrightarrow{a.s.}
G_\ba^\mathrm{sample}(\{0\})+\int_{\lambda_-}^{\lambda_t} g_\ba^\mathrm{sample}(\lambda) d\lambda.
\]

Finally, if $\gamma_u>1$ and $\alpha\ge 1/\gamma_u$, all positive-eigenvalue directions are retained, and the retained subspace is completed by adding
\[
m-\operatorname{rank}(\Ss_u) = m - n_u
\]
orthonormal directions chosen uniformly at random in the nullspace, independently of $\ba$. Therefore only the fraction
\[
\frac{p-m}{p-n_u}
\;\xrightarrow{a.s.}\;
\frac{1-\alpha}{1-\frac{1}{\gamma_u}}
\]
of the nullspace remains discarded. Since the total asymptotic mass of $\ba$ in the nullspace is $G_\ba^\mathrm{sample}(\{0\})$, we obtain
\[
\frac{\|\Pperp \ba\|^2}{\|\ba\|^2}
\xrightarrow{a.s.}
\frac{1-\alpha}{1-\frac{1}{\gamma_u}}\,G_\ba^\mathrm{sample}(\{0\}).
\]
The identity \(
\|\Pm \ba\|^2 = \|\ba\|^2-\|\Pperp \ba\|^2
\) is immediate.
\end{proof}


\begin{lemma}[Deterministic limit of $\ba^\top \Pm \bb, \ba^\top \Pperp \bb$]
\label{lem:bilinear_projection}
Consider the setup of Section~\ref{sec:setup}, let $\ba,\bb\in\mathbb R^p$ be fixed vectors and let $\Pm$ denote the projection onto the top $m$ principal components of $\Ss_u$, with $m/p\to\alpha$.
Define the signed \emph{population} spectral measure
\begin{equation}
G_{\ba, \bb}(\tau)
= \frac{1}{\| \ba\|\| \bb\|}
\sum_j (\ba^\top \bv_j)(\bb^\top \bv_j)\,\mathbf 1\{\lambda_{u,j}\le \tau\},
\end{equation}
where $(\bv_j,\lambda_{u,j})$ are the population eigenpairs of $\bSigma_u$.
Let $G^\mathrm{sample}_{\ba, \bb}$ denote the corresponding deterministic limiting signed \emph{sample} spectral measure onto the eigenpairs of $\Ss_u$. Use the same bulk-threshold convention and Haar-uniform
nullspace-completion rule as in Lemma~\ref{lem:pperpv}. Then,
\begin{equation}
\ba^\top \Pm \bb
=
\ba^\top \bb-\ba^\top \Pperp \bb,
\end{equation}
and, almost surely,
\begin{equation}\label{eq:bilinear_proj}
\frac{\ba^\top \Pperp \bb}{\|\ba\| \| \bb\|}
\xrightarrow{a.s.}
\begin{cases}
\displaystyle \int_{\lambda_-}^{\lambda_t} g^\mathrm{sample}_{\ba, \bb}(\lambda) d\lambda,
& \gamma_u\le1,\\[8pt]
\displaystyle G^\mathrm{sample}_{\ba, \bb}(\{0\})+\int_{\lambda_-}^{\lambda_t} g^\mathrm{sample}_{\ba, \bb}(\lambda) d\lambda,
& \gamma_u>1,\ \alpha<1/\gamma_u,\\[8pt]
\displaystyle \frac{1-\alpha}{1-\frac1{\gamma_u}}\,G^\mathrm{sample}_{\ba, \bb}(\{0\}),
& \gamma_u>1,\ \alpha\ge 1/\gamma_u.
\end{cases}
\end{equation}
Moreover, $G^\mathrm{sample}_{\ba, \bb}(\{0\})$ and the bulk density of $g^\mathrm{sample}_{\ba, \bb}$ are given by the formulas of Lemma~\ref{lem:overlap_measure} with $dG_\ba$ replaced by the signed measure $dG_{\ba,\bb}$.
\end{lemma}
\begin{proof}
Define the signed mass density onto the \emph{sample} eigenvectors of $\Ss_u$ 
\begin{equation}
\hat G_{\ba,\bb}^\mathrm{sample} (\theta)
:=
\frac{1}{\|\ba\| \| \bb\|}\sum_{i=1}^p (\ba^\top \bu_i)(\bb^\top \bu_i)\,\mathbf 1 \{ \lambda_i(\Ss_u) \leq \theta\},
\end{equation}
so that its Stieltjes transform is
\begin{equation}
m_{\hat G_{\ba,\bb}^\mathrm{sample}}(z)
=
\int \frac{1}{\lambda-z}\,d\hat G_{\ba,\bb}^\mathrm{sample}(\lambda)
=
\frac{\ba^\top(\Ss_u-zI)^{-1}\bb}{\|\ba\|\|\bb\|}
=
\frac{\Tr\!\big[\ba\bb^\top(\Ss_u-zI)^{-1}\big]}{\|\ba\|\|\bb\|}.
\end{equation}
By \citet[Theorem 1]{rubio2011Spectral}, for any deterministic matrix $\Theta$ with bounded trace norm (i.e., $\Tr\!\big[(\Theta^\top \Theta)^{1/2}\big] = \| \Theta \|_\mathrm{tr}$),
\[
\Tr\!\big[\Theta(\Ss_u-zI)^{-1}\big]-\delta_\Theta(z)\to 0
\qquad\text{a.s.}
\]
where $\delta_\Theta(z)$ is a deterministic function of $z$. Applying this with $\Theta=\ba\bb^\top / \|\ba\|\|\bb\|$ and repeating the argument from \cite{green_high-dimensional_2025} (in their paper, see Lemma 1 proven in Appendix A) yields the following almost sure weak convergence
\[
\hat G_{\ba,\bb}^\mathrm{sample} \Rightarrow G_{\ba,\bb}^\mathrm{sample},
\]
where the Stieltjes transform of $G_{\ba,\bb}^\mathrm{sample}$ is
\begin{equation}
m_{G_{\ba,\bb}^\mathrm{sample}}(z)
=
-\frac1z\int \frac{1}{1+\tau\,\underline m_{\gamma_u}(z)}\,dG_{\ba,\bb}(\tau).
\end{equation}
Here, the companion Stieltjes transform is defined as $\underline m_{\gamma_u}(z) = \gamma_u m_{\gamma_u}(z) - (1-\gamma_u)\frac{1}{z}$, and $m_{\gamma_u}$ is a Stieltjes transform of a Marcenko-Pastur density $f_{\gamma_u}$. 
This is exactly the same expression as in Lemma 1 in \cite{green_high-dimensional_2025} with the population measure $G_\ba$ replaced by the signed measure $G_{\ba,\bb}$. Hence, the atom at $0$ and the bulk density of $g_{\ba,\bb}^\mathrm{sample}(\{0\})$ are obtained from the formulas of Lemma~\ref{lem:overlap_measure} by the same substitution. Now the mass of $\hat G_{\ba,\bb}^\mathrm{sample}$ on the discarded eigenspaces is 
\begin{equation}
\ba^\top \Pperp \bb
=
\sum_{i = m+1}^{p} (\ba^\top \bu_i)(\bb^\top \bu_i)
\end{equation}
noting that the eigenvectors $\bu_i$ are sorted according to their eigenvalues, such that the set $\{m+1, \dots, p\}$ contains the indices of the $p-m$ smallest eigenvalues. 
Therefore, the same geometric argument as in Lemma~\ref{lem:pperpv} gives the three regimes above: discarded bulk only when $\gamma_u\le1$; bulk plus the full atom at $0$ when $\gamma_u>1$ and $\alpha<1/\gamma_u$; and, when $\gamma_u>1$ and $\alpha\ge 1/\gamma_u$, only the fraction
\[
\frac{1-\alpha}{1-\frac1{\gamma_u}}
\]
of the nullspace mass remains discarded, by the same random nullspace-completion argument. Finally,
\[
\ba^\top \Pm \bb=\ba^\top \bb-\ba^\top \Pperp \bb,
\]
which proves the result.
\end{proof}


\begin{lemma}[Deterministic limit of the effective projection $\|\Pi \ba\|^2$]
\label{lem:projection_effective}
Consider the setup of Section~\ref{sec:setup}, let \(\X_m = \X_l\U_m \in \mathbb{R}^{n_l \times m}\) and define
\[
\Pi_m := \X_m^\dagger \X_m \in \mathbb{R}^{m \times m},
\qquad
\Pi := \U_m \Pi_m \U_m^\top \in \mathbb{R}^{p \times p}.
\]
Let the eigendecomposition of $\bSigma_{mm}$ be \(\bSigma_{mm} := \U_m^\top \bSigma_l \U_m\) with \((\tilde \bq_i,\tilde\lambda_i)\) denoting the eigenpairs. For any fixed vector \(\ba \in \mathbb{R}^p\) with \(\Pm \ba \neq 0\), define
\[
\tilde \ba := \frac{\U_m^\top \ba}{\|\Pm \ba\|} \in \mathbb{R}^m,
\]
and let \(\tilde G_{\tilde \ba}\) be the mass distribution of \(\tilde \ba\) onto the eigenvectors of \(\bSigma_{mm}\):
\[
\tilde G_{\tilde \ba}(\tilde\tau)
=
\sum_i (\tilde \ba^\top \tilde \bq_i)^2\,\mathbf{1}\{\tilde\lambda_i \le \tilde\tau\}.
\]
Then, it holds almost surely
\[
\tilde G_{\tilde \ba}(\tau) \Rightarrow \left(1 - \frac{\bar p_{\ba, \bq}^2}{\bar p_{\ba, \ba} \bar p_{\bq, \bq}}\right) \delta_{1} +  \frac{\bar p_{\ba, \bq}^2}{\bar p_{\ba, \ba} \bar p_{\bq, \bq}} \delta_{\bar \lambda_l},
\]
and
\begin{equation} \label{eq:limit_pitilde_projections}
\frac{\ba^\top \Pi \ba}{\|\ba\|^2} \xrightarrow{a.s.}
\bar p_{\ba, \ba}\cdot
\begin{cases}
1, & \gamma_\mathrm{eff} < 1, \\[6pt]
1 - \displaystyle\int
\frac{\gamma_{\mathrm{eff}} - 1}{\gamma_{\mathrm{eff}} - 1 + \tilde \tau}
\, d\tilde G_{\tilde \ba}(\tilde \tau), & \gamma_\mathrm{eff} > 1.
\end{cases}
\end{equation}
with $\bar p_{\ba, \ba}, \bar p_{\bq, \bq}, \bar p_{\ba, \bq}$ defined via Equations~\eqref{eq:pperpv} and \eqref{eq:bilinear_proj}.
\end{lemma}

\begin{proof}
We start by computing explicitly $\tilde G_{\tilde \ba}$. Recall that
\[
\bSigma_{mm}
=
\U_m^\top \bSigma_l \U_m
=
\mathbf I_m+(\lambda_l-1)(\U_m^\top \bq)(\U_m^\top \bq)^\top.
\]
Then, $\bSigma_{mm}$ has $m-1$ eigenvalues equal to $1$ and the remaining one is
\[
\tilde\lambda_l
= 1+(\lambda_l-1)\|\Pm \bq\|^2 \quad\Rightarrow \quad \bar\lambda_l = 1 + (\lambda_l-1)\bar p_{\bq, \bq},
\]
where the convergence occurs almost surely. The corresponding unit eigenvector is \(\tilde \bq
=
\frac{\U_m^\top \bq}{\|\U_m^\top \bq\|}
=
\frac{\U_m^\top \bq}{\|\Pm \bq\|}.
\)
Therefore the spectral measure \(\tilde G_{\tilde \ba}\) of
\[
\tilde \ba=\frac{\U_m^\top \ba}{\|\Pm \ba\|}
\]
with respect to the eigenbasis of \(\bSigma_{mm}\) is supported on \(\{1,\tilde\lambda_l\}\), and hence
\[
\tilde G_{\tilde \ba}
=
\bigl(1-(\tilde \ba^\top \tilde \bq)^2\bigr)\delta_1
+
(\tilde \ba^\top \tilde \bq)^2\delta_{\tilde\lambda_l},
\]
with
\[
(\tilde \ba^\top \tilde \bq)^2
=
\frac{((\U_m^\top \ba)^\top(\U_m^\top \bq))^2}
{\|\Pm \ba\|^2\,\|\Pm \bq\|^2}
=
\frac{(\ba^\top \U_m\U_m^\top \bq)^2}
{\|\Pm \ba\|^2\,\|\Pm \bq\|^2}
=
\frac{(\ba^\top \Pm \bq)^2}
{\|\Pm \ba\|^2\,\|\Pm \bq\|^2}.
\]
Thus, \(\tilde G_{\tilde \ba}\) is a two-point measure, with the original atom at \(\lambda\) replaced by an atom at \(\tilde\lambda_l\), and with weight determined by the projected overlap. Almost surely, it holds that:
\[
\tilde G_{\tilde \ba}(\tau) \Rightarrow \left(1 - \frac{\bar p_{\ba, \bq}^2}{\bar p_{\ba, \ba} \bar p_{\bq, \bq}}\right) \delta_{1} +  \frac{\bar p_{\ba, \bq}^2}{\bar p_{\ba, \ba} \bar p_{\bq, \bq}} \delta_{\bar \lambda_l}.
\]

\paragraph{Case \(m < n_l\).}
In this regime, \(\X_m \in \mathbb{R}^{n_l \times m}\) has full column rank almost surely, so \(\Pi_m = \mathbf I_m\). Hence \(\tilde \ba^\top \Pi_m \tilde \ba = 1\), giving
\[
\frac{\ba^\top \Pi \ba}{\|\ba\|^2} \xrightarrow{a.s.} \bar p_{\ba, \ba},
\]
as $\frac{\| \Pm \ba\|^2}{\|\ba\|^2} \xrightarrow{a.s.} \bar p_{\ba, \ba}$ from Lemma~\ref{lem:pperpv}.
\paragraph{Case \(m > n_l\).}
In this regime, \(\X_m\) has a nullspace of dimension \(m - n_l\), and
\[
\tilde \ba^\top \Pi_m \tilde \ba
=
1 - \tilde \ba^\top \Pi_{m,\mathrm{null}} \tilde \ba.
\]
Conditioned on \(\U_m\), the rows of \(\X_m\) are Gaussian with covariance \(\bSigma_{mm} = \U_m^\top \bSigma_l \U_m\). Therefore, by the nullspace mass formula from Lemma~\ref{lem:overlap_measure} applied to the \(m\)-dimensional model with aspect ratio
\(\gamma_{\mathrm{eff}} = m/n_l > 1\),
\begin{equation}
\tilde \ba^\top \Pi_{m,\mathrm{null}} \tilde \ba
\xrightarrow{a.s.}
\int
\frac{\gamma_{\mathrm{eff}} - 1}{\gamma_{\mathrm{eff}} - 1 + \tilde\tau}
\, d\tilde G_{\tilde \ba}(\tilde\tau).
\end{equation}
Finally, noting that $\frac{\| \Pm \ba\|^2}{\|\ba\|^2} \xrightarrow{a.s.} \bar p_{\ba, \ba}$ according to Lemma~\ref{lem:pperpv} and combining the two terms via the continuous mapping theorem gives the required result.

\end{proof}


\begin{lemma}[Deterministic limit of the bilinear effective projection $\ba^\top \Pi \bb$]
\label{lem:projection_effective_bilinear}
Consider the setup of  Section~\ref{sec:setup}, let \(\ba, \bb \in \mathbb{R}^p\), with \(\Pm \ba \neq \mathbf 0, \Pm \bb \neq \mathbf 0\), and define
\begin{align*}
    &\tilde \ba := \frac{\U_m^\top \ba}{\|\Pm \ba\|} \in \mathbb{R}^m, \qquad \tilde \bb := \frac{\U_m^\top \bb}{\|\Pm \bb\|} \in \mathbb{R}^m .
\end{align*}
Let \(\tilde G_{\tilde \ba, \tilde \bb}\) be the signed joint mass distribution of \(\tilde \ba, \tilde \bb\) onto the eigenvectors $\tilde \bq_i$ of \(\bSigma_{mm}\):
\[
\tilde G_{\tilde \ba, \tilde \bb}(\tilde\tau)
= \sum_i (\tilde \ba^\top \tilde \bq_i)(\tilde \bb^\top \tilde \bq_i)\,\mathbf{1}\{\tilde\lambda_i \le \tilde\tau\}.
\]
Then, it holds almost surely
\[
\tilde G_{\tilde \ba, \tilde \bb}(\tau) \Rightarrow \left(\frac{\bar p_{\ba, \bb}}{\bar{p}_{\ba,\ba}^{1/2} \bar{p}_{\bb,\bb}^{1/2}} - \frac{\bar p_{\ba, \bq} \bar p_{\bb, \bq}}{\bar p_{\ba, \ba}^{1/2} \bar p_{\bb, \bb}^{1/2}\bar p_{\bq, \bq}}\right) \delta_{1} +  \frac{\bar p_{\ba, \bq} \bar p_{\bb, \bq}}{\bar p_{\ba, \ba}^{1/2} \bar p_{\bb, \bb}^{1/2}\bar p_{\bq, \bq}} \delta_{\bar \lambda_l},
\]
and
\begin{equation} \label{eq:biliear_effective}
\frac{\ba^\top \Pi \bb}{\|\ba\| \| \bb \|} \xrightarrow{a.s.}
\begin{cases}
\bar p_{\ba,\bb}, & \gamma_\mathrm{eff} < 1, \\[6pt]
\bar p_{\ba, \bb} - \bar p_{\ba, \ba}^{1/2} \bar p_{\bb, \bb}^{1/2}\displaystyle\int
\frac{\gamma_{\mathrm{eff}} - 1}{\gamma_{\mathrm{eff}} - 1 + \tilde \tau}
\, d\tilde G_{\tilde \ba, \tilde \bb}(\tilde \tau), & \gamma_\mathrm{eff} > 1.
\end{cases}
\end{equation}
with $\bar p_{\ba, \ba}, \bar p_{\bb, \bb}, \bar p_{\bq, \bq}, \bar p_{\ba, \bq}, \bar p_{\bb, \bq}$ defined via Eq.~\eqref{eq:pperpv} and \eqref{eq:bilinear_proj}.
\end{lemma}

\begin{proof}
    The proof follows the steps of the proof of Lemma~\ref{lem:projection_effective}. The mass distribution $\tilde G_{\tilde \ba, \tilde \bb}$ can be written as
    \[
\tilde G_{\tilde \ba, \tilde \bb}
=
\bigl(\tilde \ba^\top \tilde \bb-(\tilde \ba^\top \tilde \bq)(\tilde \bb^\top \tilde \bq)\bigr)\delta_1
+
(\tilde \ba^\top \tilde \bq)(\tilde \bb^\top \tilde \bq)\delta_{\tilde\lambda_l},
\]
with 
\[
(\tilde \ba^\top \tilde \bq)(\tilde \bb^\top \tilde \bq)
=
\frac{((\U_m^\top \ba)^\top(\U_m^\top \bq))((\U_m^\top \bb)^\top(\U_m^\top \bq))}
{\|\Pm \ba\|\,\|\Pm \bb\|\,\|\Pm \bq\|^2}
=
\frac{(\ba^\top \Pm \bq)(\bb^\top \Pm \bq)}
{\|\Pm \ba\|\,\|\Pm \bb\|\,\|\Pm \bq\|^2},
\]
which converges almost surely to $\frac{\bar p_{\ba, \bq} \bar p_{\bb, \bq}}{\bar p_{\ba, \ba}^{1/2} \bar p_{\bb, \bb}^{1/2}\bar p_{\bq, \bq}}$ and 
\[
\tilde \ba^\top \tilde \bb = \frac{(\U_m^\top \ba)^\top(\U_m^\top \bb)}{\|\Pm \ba\| \|\Pm \bb\|} \xrightarrow{a.s.} \frac{\bar p_{\ba, \bb}}{\bar{p}_{\ba,\ba}^{1/2} \bar{p}_{\bb,\bb}^{1/2}}
\]
This readily gives the almost sure limit of  $\tilde G_{\tilde \ba, \tilde \bb}(\tau)$.
Next, observe
\begin{equation}
\frac{\ba^\top \Pi \bb}{\| \ba\| \| \bb \|}
=
\frac{1}{\|\ba\| \| \bb \|}(\U_m^\top \ba)^\top \Pi_m (\U_m^\top \bb)
=
\frac{\|\Pm \ba\|}{\|\ba\|}\frac{\|\Pm \bb\|}{\|\bb\|} \, \tilde \ba^\top \Pi_m \tilde \bb.
\end{equation}
Combining this with the almost sure limits from Lemma~\ref{lem:pperpv}, with the same steps as in the proof of  Lemma~\ref{lem:projection_effective} and with the mass distribution of $\tilde G_{\tilde \ba, \tilde \bb}$ proves the lemma.
\end{proof} 

\subsection{Phase transitions}
\subsubsection{Phase transition for \texorpdfstring{$\gamma_l>1$}{gamma\_l > 1}}
\label{app:sharp_transition}
Here, we derive parametric equations for the two phase transitions we observe in the phase diagrams for $\alpha$ in Figure~\ref{fig:fig2}. For small downstream tasks, $\gamma_l > 1$, there is a sharp phase transition due to competing local minima at different sides of the double descent (local minimum at large vs small $\alpha$), which we tackle first. We find a proxy of the parametric curve of this phase transition by comparing the values of the error at the endpoints in $\alpha$, and check when $E^\mathrm{gen}_\infty (\alpha \to 0) = E^\mathrm{gen}_\infty(\alpha=1)$. Note that the left endpoint is not exactly 0, as we always retain the first principal component of the data, but in the high-dimensional regime $p \to \infty$, $\alpha$ approaches 0. 

We denote signed normalized overlaps between deterministic vectors $\ba,\bb$ as
\[
\rho_{\ba,\bb}:=\frac{\ba^\top \bb}{\|\ba\|\,\|\bb\|}.
\]
Note that the signs of the corresponding \(\rho\)'s must be retained. Once we specify vectors $\ba,\bb$ to be one of the $\{\wstar, \bv,\bq\}$, we use the previous notation $r, r_v, r_q$ for the overlaps.

Define the limiting value of the population and sample spike eigenvector overlap $(\bv^\top \bu_1)^2 \to c$, which depends on whether the spike is below or above the BBP transition \cite{green_high-dimensional_2025}:
\[
c := \begin{cases} \displaystyle \frac{1-\frac{\gamma_u}{(\lambda_u-1)^2}}
     {1 + \frac{\gamma_u}{\lambda_u-1}} & \gamma_u < (\lambda_u-1)^2 \\
     0 & \gamma_u \geq (\lambda_u-1)^2
     \end{cases}.
\]
Additionally, the squared overlap between the leading empirical principal component and any other vector $\ba \perp \bv$ is 0.

\paragraph{Endpoint values for the eigenvalues and projections.}

At the left endpoint, meaning \(m=1\) and hence \(\alpha\to 0\), the retained direction is the empirical spike direction. Therefore, the limiting value of projections for any fixed vectors \(\ba,\bb\) are given by
\begin{align*}
&\bar p_{\ba, \ba}(0) = \rho_{\ba,\bv}^2 c,
&\bar p_{\perp,\ba, \ba}(0) = 1-\rho_{\ba,\bv}^2 c, \\
&\bar p_{\ba,\bb}(0) = \rho_{\ba,\bv}\rho_{\bb,\bv} c,
&\bar p_{\perp,\ba,\bb}(0)= \rho_{\ba,\bb} -\rho_{\ba,\bv}\rho_{\bb,\bv} c .
\end{align*}
At the right endpoint, \(\alpha=1\), the projection is the identity, \(\Pm=\mathbf I_p\). Hence, we have
\begin{align*}
&\bar p_{\ba, \ba}(1)=1, &\bar p_{\perp,\ba, \ba}(1)=0 \\
&\bar p_{\ba,\bb}(1)=\rho_{\ba,\bb}, &\bar p_{\perp,\ba,\bb}(1)=0 .
\end{align*}

Next, let us analyse the effective limiting projections $\bar \pi$. At \(\alpha=0\), the effective aspect ratio satisfies \(\gamma_{\mathrm{eff}}=0\). Hence the compressed task design is overdetermined, and the task projection is exact inside the retained subspace:
\begin{align*}
    &\bar\pi_{\ba, \ba}(0)=\bar p_{\ba, \ba}(0) = \rho_{\ba,\bv}^2c, \\
    &\bar\pi_{\ba,\bb}(0)=\bar p_{\ba,\bb}(0) = \rho_{\ba,\bv}\rho_{\bb,\bv}c .
\end{align*}

At $\alpha=1$, the effective downstream eigenvalue is $\lambda_l$
and $\gamma_{\mathrm{eff}}=\gamma_l$. Therefore,
\[
\bar\pi_{\mathbf a,\mathbf b}(1)
=
\begin{cases}
\rho_{\mathbf a,\mathbf b},
& \gamma_l<1,\\[6pt]
\displaystyle
\rho_{\mathbf a,\bq}\rho_{\mathbf b,\bq}
\frac{\lambda_l}{\gamma_l-1+\lambda_l}
+
\frac{
\rho_{\mathbf a,\mathbf b}
-\rho_{\mathbf a,\bq}\rho_{\mathbf b,\bq}
}{\gamma_l},
& \gamma_l>1.
\end{cases}
\]
At the compressed endpoint, the effective downstream
spike eigenvalue is
\[
\bar\lambda_{l,0}
=
1+(\lambda_l-1)
r^2c.
\]

\paragraph{Endpoint values for the generalisation error.}
At $\alpha=1$, we have $P_m=I_p$, hence
\[
\bar p_{\perp,\wstar,\wstar}=0,
\qquad
\bar p_{\perp,\wstar,\bq}=0,
\qquad
\bar\sigma_{\mathrm{eff}}^2=0.
\]

Therefore, for $\gamma_l>1$,
\begin{equation*}
E_\infty^{\mathrm{gen}}(1)
=
\bar w^\ast
\bigl(1-\bar\pi_{\wstar,\wstar}(1)\bigr)
+
\frac{\sigma^2}{\gamma_l-1}
+\sigma^2
+
\bar w^\ast(\lambda_l-1)
\left[
r_q
-
\bar\pi_{\wstar,\bq}(1)
\right]^2.
\end{equation*}

Substituting the endpoint projections gives
\begin{equation*}
\begin{aligned}
B_1(\gamma_l) &:=
1-
r_q^2
\frac{\lambda_l}{\gamma_l-1+\lambda_l}
-
\frac{1-r_q^2}{\gamma_l}
+(\lambda_l-1)
\bigg[
r_q
-
r_q\frac{\lambda_l}{\gamma_l-1+\lambda_l}
\bigg]^2 \\
&= (1-r_q^2)\frac{\gamma_l-1}{\gamma_l}
+
r_q^2
\frac{\lambda_l\gamma_l(\gamma_l-1)}
{(\gamma_l-1+\lambda_l)^2}.
\end{aligned}
\end{equation*}
Thus,
\[
E_\infty^{\mathrm{gen}}(1)
=
\bar w^\ast B_1(\gamma_l)
+\frac{\sigma^2}{\gamma_l-1}
+\sigma^2.
\]

At $\alpha=0$, the effective aspect ratio vanishes, so
the variance term in the estimation error vanishes. Moreover,
\begin{align*}
\bar\pi_{\wstar,\wstar}(0)
&=
r_v^2c,
\\
\bar\pi_{\bq,\bq}(0)
&=
r^2c,
\\
\bar p_{\perp,\wstar,\bq}(0)
&=
r_q 
-
r_vrc.
\end{align*}
Substitution into the estimation-error formula gives
\begin{align*}
\frac{E_\infty^{\mathrm{est}}(0)}{\bar w^\ast}
=
1-r_v^2c
+
\frac{
(\lambda_l-1)^2
r^2c
}{
\bar\lambda_{l,0}^2
}
\left(
r_q
-
r_vrc
\right)^2.
\end{align*}

The generalisation error residual simplifies to
\begin{align*}
&(\lambda_l-1)\left(r_q
-
\bar\pi_{\wstar,\bq}(0)
-
\frac{\lambda_l-1}{\bar\lambda_{l,0}}
\bar p_{\perp,\wstar,\bq}(0)
\bar\pi_{\bq,\bq}(0)\right)^2
\\
&\qquad= (\lambda_l - 1)
\frac{(r_q 
-
r_vrc)^2}{\bar\lambda_{l,0}^2}
\
\end{align*}
Consequently, define
\begin{align*}
B_0(\gamma_u)
&:=
1-r_v^2c
+
\frac{
(\lambda_l-1)^2
r^2c
}{
\bar\lambda_{l,0}^2
}
\left(
r_q
-
r_vrc
\right)^2
+\frac{(\lambda_l-1)}{\bar\lambda_{l,0}^2} 
\left(
r_q 
-
r_vrc \right)^2 \\
&= 
1-r_v^2c + \frac{(\lambda_l-1)}{\bar\lambda_{l,0}}\left(
r_q 
-
r_vrc \right)^2
\end{align*}
Then
\[
E_\infty^{\mathrm{gen}}(0)
=
\bar w^\ast B_0(\gamma_u)+\sigma^2.
\]

Let $S:=\bar w^\ast/\sigma^2$. The phase boundary in the
$(\gamma_u,\gamma_l)$-plane is therefore approximated by

\[
B_0(\gamma_u)
=
B_1(\gamma_l)
+
\frac{1}{S(\gamma_l-1)},
\qquad \gamma_l>1.
\]
Here the dependence of $B_0$ on $\gamma_u$ is through
$c$. This endpoint comparison approximates
the transition between the actual minima when they lie close
to the endpoints.

\paragraph{Specialization to
$\bSigma_l=\bSigma_u$.}
Set $\mathbf q=\bv$ and $\lambda_u=\lambda_l=\lambda$.
Then
\begin{equation*}
\begin{aligned}
B_0
&=
1-r_v^2c
+
\frac{(\lambda-1)
r_v^2(1-c)^2
}{
1+(\lambda-1)c
}\\
&= 1 + r_v^2(\lambda-1) - \frac{r_v^2 \lambda^2 c}{1 + (\lambda-1)c},
\end{aligned}
\end{equation*}
while
\[
B_1
=
(1-r_v^2)
\frac{\gamma_l-1}{\gamma_l}
+
r_v^2
\frac{\lambda\gamma_l(\gamma_l-1)}
{(\gamma_l-1+\lambda)^2}.
\]

\paragraph{Specialization to
$\bSigma_l=\bSigma_u$
and $\wstar=\sqrt{\bar w^\ast}\bv$.}
Now everything simplifies even further:
\[
B_0
=
\frac{\lambda(1-c)}
{1+(\lambda-1)c},
\qquad
B_1
=
\frac{\lambda\gamma_l(\gamma_l-1)}
{(\gamma_l-1+\lambda)^2}.
\]

Due to the two regimes of the empirical spike overlap,
the proxy transition splits into two components:
\[
\frac{\lambda\gamma_l(\gamma_l-1)}
{(\gamma_l-1+\lambda)^2}
+
\frac{1}{S(\gamma_l-1)}
=
\begin{cases}
\lambda,
& \gamma_u\ge(\lambda-1)^2,\\[6pt]
\displaystyle\frac{\lambda\gamma_u}{(\lambda-1)^2},
& \gamma_u<(\lambda-1)^2.
\end{cases}
\]

\subsubsection{Phase transition for \texorpdfstring{$\gamma_l<1$}{gamma\_l < 1} }
\label{app:smooth_transition}

Next, we study the regime of larger downstream tasks, $\gamma_l<1$, and find a transition between the phases $\alpha=1$ and $\alpha < 1$ by finding a parametric curve describing
the loss of stability of $\alpha=1$, i.e., when
\[
\left.
\frac{d}{d\alpha}E_\infty^{\mathrm{gen}}(\alpha)
\right|_{\alpha=1^-}=0.
\]

\paragraph{Derivatives of the pretraining projection at $\alpha=1$.}
Let $\lambda_\pm=(1\pm\sqrt{\gamma_u})^2$, and define
\[
H_{\mathbf a}(\theta)
:=
(1-\rho_{\ba, \bv}^2)
+
\rho_{\ba, \bv}^2
\frac{\lambda_u\gamma_u}
{\lambda_u^2-\lambda_u(1-\gamma_u+\theta)+\theta}.
\]

Since
\[
g_{\mathbf a}^{\mathrm{sample}}(\theta)
=
H_{\mathbf a}(\theta)f_{\gamma_u}(\theta),
\]
the derivative of $\bar p_{\perp,\mathbf a,\mathbf a}$
at $\alpha=1$ is
\[
\left.
\frac{d}{d\alpha}
\bar p_{\perp,\mathbf a,\mathbf a}(\alpha)
\right|_{\alpha=1^-}
=
\begin{cases}
-H_{\mathbf a}(\lambda_-),
& \gamma_u<1,\\[7pt]
\displaystyle
-\frac{G_{\mathbf a}^{\mathrm{sample}}(\{0\})}
{1-\gamma_u^{-1}},
& \gamma_u>1.
\end{cases}
\]
Equivalently,
\[
\left.
\frac{d}{d\alpha}
\bar p_{\mathbf a,\mathbf a}(\alpha)
\right|_{\alpha=1^-}
=
(1-\rho_{\ba, \bv}^2)
+
\rho_{\ba, \bv}^2D_\bv(\gamma_u;\lambda_u),
\]
where
\[
D_\bv(\gamma_u;\lambda_u)
=
\begin{cases}
\displaystyle
\frac{\lambda_u\gamma_u}
{(\lambda_u-1+\sqrt{\gamma_u})^2},
& \gamma_u<1,\\[10pt]
\displaystyle\frac{\gamma_u}{\gamma_u-1+\lambda_u},
& \gamma_u>1.
\end{cases}
\]
The value at $\gamma_u=1$ follows from the limiting spectral
density at zero, or equivalently by continuity of the displayed
expressions. 

For signed bilinear projections, replace
$\rho_{\ba, \bv}^2$
by $\rho_{\ba, \bv}\rho_{\bb, \bv}$, and replace the
constant bulk contribution
$1-\rho_{\ba, \bv}^2$
by
$\rho_{\mathbf a,\mathbf b}
-\rho_{\ba, \bv}\rho_{\bb, \bv}$.

\paragraph{Derivative of the effective projection at $\alpha=1$.}
For $\gamma_l<1$, the compressed task design remains
overdetermined in a neighbourhood of $\alpha=1$. Thus,
\[
\bar\pi_{\mathbf a,\mathbf b}(\alpha)
=
\bar p_{\mathbf a,\mathbf b}(\alpha),
\]
and hence
\[
\left.
\frac{d}{d\alpha}\bar\pi_{\mathbf a,\mathbf b}(\alpha)
\right|_{\alpha=1^-}
=
\left.
\frac{d}{d\alpha}\bar p_{\mathbf a,\mathbf b}(\alpha)
\right|_{\alpha=1^-}.
\]

\paragraph{Stability of $\alpha=1$ for $\gamma_l<1$.}
Assume $\gamma_l<1$ and let $\alpha\to1^-$. Define
\begin{equation*}
\begin{aligned}
D_{\wstar}
&:=
\left.
\frac{d}{d\alpha}
\bar p_{\wstar,\wstar}(\alpha)
\right|_{\alpha=1^-}\\
&=
(1-r_v^2)
+
r_v^2D_\bv(\gamma_u;\lambda_u).
\end{aligned}
\end{equation*}

At $\alpha=1$, $\bar p_{\perp,\mathbf a,\mathbf b}=0$
and $\bar\sigma_{\mathrm{eff}}^2=0$, thus simplifying
\[
\left.
\frac{d}{d\alpha}E_\infty^{\mathrm{gen}}(\alpha)
\right|_{\alpha=1^-}
=
\frac{\gamma_l\sigma^2}{(1-\gamma_l)^2}
-
\frac{\bar w^\ast}{1-\gamma_l}D_{\wstar}.
\]

Consequently, the derivative is negative when
\[
\gamma_l<
\frac{SD_{\wstar}}{1+SD_{\wstar}},
\qquad
S:=\frac{\bar w^\ast}{\sigma^2},
\]
so sufficiently small decreases in $\alpha$ increase the error.
The derivative is positive when the inequality is reversed,
in which case compression locally improves generalisation.
Equality gives the marginal stability boundary

\[
\gamma_l^{\mathrm{crit}}
=
\frac{SD_{\wstar}}{1+SD_{\wstar}}.
\]

Writing in terms of $D_\bv$: 
\[
\gamma_l^{\mathrm{crit}}
=
\frac{
S[(1-r_v^2)+r_v^2 D_\bv(\gamma_u;\lambda_u)]
}{
1+S[(1-r_v^2)+r_v^2 D_\bv(\gamma_u;\lambda_u)]
}.
\]

\subsection{Infinite data limits}
\label{app:theory_inf_limits}
We consider what happens to the errors and optimal $\alpha$ when $\gamma_u \to 0$ or $\gamma_l \to 0$. We specify this to the case of $\bSigma_u = \bSigma_l$, i.e. $\lambda_u = \lambda_l = \lambda$, and $\bq = \bv$. The more general case is easy to obtain as well, but the resulting expressions are not as readable.

\paragraph{Prior data limit $\gamma_u \to 0$.}

In this limit, the empirical spike aligns with the population spike, and the retained subspace contains the spike direction together with an $\alpha$-fraction of the isotropic bulk.

For any fixed vectors $\ba,\bb$, we have
\begin{align*}
\bar p_{\ba, \ba}
&\to
\rho_{\ba,\bv}^2 + (1-\rho_{\ba,\bv}^2)\alpha, \\
\bar p_{\perp,\ba,\ba}
&\to
(1-\rho_{\ba,\bv}^2)(1-\alpha), \\
\bar p_{\ba,\bb}
&\to
\rho_{\ba,\bv}\rho_{\bb,\bv}
+
\alpha\left(
\rho_{\ba,\bb}
-
\rho_{\ba,\bv}\rho_{\bb,\bv}
\right), \\
\bar p_{\perp,\ba,\bb}
&\to
(1-\alpha)\left(
\rho_{\ba,\bb}
-
\rho_{\ba,\bv}\rho_{\bb,\bv}
\right).
\end{align*}
In particular,
\[
\bar p_{\bv, \bv}\to1,
\qquad
\bar p_{\perp, \bv, \bv}\to0,
\qquad
\bar\lambda \to \lambda.
\]
For the effective projections $\bar\pi$, the behaviour depends on $\gamma_{\mathrm{eff}}$.

If $\gamma_{\mathrm{eff}}<1$, then
\[
\bar\pi_{\ba, \ba} \to \bar p_{\ba, \ba},
\qquad
\bar\pi_{\ba,\bb} \to \bar p_{\ba,\bb}.
\]

If $\gamma_{\mathrm{eff}}>1$, define
\[
A_\lambda := \frac{\lambda}{\gamma_{\mathrm{eff}}-1+\lambda},
\qquad
A_1 := \frac{1}{\gamma_{\mathrm{eff}}},
\]
then
\begin{align*}
\bar\pi_{\ba, \ba}
&\to
\rho_{\ba,\bv}^2 A_\lambda
+
(1-\rho_{\ba,\bv}^2)\alpha A_1, \\
\bar\pi_{\ba,\bb}
&\to
\rho_{\ba,\bv}\rho_{\bb,\bv} A_\lambda
+
\alpha\left(
\rho_{\ba,\bb}
-
\rho_{\ba,\bv}\rho_{\bb,\bv}
\right) A_1.
\end{align*}

\paragraph{Errors in the limit $\gamma_u\to0$.}

In the infinite-prior-data limit, the empirical spike is recovered exactly. Hence
\[
\bar p_{\bv, \bv}\to1,\qquad 
\bar p_{\perp,\wstar,\bv}\to0,\qquad 
\bar\lambda\to\lambda,
\]
and
\[
\bar p_{\wstar, \wstar}\to r_v^2+(1-r_v^2)\alpha,
\qquad
\bar p_{\perp, \wstar, \wstar}\to (1-r_v^2)(1-\alpha).
\]
Therefore
\[
\bar\sigma_{\mathrm{eff}}^2
\to
(1-r_v^2)(1-\alpha).
\]

Since the leak term vanishes,
\[
\boxed{
E^\mathrm{est}_\infty
\to
\bar w^\ast \left(1-\bar\pi_{\wstar, \wstar}\right)
+
\frac{\min\{\gamma_{\mathrm{eff}},1\}}
{|\gamma_{\mathrm{eff}}-1|}
\left[
\bar w^\ast(1-r_v^2)(1-\alpha)+\sigma^2
\right],
}
\]
where
\[
\bar\pi_{\wstar, \wstar}
=
\begin{cases}
r_v^2+(1-r_v^2)\alpha,
& \gamma_{\mathrm{eff}}<1,\\[6pt]
\displaystyle
r_v^2\frac{\lambda}{\gamma_{\mathrm{eff}}-1+\lambda}
+
(1-r_v^2)\frac{1}{\gamma_l},
& \gamma_{\mathrm{eff}}>1.
\end{cases}
\]

Moreover, we obtain
\[
\boxed{
E^\mathrm{gen}_\infty \to E^\mathrm{est}_\infty
+ \sigma^2 +
\bar w^\ast (\lambda -1) r_v^2 \begin{cases}0, & \gamma_\mathrm{eff} < 1 \\
\left(\frac{\gamma_\mathrm{eff} - 1}{\gamma_\mathrm{eff} - 1 + \lambda}
\right)^2, &\gamma_\mathrm{eff} > 1
\end{cases}.
}
\]
Finally, the training error satisfies
\[
\boxed{
E^\mathrm{train}_\infty
\to
\left[
\sigma^2+\bar w^\ast(1-r_v^2)(1-\alpha)
\right]
\max\{0,1-\gamma_{\mathrm{eff}}\}.
}
\]

\paragraph{Downstream-data limit $\gamma_l \to 0$.}

For any fixed $\alpha>0$, we have $\gamma_{\mathrm{eff}}=\alpha\gamma_l\to0$, hence
\[
\bar\pi_{\ba, \ba} \to \bar p_{\ba, \ba},
\qquad
\bar\pi_{\ba,\bb} \to \bar p_{\ba,\bb}.
\]

Moreover, the variance factor satisfies
\[
\frac{\min\{\gamma_{\mathrm{eff}},1\}}{|\gamma_{\mathrm{eff}}-1|}
=
\frac{\alpha\gamma_l}{1-\alpha\gamma_l}
\to 0,
\]
so all downstream finite-sample variance contributions vanish.

The pretraining projections $\bar p$ remain determined by $\gamma_u$.

At $\alpha=1$,
\[
\bar p_{\ba, \ba}=1,
\qquad
\bar p_{\perp,\ba,\ba}=0,
\qquad
\bar\pi_{\ba, \ba}\to1,
\]
and
\[
\bar p_{\ba,\bb}=\rho_{\ba,\bb},
\qquad
\bar p_{\perp,\ba,\bb}=0,
\qquad
\bar\pi_{\ba,\bb}\to\rho_{\ba,\bb}.
\]

At the spike-retaining endpoint $\alpha\to0$,
\[
\bar\pi_{\ba, \ba}(0)\to\bar p_{\ba, \ba}(0),
\qquad
\bar\pi_{\ba,\bb}(0)\to\bar p_{\ba,\bb}(0),
\]
with $\bar p$ determined by the prior overlap $c(\gamma_u)$.

\paragraph{Errors in the limit $\gamma_l\to0$.}

For fixed $\alpha>0$, $\gamma_{\mathrm{eff}}=\alpha\gamma_l\to0$. Hence
\[
\bar\pi_{\ba, \ba}\to \bar p_{\ba, \ba},
\qquad
\bar\pi_{\ba,\bb}\to \bar p_{\ba,\bb},
\]
and
\[
\frac{\min\{\gamma_{\mathrm{eff}},1\}}{|\gamma_{\mathrm{eff}}-1|}
\to0.
\]

Therefore the estimation error satisfies
\[
\boxed{
E^\mathrm{est}_\infty
\to
\bar w^\ast
\left[
1-\bar p_{\wstar, \wstar}
+
\frac{(\lambda-1)^2}{\bar\lambda^2}
\bar p_{\perp,\wstar,\bv}^2\bar p_{\bv, \bv}
\right].
}
\]

The generalisation error satisfies
\begin{equation}\label{eq:inf_downstream_gen}
\boxed{
E^\mathrm{gen}_\infty
\to
E^\mathrm{est}_\infty
+
\bar w^\ast
(\lambda-1)
\left(
r_v
-
\bar p_{\wstar,\bv}
-
\frac{\lambda-1}{\bar\lambda}
\bar p_{\perp,\wstar,\bv}
\bar p_{\bv,\bv}
\right)^2
+\sigma^2.
}
\end{equation}

The training error satisfies
\begin{equation}
    \boxed{
E^\mathrm{train}_\infty
\to
\sigma^2
+
\bar w^\ast
\left[
\bar p_{\perp, \wstar, \wstar}
+
\frac{\lambda-1}{\bar\lambda}
\bar p_{\perp,\wstar,\bv}^2
\right].
}
\end{equation}

\begin{corollary}
\label{cor:infinite_unlabelled_limit}
Assume $\gamma_u=p/n_u\to0$, so that the empirical spike direction converges to the population spike $\bv$. Let \(
\gamma_{\mathrm{eff}}=\alpha\gamma_l, \;
S=\bar w^\ast/\sigma^2
\). Then the estimation error becomes
\[
E^\mathrm{est}_\infty
\to
\bar w^\ast(1-\bar\pi_{\wstar, \wstar})
+
\frac{\min\{\gamma_{\mathrm{eff}},1\}}
{|\gamma_{\mathrm{eff}}-1|}
\left[
\bar w^\ast(1-r_v^2)(1-\alpha)+\sigma^2
\right],
\]
where
\[
\bar\pi_{\wstar, \wstar}
=
\begin{cases}
r_v^2+(1-r_v^2)\alpha,
& \gamma_{\mathrm{eff}}<1,\\[6pt]
\displaystyle
r_v^2\frac{\lambda}{\gamma_{\mathrm{eff}}-1+\lambda}
+
(1-r_v^2)\frac{\alpha}{\gamma_{\mathrm{eff}}},
& \gamma_{\mathrm{eff}}>1.
\end{cases}
\]

For matched test covariance $\bSigma_{\mathrm{new}}=\bSigma$,
\[
E^\mathrm{gen}_\infty
\to
E^\mathrm{est}_\infty
+
\bar w^\ast(\lambda-1)r_v^2
\begin{cases}
0,
& \gamma_{\mathrm{eff}}<1,\\[6pt]
\displaystyle
\left(
\frac{\gamma_{\mathrm{eff}}-1}
{\gamma_{\mathrm{eff}}-1+\lambda}
\right)^2,
& \gamma_{\mathrm{eff}}>1,
\end{cases}
+\sigma^2 .
\]
Finally,
\[
E^\mathrm{train}_\infty
\to
\left[
\sigma^2+\bar w^\ast(1-r_v^2)(1-\alpha)
\right]
\max\{0,1-\gamma_{\mathrm{eff}}\}.
\]
\end{corollary}
Consequently, the phase transitions in this limit can either be derived by taking derivatives and limiting values $\alpha$, or by substituting the limits directly in Corollaries \ref{cor:sharp_transition} and \ref{cor:smooth_transition}. In this regime, the phase boundaries become independent of $\gamma_u$ and depend only on $r_v^2$, $\mathrm{SNR}$, $\lambda$, and $\gamma_l$. 

\begin{corollary}[Phase transitions for $\gamma_u\to 0$]
\label{cor:phases_infinite_unlabelled_limit}
In the limit $\gamma_u\to0$, the spike direction is recovered exactly. Let $\bSigma_l = \bSigma_u, \bq = \bv$. Hence,
$c(\gamma_u)\to1$ and $D_\bv(\gamma_u)\to0$. Thus,  the transition curve when
$\gamma_l>1$ simplifies to
\[
(1-r_v^2)\frac{1}{\gamma_l} =
r_v^2\frac{\lambda \gamma_l(\gamma_l-1)}
{(\gamma_l-1+\lambda)^2} + \frac{1}{S(\gamma_l-1)}.
\]
Furthermore, when $\gamma_l<1$, the stability boundary reduces to
\[
\gamma_l^{\mathrm{crit}}
=
\frac{S(1-r_v^2)}
{1+S(1-r_v^2)}.
\]
\end{corollary}

\subsection{Regularized regression}
\label{app:ridge_extension}

We consider the ridge estimator
\begin{equation}
\what_\rho
=
\arg\min_{\w\in\mathbb R^p}
\frac1{n_l}\|\y_l-\X_l\Pm\w\|^2+\rho\|\w\|^2,
\qquad \rho>0.
\label{eq:ridge_objective_p_dim}
\end{equation}
The loss is normalized by $n_l$ so that $\rho$ remains fixed
in the high-dimensional limit.

Let $\Xm:=\X_l\U_m$ and $\Ss_m:=n_l^{-1}\Xm^\top\Xm$.
As before, the minimizer lies in the $\Pm$ subspace and is
\begin{equation}
\what_\rho
=
\U_m(\Ss_m+\rho\I_m)^{-1}
\frac1{n_l}\Xm^\top\y_l.
\label{eq:ridge_estimator_retained_coordinates}
\end{equation}
Define the regularized effective projection
\begin{equation}
\Pi_\rho
:=
\U_m(\Ss_m+\rho\I_m)^{-1}\Ss_m\U_m^\top.
\label{eq:regularized_effective_projection}
\end{equation}
As $\rho \to 0$, this converges to the row-space projector
$\Pi=\U_m\Xm^\dagger\Xm\U_m^\top$.
For $\rho>0$, $\Pi_\rho$ is generally not idempotent,
so the error analysis requires both $\Pi_\rho$ and $\Pi_\rho^2$. We next state all lemmas and theorems that differ from the unregularized setup.

\paragraph{Spectrum of the regularized effective projection.}

\begin{lemma}[Spectrum of $\Pi_\rho$]
\label{lem:spectrum_regularized_projection}
Let
\[
\Ss_m=\mathbf Q\B\mathbf Q^\top,
\qquad
\B=\operatorname{diag}(b_1,\ldots,b_m),
\]
be the eigendecomposition of $\Ss_m$. Then
\begin{equation}
\Pi_\rho
=
\U_m\mathbf Q
\operatorname{diag}\left(
\frac{b_1}{b_1+\rho},
\ldots,
\frac{b_m}{b_m+\rho}
\right)
\mathbf Q^\top\U_m^\top .
\label{eq:Pi_rho_spectrum}
\end{equation}
Moreover, the empirical spectral distribution of the retained block of
$\Pi_\rho$ converges weakly almost surely to the pushforward of
$F_{\gamma_{\rm eff}}$ under the map
\[
h_\rho(x)=\frac{x}{x+\rho}.
\]
That is,
\[
\frac1m\sum_{i=1}^m
\delta_{b_i/(b_i+\rho)}
\Rightarrow
(h_\rho)_\#F_{\gamma_{\rm eff}}
\qquad\text{almost surely.}
\]
\end{lemma}

\begin{proof}
Since $\Ss_m$ is symmetric positive semidefinite,
\[
\Ss_m+\rho\I_m
=
\mathbf Q(\B+\rho\I_m)\mathbf Q^\top,
\]
and hence
\[
(\Ss_m+\rho\I_m)^{-1}\Ss_m
=
\mathbf Q
\operatorname{diag}\left(
\frac{b_1}{b_1+\rho},\ldots,\frac{b_m}{b_m+\rho}
\right)
\mathbf Q^\top .
\]
Substitution into \eqref{eq:regularized_effective_projection} gives
\eqref{eq:Pi_rho_spectrum}. By Lemma~\ref{lem:esd_projected_x}, the empirical
spectral distribution of $\Ss_m$ converges weakly almost surely to
$F_{\gamma_{\rm eff}}$. Since $h_\rho(x)=x/(x+\rho)$ is bounded and
continuous on $[0,\infty)$ for every $\rho>0$, the continuous mapping theorem
gives the claimed pushforward limit.
\end{proof}

\paragraph{Regularized bilinear effective projections.}

We next state the ridge analogue of Lemmas~\ref{lem:projection_effective}
and~\ref{lem:projection_effective_bilinear}. Since $\Pi_\rho$ is not
idempotent, the estimation error depends on both $\Pi_\rho$ and
$\Pi_\rho^2$.

As before, for deterministic vectors $\ba,\bb$, define the retained-subspace overlaps
\[
\frac{\ba^\top\Pm\bb}{\|\ba\|\|\bb\|} \to \bar p_{\ba,\bb}, \qquad \bar p_{\ba,\bb}
=
\rho_{\ba,\bb}
-
\bar p_{\perp,\ba,\bb}..
\]
These quantities depend on the pretraining covariance
$\bSigma_u=\I_p+(\lambda_u-1)\bv\bv^\top$ through the learned subspace
$\Pm$.

Conditioned on $\U_m$, however, the downstream retained-coordinate covariance is
\[
\bSigma_{mm}
=
\U_m^\top\bSigma_l\U_m
=
\I_m+(\lambda_l-1)(\U_m^\top\bq)(\U_m^\top\bq)^\top.
\]
Its effective spike eigenvalue converges to
\begin{equation}
\bar\lambda_l
:=
1+(\lambda_l-1)\bar p_{\bq,\bq}.
\label{eq:ridge_effective_downstream_spike}
\end{equation}

Assume $\bar p_{\ba,\ba}>0$ and $\bar p_{\bb,\bb}>0$, and define
\[
\tilde \ba
:=
\frac{\U_m^\top\ba}{\|\Pm\ba\|},
\qquad
\tilde \bb
:=
\frac{\U_m^\top\bb}{\|\Pm\bb\|}.
\]
The limiting signed population spectral measure of
$(\tilde\ba,\tilde\bb)$ with respect to $\bSigma_{mm}$ is
\begin{equation}
\begin{aligned}
\widetilde G_{\ba,\bb}
&=
\left(
\omega_{\ba,\bb}
-
\omega_{\ba,\bq}\omega_{\bb,\bq}
\right)\delta_1
+
\omega_{\ba,\bq}\omega_{\bb,\bq}\delta_{\bar\lambda_l},
\\
\omega_{\ba,\bb}
&:=
\frac{\bar p_{\ba,\bb}}
{\bar p_{\ba,\ba}^{1/2}\bar p_{\bb,\bb}^{1/2}},
\qquad\omega_{\ba,\bq}
:=
\frac{\bar p_{\ba,\bq}}
{\bar p_{\ba,\ba}^{1/2}\bar p_{\bq,\bq}^{1/2}},
\qquad
\omega_{\bb,\bq}
:=
\frac{\bar p_{\bb,\bq}}
{\bar p_{\bb,\bb}^{1/2}\bar p_{\bq,\bq}^{1/2}} .
\end{aligned}
\label{eq:retained_population_signed_measure}
\end{equation}
Let $G^{l,{\rm sample}}_{\tilde\ba,\tilde\bb}$ denote the corresponding limiting
signed sample spectral measure after observing $n_l$ downstream samples in
$m$ retained dimensions, obtained from $\widetilde G_{\ba,\bb}$ through the
same deterministic equivalent as in Lemma~\ref{lem:overlap_measure}, but with
aspect ratio $\gamma_{\rm eff}$.

\begin{lemma}[Regularized bilinear effective projection]
\label{lem:regularized_bilinear_effective_projection}
For fixed $\rho>0$ and deterministic vectors $\ba,\bb$ with
$\bar p_{\ba,\ba},\bar p_{\bb,\bb}>0$, we have, for $k=1,2$,
\begin{equation}
\frac{\ba^\top\Pi_\rho^k\bb}{\|\ba\|\|\bb\|}
\xrightarrow{\rm a.s.}
\pi^{(k)}_{\rho,\ba,\bb}
:=
\sqrt{\bar p_{\ba,\ba}\bar p_{\bb,\bb}}
\int \left(\frac{x}{x+\rho}\right)^k
\,dG^{l,{\rm sample}}_{\tilde\ba,\tilde\bb}(x).
\label{eq:regularized_bilinear_projection_limit}
\end{equation}
\end{lemma}

\begin{proof}
Write $\Ss_m=n_l^{-1}(\X_l\U_m)^\top(\X_l\U_m)$.
Since
\[
\Pi_\rho
=
\U_m\Ss_m(\Ss_m+\rho\I_m)^{-1}\U_m^\top,
\]
the result follows from convergence of the retained-subspace
overlaps and the downstream signed spectral measures, applied
to the bounded continuous functions $(x/(x+\rho))^k$.
\end{proof}

\paragraph{Regularized variance trace.}

The variance terms in the ridge estimator involve
\[
\frac1{n_l}
\Tr\left[
\Ss_m(\Ss_m+\rho\I_m)^{-2}
\right],
\]
rather than the pseudoinverse trace appearing in the ridgeless case.

\begin{lemma}[Regularized variance trace]
\label{lem:regularized_variance_trace}
Under the assumptions of Section~\ref{sec:setup}, for fixed $\rho>0$,
\begin{equation}
\frac1{n_l}
\Tr\left[
\Ss_m(\Ss_m+\rho\I_m)^{-2}
\right]
\xrightarrow{\rm a.s.}
V_\rho(\gamma_{\rm eff}),
\label{eq:regularized_variance_trace_limit}
\end{equation}
where
\begin{equation}
V_\rho(\gamma_{\rm eff})
:=
\gamma_{\rm eff}
\int
\frac{x}{(x+\rho)^2}\,
dF_{\gamma_{\rm eff}}(x).
\label{eq:V_rho_definition}
\end{equation}
\end{lemma}

\begin{proof}
We write
\[
\frac1{n_l}
\Tr\left[
\Ss_m(\Ss_m+\rho\I_m)^{-2}
\right]
=
\frac{m}{n_l}
\frac1m
\Tr\left[
\Ss_m(\Ss_m+\rho\I_m)^{-2}
\right].
\]
By Lemma~\ref{lem:esd_projected_x}, the empirical spectral distribution of
$\Ss_m$ converges weakly almost surely to $F_{\gamma_{\rm eff}}$.
Since $x\mapsto x/(x+\rho)^2$ is bounded and continuous on
$[0,\infty)$, weak convergence gives
\[
\frac1m
\Tr\left[
\Ss_m(\Ss_m+\rho\I_m)^{-2}
\right]
\xrightarrow{\rm a.s.}
\int
\frac{x}{(x+\rho)^2}
\,dF_{\gamma_{\rm eff}}(x).
\]
Multiplying by $m/n_l\to\gamma_{\rm eff}$ gives the claim.
\end{proof}

\subsubsection{Decomposition of the ridge estimator}

We now repeat the estimator decomposition from Appendix~\ref{app:proofideas},
with the hard projector $\Pi$ replaced by $\Pi_\rho$. Recall from the
conditional Gaussian decomposition that
\[
\Xperp
=
\Xm\bSigma_{mm}^{-1}\bSigma_{m\perp}
+
\E,
\]
where $\E$ is independent of $\Xm$. Starting from \eqref{eq:ridge_estimator_retained_coordinates}, substitute
\[
\y_l
=
\X_l\wstar+\bxi
=
\Xm\U_m^\top\wstar
+
\Xperp\U_\perp^\top\wstar
+
\bxi .
\]
Using
\[
\Xperp
=
\Xm\bSigma_{mm}^{-1}\bSigma_{m\perp}
+
\E
\]
and
\[
\bSigma_{mm}^{-1}\bSigma_{m\perp}\U_\perp^\top\wstar
=
A\U_m^\top\bq,
\]
where 
\[
A
:=
\frac{(\lambda_l-1)\wstar^\top\Pperp\bq}{\tilde\lambda_l}, \qquad
\tilde\lambda_l
:=
1+(\lambda_l-1)\|\Pm\bq\|^2 .
\]
We obtain
\[
\begin{aligned}
\what_\rho
&=
\U_m(\Ss_m+\rho\I_m)^{-1}\Ss_m\U_m^\top\wstar
+
A\U_m(\Ss_m+\rho\I_m)^{-1}\Ss_m\U_m^\top\bq
\\
&\quad+
\U_m(\Ss_m+\rho\I_m)^{-1}
\frac1{n_l}\Xm^\top\E\U_\perp^\top\wstar
+
\U_m(\Ss_m+\rho\I_m)^{-1}
\frac1{n_l}\Xm^\top\bxi .
\end{aligned}
\]
The first two terms are precisely $\Pi_\rho\wstar$ and
$A\Pi_\rho\bq$, thus the final decomposition is:
\begin{equation}
\begin{aligned}
\what_\rho
&=
\Pi_\rho\wstar
+
A\Pi_\rho\bq
\\
&\quad+
\U_m(\Ss_m+\rho\I_m)^{-1}
\frac1{n_l}\Xm^\top\E\U_\perp^\top\wstar
+
\U_m(\Ss_m+\rho\I_m)^{-1}
\frac1{n_l}\Xm^\top\bxi .
\end{aligned}
\label{eq:ridge_estimator_decomposition}
\end{equation}

\subsubsection{Asymptotic estimation error for ridge regression}

The ridge estimator decomposition from \eqref{eq:ridge_estimator_decomposition} can be written as
\begin{equation}
\what_\rho
=
\Pi_\rho(\wstar+A\bq)
+
\bz_E
+
\bz_\xi ,
\label{eq:ridge_grouped_decomposition}
\end{equation}
where
\[
\bz_E
:=
\U_m(\Ss_m+\rho\I_m)^{-1}
\frac1{n_l}\Xm^\top\E\U_\perp^\top\wstar,
\qquad
\bz_\xi
:=
\U_m(\Ss_m+\rho\I_m)^{-1}
\frac1{n_l}\Xm^\top\bxi .
\]
Thus the deterministic part of the estimator is the regularized projection of
the effective signal $\wstar+A\bq$.

Define
\begin{equation}
\mathbf e_\star
:=
\frac{\wstar}{\|\wstar\|},
\qquad
c_p
:=
\frac{A}{\|\wstar\|}
=
\frac{\lambda_l-1}{\tilde\lambda_l}
\frac{\wstar^\top\Pperp\bq}{\|\wstar\|}.
\end{equation}
In the high-dimensional limit,
\begin{equation}
c_p
\xrightarrow{\rm a.s.}
c_\star
:=
\frac{\lambda_l-1}{\bar\lambda_l}
\bar p_{\perp,\wstar,\bq},
\qquad
\bar\lambda_l
=
1+(\lambda_l-1)\bar p_{m,\bq,\bq}.
\label{eq:c_star_definition}
\end{equation}
We then define
\begin{equation}
\bu_\star
:=
\mathbf e_\star+c_\star\bq,
\qquad
s_\star
:=
\|\bu_\star\|.
\label{eq:u_star_definition}
\end{equation}
Equivalently,
\begin{equation}
s_\star^2
=
1+c_\star^2+2c_\star r_q,
\qquad
r_q
:=
\lim
\frac{\wstar^\top\bq}{\|\wstar\|}.
\end{equation}

Finally, define
\begin{equation}
\tau^2
:=
\bar w^\ast \bar\sigma_{\rm eff}^2+\sigma^2,
\qquad
\bar\sigma_{\rm eff}^2
=
\bar p_{\perp,\wstar,\wstar}
+
\frac{\lambda_l-1}{\bar\lambda_l}
\bar p_{\perp,\wstar,\bq}^2 .
\label{eq:tau_definition_ridge}
\end{equation}

\begin{theorem}[Asymptotic estimation error for ridge regression]
\label{thm:ridge_estimation_error}
Under the assumptions of Section~\ref{sec:setup}, for fixed $\rho>0$,
\begin{equation}
\boxed{
E^{\rm est}_{\infty,\rho}
=
\bar w^\ast
\left[
1
-
2s_\star
\pi^{(1)}_{\rho,\mathbf e_\star,\bu_\star}
+
s_\star^2
\pi^{(2)}_{\rho,\bu_\star,\bu_\star}
\right]
+
\tau^2 V_\rho(\gamma_{\rm eff}) .
}
\label{eq:ridge_estimation_error_compact}
\end{equation}
\end{theorem}

\begin{proof}
From \eqref{eq:ridge_grouped_decomposition},
\[
\what_\rho
=
\Pi_\rho(\wstar+A\bq)
+
\bz_E
+
\bz_\xi .
\]
Conditionally on $\Xm$ and $\U_m$, the random vectors $\bz_E$ and $\bz_\xi$
are centred and independent of the deterministic bias term. Hence the
corresponding cross terms vanish by the same conditional-centering and
concentration arguments as in the proof of
Theorem~\ref{thm:estimation_error}.

For the deterministic bias,
\[
\wstar+A\bq
=
\|\wstar\|(\mathbf e_\star+c_p\bq),
\]
and $c_p\to c_\star$ almost surely. Therefore,
\[
\frac1{\|\wstar\|^2}
\left\|
\wstar-\Pi_\rho(\wstar+A\bq)
\right\|^2
\to
\left\|
\mathbf e_\star-\Pi_\rho\bu_\star
\right\|^2 .
\]
Expanding,
\[
\left\|
\mathbf e_\star-\Pi_\rho\bu_\star
\right\|^2
=
1
-
2\mathbf e_\star^\top\Pi_\rho\bu_\star
+
\bu_\star^\top\Pi_\rho^2\bu_\star .
\]
By Lemma~\ref{lem:regularized_bilinear_effective_projection},
\[
\mathbf e_\star^\top\Pi_\rho\bu_\star
\xrightarrow{\rm a.s.}
s_\star
\pi^{(1)}_{\rho,\mathbf e_\star,\bu_\star},
\]
and
\[
\bu_\star^\top\Pi_\rho^2\bu_\star
\xrightarrow{\rm a.s.}
s_\star^2
\pi^{(2)}_{\rho,\bu_\star,\bu_\star}.
\]
Using $\|\wstar\|^2\to\bar w^\ast$ gives the bias contribution in
\eqref{eq:ridge_estimation_error_compact}.

For the variance terms, conditionally on $\Xm,\U_m$,
\[
\mathbb E_\xi\|\bz_\xi\|^2
=
\frac{\sigma^2}{n_l}
\Tr\left[
\Ss_m(\Ss_m+\rho\I_m)^{-2}
\right].
\]
Similarly,
\[
\mathbb E_E\|\bz_E\|^2
=
\frac{\|\wstar\|^2\sigma_{\rm eff}^2}{n_l}
\Tr\left[
\Ss_m(\Ss_m+\rho\I_m)^{-2}
\right],
\]
where
\[
\sigma_{\rm eff}^2
\xrightarrow{\rm a.s.}
\bar\sigma_{\rm eff}^2 .
\]
Thus
\[
\mathbb E\|\bz_E+\bz_\xi\|^2
=
\left(
\|\wstar\|^2\sigma_{\rm eff}^2+\sigma^2
\right)
\frac1{n_l}
\Tr\left[
\Ss_m(\Ss_m+\rho\I_m)^{-2}
\right].
\]
By Lemma~\ref{lem:regularized_variance_trace}, this converges almost surely to
$\tau^2V_\rho(\gamma_{\rm eff})$. Combining the bias and variance terms gives
\eqref{eq:ridge_estimation_error_compact}.
\end{proof}

\subsubsection{Asymptotic generalisation error for ridge regression}

For a fresh test point drawn from the downstream distribution,
\[
\x\sim\mathcal N(0,\bSigma_l),
\qquad
\bSigma_l
=
\I_p+(\lambda_l-1)\bq\bq^\top,
\qquad
y=\x^\top\wstar+\xi,
\]
the prediction error is
\begin{equation}
E^{\rm gen}_{p,\rho}
=
(\what_\rho-\wstar)^\top
\bSigma_l
(\what_\rho-\wstar)
+
\sigma^2 .
\end{equation}
Hence,
\begin{equation}
E^{\rm gen}_{p,\rho}
=
\|\what_\rho-\wstar\|^2
+
(\lambda_l-1)
\left[
\bq^\top(\what_\rho-\wstar)
\right]^2
+
\sigma^2 .
\end{equation}
The first term is exactly the estimation error analysed in
Theorem~\ref{thm:ridge_estimation_error}. It remains to compute the residual
along the downstream spike.

Using the grouped ridge decomposition,
\[
\what_\rho
=
\Pi_\rho(\wstar+A\bq)
+
\bz_E+\bz_\xi,
\]
we obtain
\[
\begin{aligned}
\frac{\bq^\top(\wstar-\what_\rho)}{\|\wstar\|}
&=
r_q
-
\frac{\bq^\top\Pi_\rho\wstar}{\|\wstar\|}
-
\frac{A}{\|\wstar\|}
\bq^\top\Pi_\rho\bq
\\
&\qquad
-
\frac{\bq^\top\bz_E}{\|\wstar\|}
-
\frac{\bq^\top\bz_\xi}{\|\wstar\|}.
\end{aligned}
\]
The deterministic terms converge almost surely as
\[
\frac{\bq^\top\Pi_\rho\wstar}{\|\wstar\|}
\to
\pi^{(1)}_{\rho,\wstar,\bq},
\qquad
\bq^\top\Pi_\rho\bq
\to
\pi^{(1)}_{\rho,\bq,\bq},
\]
and
\[
\frac{A}{\|\wstar\|}
\to
\frac{\lambda_l-1}{\bar\lambda_l}
\bar p_{\perp,\wstar,\bq}.
\]
The random terms projected onto the fixed direction $\bq$ vanish almost surely
by the same argument used in the proof of
Theorem~\ref{thm:generalisation_error}. Therefore,
\begin{equation}
\frac{\bq^\top(\wstar-\what_\rho)}{\|\wstar\|}
\xrightarrow{\rm a.s.}
r_q
-
\pi^{(1)}_{\rho,\wstar,\bq}
-
\frac{\lambda_l-1}{\bar\lambda_l}
\bar p_{\perp,\wstar,\bq}
\pi^{(1)}_{\rho,\bq,\bq}.
\label{eq:ridge_spike_residual}
\end{equation}

\begin{theorem}[Asymptotic generalisation error for ridge regression]
\label{thm:ridge_generalisation_error}
Under the same assumptions as above,
\begin{equation}
\boxed{
\begin{aligned}
E^{\rm gen}_{\infty,\rho}
&=
E^{\rm est}_{\infty,\rho}
+
\bar w^\ast(\lambda_l-1)
\left[
r_q
-
\pi^{(1)}_{\rho,\wstar,\bq}
-
\frac{\lambda_l-1}{\bar\lambda_l}
\bar p_{\perp,\wstar,\bq}
\pi^{(1)}_{\rho,\bq,\bq}
\right]^2
+\sigma^2 .
\end{aligned}
}
\label{eq:ridge_generalisation_error}
\end{equation}
\end{theorem}

\begin{proof}
The decomposition above shows that the deterministic spike residual converges
to~\eqref{eq:ridge_spike_residual}. The terms involving $\bz_E$ and $\bz_\xi$
along the fixed direction $\bq$ vanish almost surely by the same
conditional-Gaussian concentration argument as in the proof of
Theorem~\ref{thm:generalisation_error}. Substitution into the decomposition of
$E^{\rm gen}_{p,\rho}$ gives the result.
\end{proof}

\subsubsection{Asymptotic training error for ridge regression}

For completeness, we also record the training error. This is no longer the
simple interpolation expression from the ridgeless case because ridge shrinkage
leaves residual error even in the retained signal subspace.

Define
\begin{equation}
\kappa_{\rho,\ba,\bb}
:=
\bar p_{m,\ba,\ba}^{1/2}
\bar p_{m,\bb,\bb}^{1/2}
\int
\frac{\rho^2x}{(x+\rho)^2}
\,dG^{l,{\rm sample}}_{\tilde\ba,\tilde\bb}(x),
\label{eq:kappa_rho_definition}
\end{equation}
and
\begin{equation}
T_\rho(\gamma_{\rm eff})
:=
1
-
2\gamma_{\rm eff}
\int
\frac{x}{x+\rho}\,dF_{\gamma_{\rm eff}}(x)
+
\gamma_{\rm eff}
\int
\left(
\frac{x}{x+\rho}
\right)^2
dF_{\gamma_{\rm eff}}(x).
\label{eq:T_rho_definition}
\end{equation}

\begin{theorem}[Asymptotic training error for ridge regression]
\label{thm:ridge_training_error}
Under the same assumptions as above, the ridge training error converges almost
surely to
\begin{equation}
\begin{aligned}
E^{\rm train}_{\infty,\rho}
&=
\bar w^\ast
\left[
\kappa_{\rho,\wstar,\wstar}
+
2c_\star\kappa_{\rho,\wstar,\bq}
+
c_\star^2\kappa_{\rho,\bq,\bq}
\right]
+
T_\rho(\gamma_{\rm eff})
\left(
\bar w^\ast\bar\sigma_{\rm eff}^2+\sigma^2
\right).
\end{aligned}
\label{eq:ridge_training_error}
\end{equation}
\end{theorem}

\begin{proof}
Let
\[
H_\rho
:=
\frac1{n_l}\Xm(\Ss_m+\rho\I_m)^{-1}\Xm^\top.
\]
Its nonzero eigenvalues are
\[
\frac{b_i}{b_i+\rho},
\]
where $b_i$ are the eigenvalues of $\Ss_m$. The training residual is
\[
\y_l-\X_l\Pm\what_\rho
=
(\I_{n_l}-H_\rho)\y_l .
\]
Using the same conditional decomposition as before,
\[
\y_l
=
\Xm\U_m^\top\wstar
+
A\Xm\U_m^\top\bq
+
\E\U_\perp^\top\wstar
+
\bxi .
\]
For the deterministic retained signal part, set
\[
\bs
:=
\U_m^\top\wstar+A\U_m^\top\bq .
\]
Then
\[
(\I_{n_l}-H_\rho)\Xm\bs
=
\rho\Xm(\Ss_m+\rho\I_m)^{-1}\bs .
\]
Therefore
\[
\frac1{n_l}
\left\|
(\I_{n_l}-H_\rho)\Xm\bs
\right\|^2
=
\bs^\top
\rho^2
\Ss_m(\Ss_m+\rho\I_m)^{-2}
\bs .
\]
Applying the same regularized signed spectral-measure result with the function
$x\mapsto\rho^2x/(x+\rho)^2$, and using
$A/\|\wstar\|\to c_\star$, gives the first line of
\eqref{eq:ridge_training_error}.

For the independent model-mismatch and label-noise terms, conditionally on
$\Xm,\U_m$,
\[
\begin{aligned}
\frac1{n_l}
\mathbb E
\left[
\|(\I_{n_l}-H_\rho)
(\E\U_\perp^\top\wstar+\bxi)\|^2
\right]
&=
\left(
\|\wstar\|^2\sigma_{\rm eff}^2+\sigma^2
\right)
\frac1{n_l}
\Tr[
(\I_{n_l}-H_\rho)^2
].
\end{aligned}
\]
Since
\[
\frac1{n_l}\Tr[(\I_{n_l}-H_\rho)^2]
=
1
-\frac{2}{n_l}\Tr H_\rho
+\frac1{n_l}\Tr H_\rho^2,
\]
and the nonzero eigenvalues of $H_\rho$ are $b_i/(b_i+\rho)$,
\[
\frac1{n_l}\Tr[(\I_{n_l}-H_\rho)^2]
\xrightarrow{\rm a.s.}
T_\rho(\gamma_{\rm eff}).
\]
Finally,
$\sigma_{\rm eff}^2\to\bar\sigma_{\rm eff}^2$ almost surely, which proves the
claim.
\end{proof}

\subsection{Theory-experiment validation}
Below, we present additional plots validating the theoretical predictions through numerical experiments. We introduce an overlap measure for the task-spike alignment $\theta$, which can be defined through the overlap $r_v^2$ as
\[
\frac{(\wstar^\top \bv)^2}{\|\wstar\|^2} = r_v^2 = (\cos \theta) ^2.
\]

\label{app:Theory-Experiment Validation}
\begin{figure}[t]
    \centering
    \includegraphics[width=\linewidth]{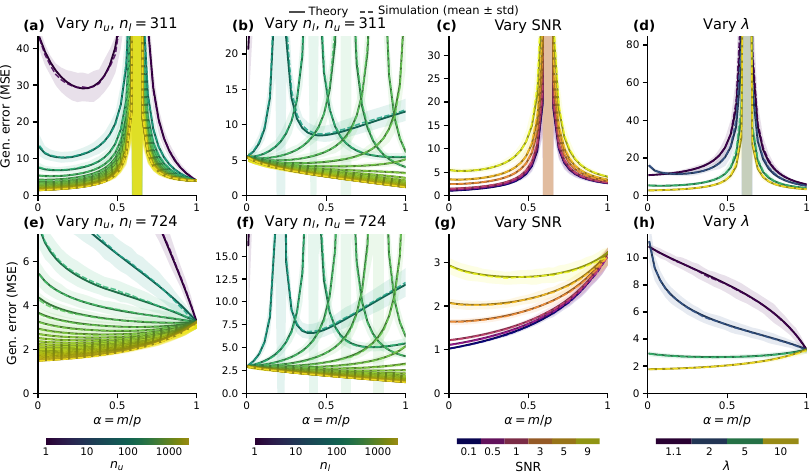}
    \caption{Generalisation error of PCR as a function of the number of retained components $\alpha= \frac{m}{p}$. Solid lines show the theoretical prediction; dashed lines show the empirical mean $\pm$ one standard deviation across trials. Each column varies one parameter while holding the others fixed: \textbf{(a,e)} varying the pre-training dataset size $n_u$ with $n_l$ fixed; \textbf{(b,f)} varying the downstream dataset size $n_l$ with $n_u$ fixed; \textbf{(c,g)} varying the signal-to-noise ratio SNR; \textbf{(d,h)} varying the spike eigenvalue $\lambda$. The two rows correspond to two regimes: smaller (top) and larger (bottom) values of $n_u$ and $n_l$. Results are shown for a rank-1 spiked covariance model with perfectly aligned signal ($r_v^2 = 1$).}
    \label{fig:gen_error_full}
\end{figure}
\subsubsection{Error curves}\label{app:errcurves}
Figure~\ref{fig:gen_error_full} illustrates the range of generalisation error curves across parameter settings. When the downstream sample size is limited ($n_l < p$), the error exhibits a double descent profile, with the peak occurring at $\alpha \approx 1/\gamma_l$. Increasing the size of the pretraining dataset $n_u$ uniformly reduces the overall error scale (panels \textbf{(a,e)}), while varying $n_l$ shifts the location of the interpolation threshold (panels \textbf{(b,f)}).
Increasing the SNR raises the generalisation error (panels \textbf{(c,g)}), as it amplifies the contribution of signal-dependent bias. This increase is non-uniform across $\alpha$: smaller values of $\alpha$ are affected more strongly, since a larger fraction of the signal lies outside the retained subspace. A similar non-uniform effect is observed when varying the spike strength $\lambda$ (panels \textbf{(d,h)}).

We further decompose these effects by examining the estimation error components in Figure~\ref{fig:app_fig_estimations}. The bias due to missing signal is determined by the downstream sample regime. In the overparametrized regime ($n_l < p$), this bias admits a strictly positive lower bound (panel \textbf{(a)}), reflecting the inability to fully reconstruct the signal. In contrast, when $n_l > p$, this component can vanish (panel \textbf{(e)}). Independently of this constraint, the bias decreases with $\alpha$, since increasing $\alpha$ enlarges the subspace $\Pm$ and thus increases the fraction of $\wstar$ that can be represented.
The leak-induced bias follows a similar decreasing trend in $\alpha$, as it is controlled by $\bar p_{\perp,\wstar,\bv}$, which shrinks as the orthogonal complement $\Pperp$ becomes smaller. However, its dependence on $\lambda$ is non-monotonic (panels \textbf{(b,f)}): it is negligible for weak spikes (e.g.\ $\lambda=1.1$), increases at intermediate values (e.g.\ $\lambda=2,5$), and decreases again for strong spikes (e.g.\ $\lambda=10$). This reflects a trade-off between decreasing leak size $\bar p_{\perp,\wstar,\bv}$ and increasing the effective component of the leak vector $\bv$ in the subspace $\X_l \Pm$ through $\bar\pi_{\bv}$. For the parameter ranges considered, the magnitude of the leak term remains small relative to the other components.
The variance behaves as expected: in the regime $n_l < p$, it decreases monotonically with $\lambda$, reflecting improved signal-to-noise separation (panels \textbf{(c,g)}).

Finally, the generalisation error (Figure~\ref{fig:app_fig_generlisation}) includes two additional contributions relative to the estimation error. The first is the additive noise variance $\sigma^2$ (panels \textbf{(c,g)}). The second is an additional bias term induced by the spike structure in the test distribution. This term combines a residual missing signal component with a correction from the leak, and it is amplified by $\lambda$. Its structure leads to a discontinuity at $\alpha = \gamma_l$ (panel \textbf{(b)}), inherited from the behaviour of the ``missing signal'' bias. Overall, the generalisation error closely tracks the estimation error, with the main differences appearing at small $\alpha$, where these additional bias contributions are most pronounced (panels \textbf{(e,f)}).

\begin{figure}[h!]
    \centering
    \includegraphics[width=\linewidth]{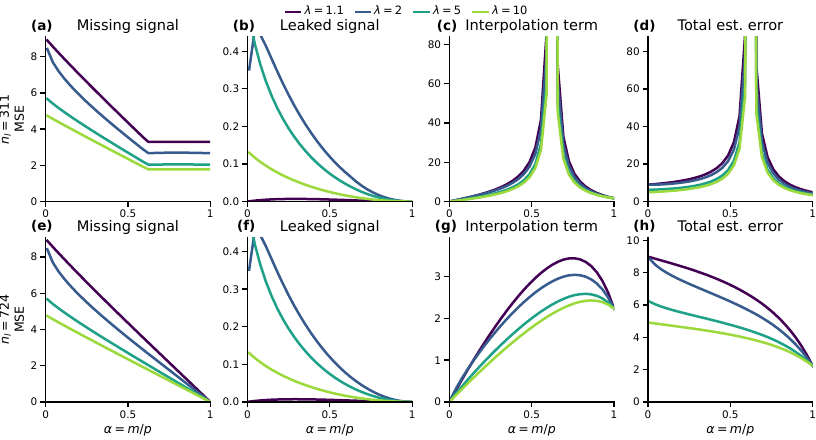}
   \caption{Components of the theoretical estimation error as a function of $\alpha$ for different values of $\lambda$. The total error (panels \textbf{(d, h)}) is decomposed into the missing signal term (panels \textbf{(a, e)}), the leak term (panels \textbf{(b, f)}), the variance term (panels \textbf{(c, g)}). All plots assume fixed parameters $n_u=311$, $n_l=[311, 724]$, $\mathrm{SNR} = \frac{\|\wstar \|^2}{\sigma^2}$, and a target signal misaligned with the spike eigenvector ($r_v^2 = 0.75$). The number of features is set to $p=500$.}
    \label{fig:app_fig_estimations}
\end{figure}

\begin{figure}[h!]
    \centering
    \includegraphics[width=\linewidth]{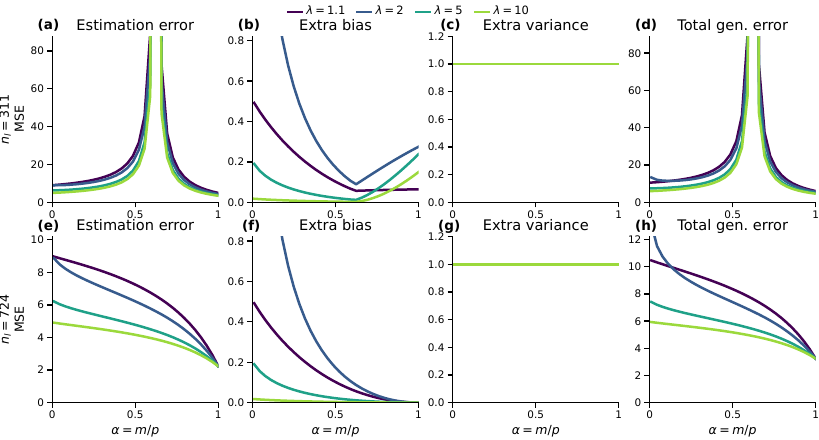}
    \caption{Components of the theoretical generalisation error as a function of $\alpha$ for different values of $\lambda$. The total error (panels \textbf{(d, h)}) is decomposed into the estimation error (panels \textbf{(a, e)}), the missing signal (panels \textbf{(b, f)}), and the leak term \textbf{(c, g)}. All plots assume fixed parameters $n_u=311$, $n_l=[311, 724]$, $\mathrm{SNR} = \frac{\|\wstar \|^2}{\sigma^2}$, and a target signal misaligned with the spike eigenvector ($r_v^2 = 0.75$). The number of features is set to $p=500$.}
    \label{fig:app_fig_generlisation}
\end{figure}

\subsubsection{Optimal dimensionality}
\label{app:Optimal}
In Figures \ref{fig:app_fig_alpha_SNR}, \ref{fig:app_fig_alpha_SNR_2},  \ref{fig:app_fig_alpha_lambda},\ref{fig:app_fig_alpha_lambda_2}, and \ref{fig:app_fig_r} we present various phase portraits in $\alpha$, reflecting their dependence on the signal-to-noise ratio (SNR), the spike strength ($\lambda$), the task overlap $r_v$ and the spike overlap $r$. 
For the optimal $\alpha$ minimizing the training error (Figures \ref{fig:app_fig_alpha_SNR_2} and \ref{fig:app_fig_alpha_lambda_2}), the SNR, task overlap, and spike strength have little effect.
In contrast, for the optimal $\alpha$ minimizing the generalisation error, the phase-transition boundaries and the size of the bulk region are highly sensitive to the task overlap, SNR, and spike strength ($\lambda$). Figures \ref{fig:app_fig_alpha_SNR} and \ref{fig:app_fig_alpha_lambda} show that the phase-transition boundary flattens into a horizontal line at a fixed $n_l$ \emph{(i)} as the task-spike overlap decreases, \emph{(ii)} as the SNR increases or \emph{(iii)} as the spike strength decreases. At the same time, the bulk region expands along the $n_u$ axis, thereby reducing the region of parameter space in which compression is used.
This behaviour can be explained by the fact that a weaker signal is harder to detect, leading to a less accurate estimate of the prior. Consequently, exploiting this poorly-estimated prior through compression is not advantageous across many regimes, even when the amount of new data is large.

Figure~\ref{fig:app_fig_alpha_ridge} shows the generalization-optimal $\alpha^*$ for the ridge extension of Section~\ref{app:ridge_extension}. We omit the training-optimal plots: with a nonzero ridge parameter the training error decreases monotonically in the number of retained components, so $\alpha^* = 1$ everywhere. For generalization, increasing the ridge parameter shrinks the region of the $(n_u, n_l)$ plane in which compression is optimal. Ridge also regularizes the downstream regression near its interpolation threshold $n_l = p$, which removes the error spike that compression otherwise mitigates there. As a result, the band of compression ($\alpha^* < 1$) that separates the low-$n_u$, low-$n_l$ "no-compression" corner from the large-$n_l$ no-compression region disappears for sufficiently large ridge. The two regions merge, and $\alpha^*$ no longer varies non-monotonically with $n_l$.

\begin{figure}[h]
\begin{center}
\includegraphics[width=1\textwidth]{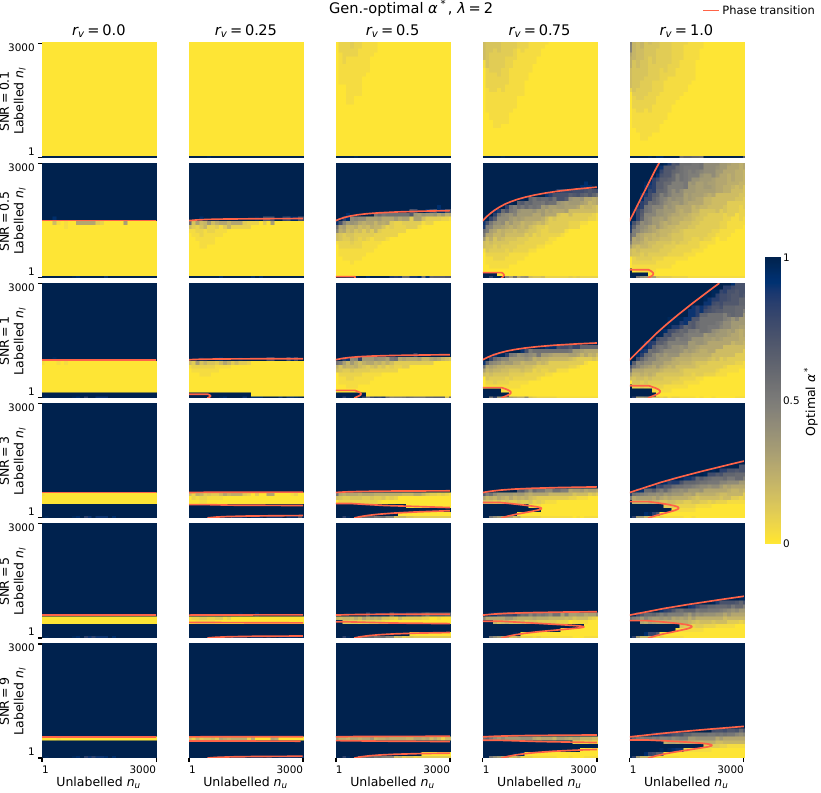}
\end{center}
\caption{Optimal value of $\alpha$ that minimizes the generalisation error, shown as a function of SNR and task alignment $r_v$. We fix $\lambda = 2$. Red lines indicate the approximate phase transitions.} 
\label{fig:app_fig_alpha_SNR}
\end{figure}

\begin{figure}[h]
\begin{center}
\includegraphics[width=1\textwidth]{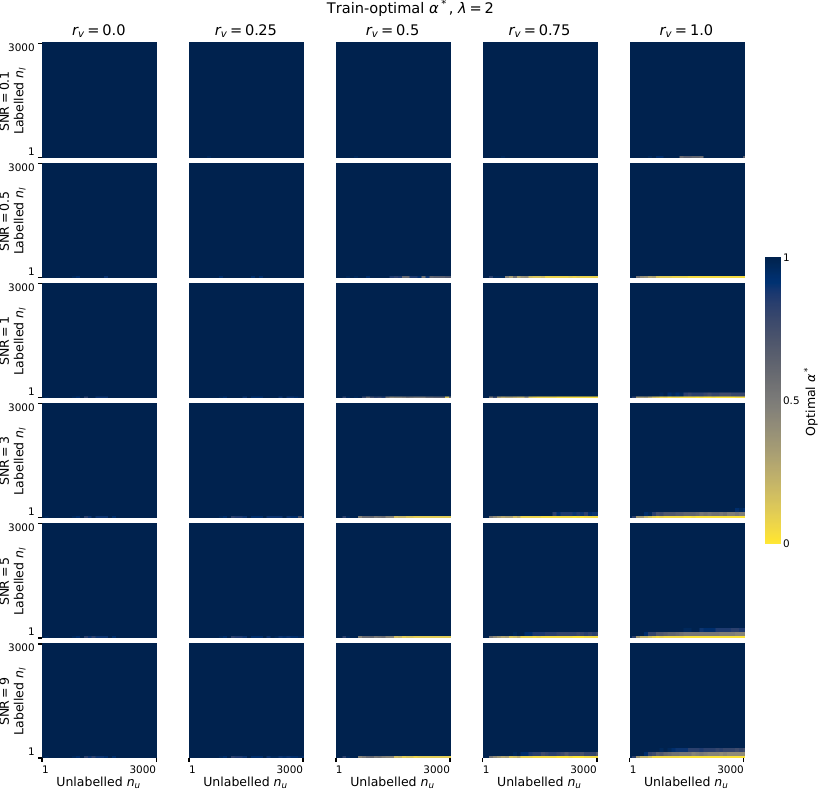}
\end{center}
\caption{Optimal value of $\alpha$ that minimizes the training error, shown as a function of SNR and task alignment $r_v$. We fix $\lambda = 2$.} 
\label{fig:app_fig_alpha_SNR_2}
\end{figure}

\begin{figure}[h!]
\begin{center}
\includegraphics[width=1\textwidth]{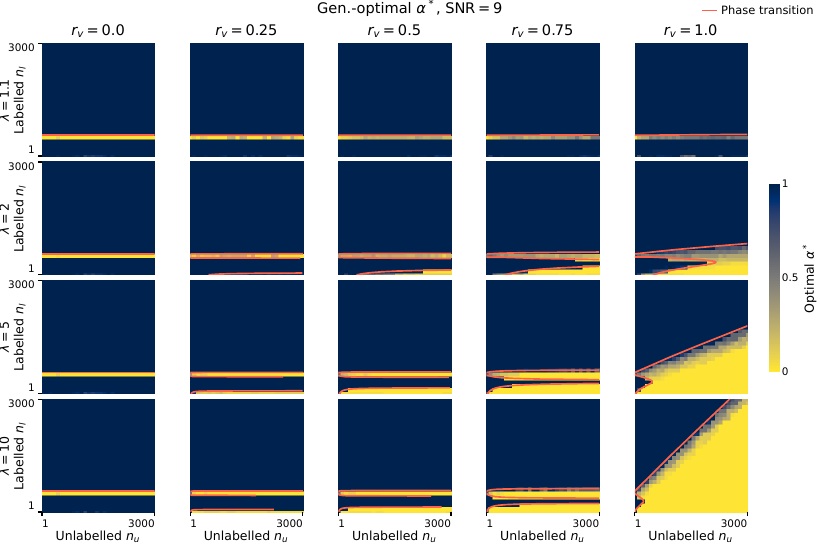}
\end{center}
\caption{Optimal value of $\alpha$ that minimizes the generalisation error, shown as a function of spike strength $\lambda$ and task alignment $r_v$ with $\wstar$. We fix $\mathrm{SNR}=9$. Red lines indicate the approximate phase transitions.} 
\label{fig:app_fig_alpha_lambda}
\end{figure}

\begin{figure}[h!]
\begin{center}
\includegraphics[width=1\textwidth]{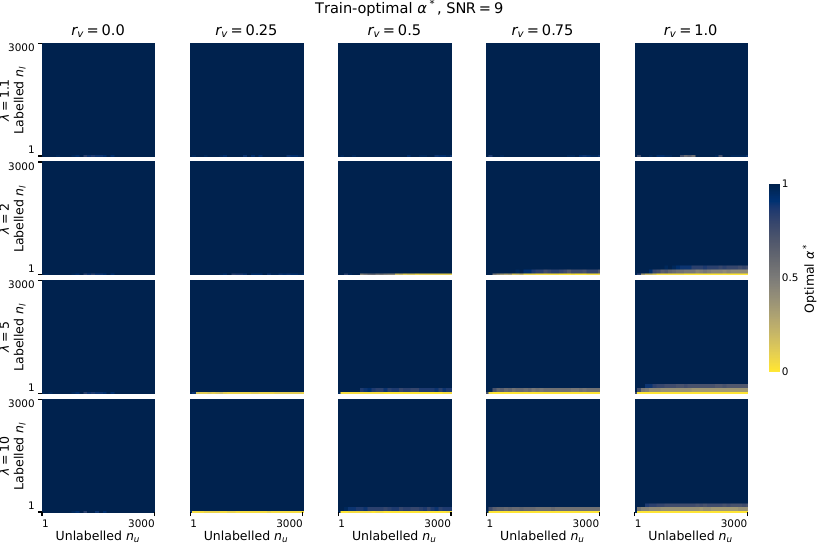}
\end{center}
\caption{Optimal value of $\alpha$ that minimizes the training error, shown as a function of spike strength $\lambda$ and task alignment $r_v$ with $\wstar$. We fix $\mathrm{SNR}=9$.} \label{fig:app_fig_alpha_lambda_2}
\end{figure}

\begin{figure}[h!]
\begin{center}
\includegraphics[width=1\textwidth]{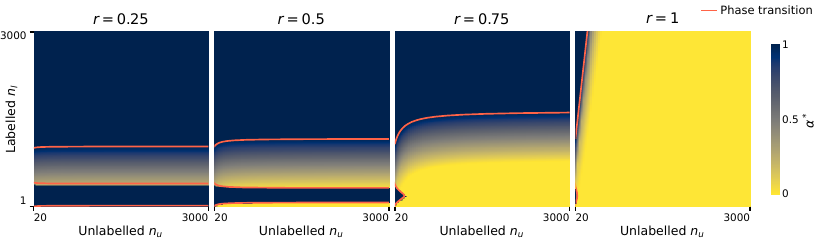}
\end{center}
\caption{Optimal value of $\alpha$ that minimizes the generalisation error, shown as a function of prior and downstream sample sizes $n_p$ and $n_d$, for varying spike alignment $r$ (with $\lambda_u = 10$, $\lambda_\ell = 5$ fixed). We fix $\mathrm{SNR} = 5$. Red lines indicate the approximate phase transition.
} \label{fig:app_fig_r}
\end{figure}

\begin{figure}[h!]
    \centering
    \includegraphics[width=\linewidth]{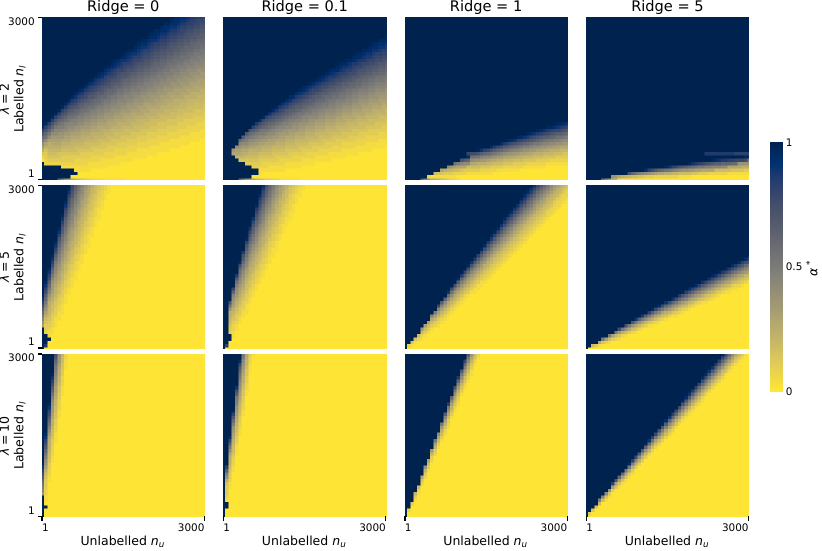}
    \caption{Optimal value of $\alpha$ that minimizes the generalization error in ridge regression, with ridge and spike parameters varied.}
    \label{fig:app_fig_alpha_ridge}
\end{figure}

\newpage

\subsection{Additional experiments}
\label{app:experiments}

We provide additional figures from the experiments outlined in the main text in Section \ref{sec:experiements}.

\subsubsection{Additional autoencoder experiments}\label{app:autoencoder}

Figure~\ref{fig:app_autoencoders_larger_snr} extends the comparison in panels \textbf{(a-d)} of Figure~\ref{fig:fig3}. The top row shows simulations that directly match our theoretical setting directly :PCA followed by regression. In the middle and bottom rows, PCA is replaced by unsupervised representation learning with linear and a nonlinear one-hidden-layer autoencoder, respectively. In both the spiked identity and general covariance settings, the linear autoencoder closely reproduces the theoretical predictions, consistent with its recovering the same subspace as PCA.

Figure~\ref{fig:app_autoencoders_smaller_snr} shows the corresponding results at lower SNR. The phase boundaries shift in the same qualitative way as predicted by the theory: the optimal $\alpha$ is low across a larger part of the $(n_u, n_l)$ plane, indicating that the theory captures how SNR affects the benefits of compression beyond the exact PCA setting.

To examine where autoencoders depart from PCA, Figure~\ref{fig:app_autoencoders_overlaps} compares the theoretical overlap $\Pm \bv$ with empirical overlaps obtained by passing $\bv$ through trained linear and nonlinear autoencoders with varying bottleneck dimension. The linear autoencoder deviates from PCA only when $n_u < p$, whereas the nonlinear autoencoder shows larger discrepancies across a broader range of $n_u$. One source of these discrepancies is the "yellow band" in the plots, which arises from a double-descent effect during autoencoder pretraining. In this region, gradient descent can reduce training loss by driving the encoder's pre-activations to extreme values. The bounded sigmoid outputs then approach binary codes, allowing the model to effectively memorize a small training set while saturation destroys continuous, graded information. The linear autoencoder does not exhibit this behavior: its optimum is exact PCA, a bounded-norm solution. In our experiments, regularizing the pre-activations to keep them small resolves this effect.

Despite these deviations, when $n_u < p$, the optimal bottleneck dimension that minimizes both training and generalisation error remains close to $\alpha \approx 1$. Accordingly, the theory continues to predict the performance of linear autoencoder-based pipelines well across all values of $n_u$. The nonlinear autoencoder, by contrast, shows a persistent mismatch when $n_u > p$, which diminishes only with sufficiently large pretraining datasets (empirically around $n_u \approx 3p$). Its magnitude also depends on SNR, as the comparison between Figures~\ref{fig:app_autoencoders_larger_snr} and~\ref{fig:app_autoencoders_smaller_snr} shows. 

To quantify the agreement, we compare the theoretically optimal representation dimensionality with the empirical optimum for PCA followed by regression, a linear autoencoder, and a nonlinear autoencoder. The average discrepancy from the PCR theory is \(7.2\%\) for the linear autoencoder, and \(15.8\%\) for the nonlinear autoencoder (Table~\ref{tab:optimal-dimensionality-percent}). Together, these results show that the theory’s main conclusions extend to both linear and nonlinear networks: compression is beneficial when downstream data are limited, and the extent of this benefit depends on SNR.

\begin{table}[t]
    \centering
    \caption{Average absolute difference in optimal dimensionality, as a percentage of the number of features \(p = 200\), between theory and PCA with regression, a linear autoencoder (AE), and a nonlinear autoencoder, averaged over the \((n_u, n_l)\) grid. Train-opt and Gen-opt compare the training- and generalisation-optimal dimensionalities; Combined is their average.}
    \label{tab:optimal-dimensionality-percent}
    \begin{tabular}{lccc}
        \toprule
        \textbf{Model}
        & \textbf{Train-opt}
        & \textbf{Gen-opt}
        & \textbf{Combined} \\
        \midrule
        Linear AE           & \(5.59\%\) & \(8.87\%\) & \(7.23\%\) \\
        Nonlinear AE        & \(12.95\%\) & \(18.66\%\) & \(15.81\%\) \\
        PCA + Regression    & \(6.11\%\) & \(8.08\%\) & \(7.10\%\) \\
        \bottomrule
    \end{tabular}
\end{table}

\begin{figure}[h!]
\begin{center}
\includegraphics[width=\textwidth]{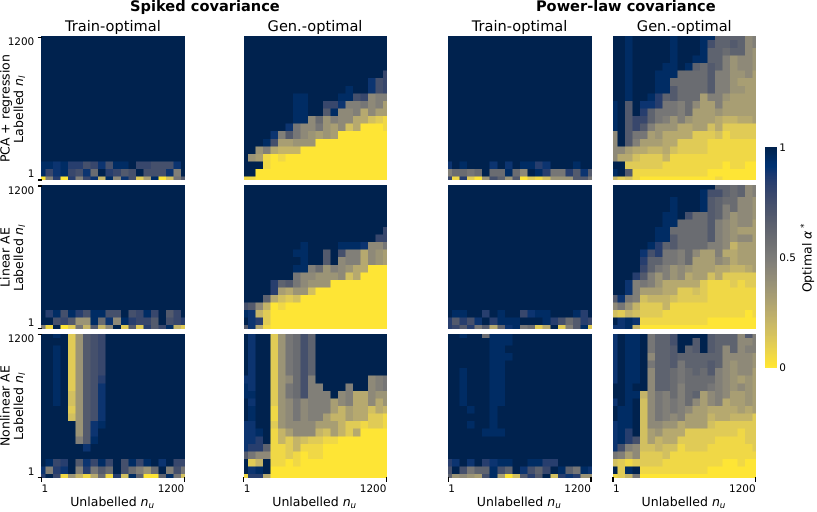}
\end{center}
\caption{Extended version of panels \textbf{(a-d)} in Figure~\ref{fig:fig3}: optimal bottleneck size in autoencoders on inputs with different covariance structure. We compare PCA on pretraining data (top row) corresponding to our theoretical setup, a linear autoencoder (middle row) and a non-linear autoencoder (bottom row) trained on Gaussian data, with varying bottleneck size $m$ ($\alpha = m / p$). The input covariance on the first two columns is $\bSigma= \mathbf I + \lambda \bv \bv^\top$, while in the last two columns the covariance is a Toeplitz matrix $\mathbf H$ with a spike, i.e., $\bSigma = \mathbf H + \lambda \bv \bv^\top$, where $H_{i,j} = \rho^{|i-j|}$ with $\rho=0.5$. The autoencoder results match those of PCA, and the identity vs general covariance exhibit a similar phenomenology. We fix $\mathrm{SNR}=9.0, \lambda=5.0$. } 
\label{fig:app_autoencoders_larger_snr}
\end{figure}

\begin{figure}[h!]
\begin{center}
\includegraphics[width=\textwidth]{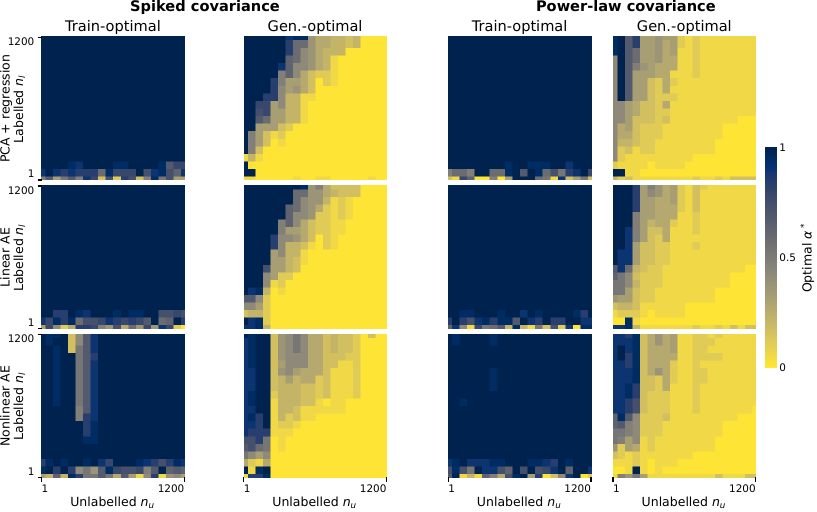}
\end{center}
\caption{Effect of SNR on optimal bottleneck size in autoencoders on inputs with different covariance structure. We compare PCA on pretraining data (top row) corresponding to our theoretical setup, a linear autoencoder (middle row) and a non-linear autoencoder (bottom row) trained on Gaussian data, with varying bottleneck size $m$ ($\alpha = m / p$). The input covariance on the first two columns is $\bSigma= \mathbf I + \lambda \bv \bv^\top$, while in the last two columns the covariance is a Toeplitz matrix $\mathbf H$ with a spike,  i.e., $\bSigma = \mathbf H + \lambda \bv \bv^\top$, where $H_{i,j} = \rho^{|i-j|}$ with $\rho=0.5$. The autoencoder results match those of PCA, and the identity vs general covariance exhibit a similar phenomenology. We fix $\mathrm{SNR}=1.8, \lambda=5.0$. } 
\label{fig:app_autoencoders_smaller_snr}
\end{figure}

\begin{figure}[h!]
    \centering
    \includegraphics[width=\linewidth]{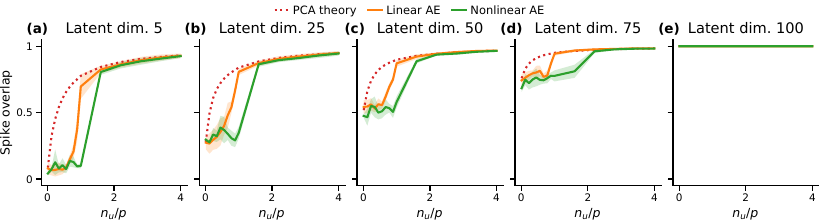}
    \caption{Spike direction recovery as a function of sample ratio $n_u / p$ for linear and nonlinear autoencoders. Each panel shows the overlap between the singular vectors of the encoder weights and the population spike $\bv$ as a function of the sample ratio $n_u/p$, for a fixed latent dimension $m \in \{5, 25, 50, 75, 100\}$. Data are generated from a spiked covariance model in $p = 100$ dimensions with spike strength $\lambda = 5$. Solid lines show empirical means over 5 random seeds; shaded bands show $\pm 1$ standard deviation. The dotted line shows the theoretical PCA overlap $\bar p_{\bv, \bv} $ derived from Lemma~\ref{lem:overlap_measure}. Both linear and nonlinear autoencoders show a mismatch in low sample regimes, with a more evident mismatch in nonlinear autoencoders and in autoencoders with larger bottlenecks. For large sample regimes, linear and nonlinear autoencoders match PCA well, and show the phase transitions predicted by our theory (see panels \textbf{(a-d)} in Figure~\ref{fig:fig3} discussed in Section \ref{sec:experiements_1}, and Figures~\ref{fig:app_autoencoders_larger_snr}-\ref{fig:app_autoencoders_smaller_snr} discussed in Appendix \ref{app:autoencoder}).}
    \label{fig:app_autoencoders_overlaps}
\end{figure}

\subsubsection{Additional transformer experiments}
\label{app:experiments_real}

We first consider the setting where PCA is applied to representations extracted from the downstream task (Figure~\ref{fig:app_transformers_downstream_pca}), rather than from pretraining data as in Figure~\ref{fig:fig3}\textbf{(e)}. In this regime, compression remains beneficial for generalisation when the downstream dataset is small, with an optimal intermediate dimension $m$. The dependence on the pretraining dataset size $n_u$ is weaker than in the pretraining-PCA setting, but still observable.

We next return to applying PCA to pretraining data and report validation accuracy across a broader range of downstream dataset sizes (Figure~\ref{fig:app_transformers_val_accuracy}). For small downstream datasets, an intermediate compression level is optimal, particularly for models trained for longer durations. In contrast, for large downstream datasets, performance is maximized without compression ($m \approx p$). At the smallest dataset sizes, the optimal $m$ is less stable across runs.

Neither experimental setup provides a direct match to the theoretical model. In particular, the pretraining and downstream datasets differ not only in semantic structure but also in sequence length, which affects the learned representations and limits the transferability of PCA computed on pretraining data. Despite this mismatch, both PCA-based approaches match or outperform the no-compression baseline (Figure~\ref{fig:app_transformers_compare}).

\begin{figure}[h!]
\begin{center}
\includegraphics[width=\textwidth]{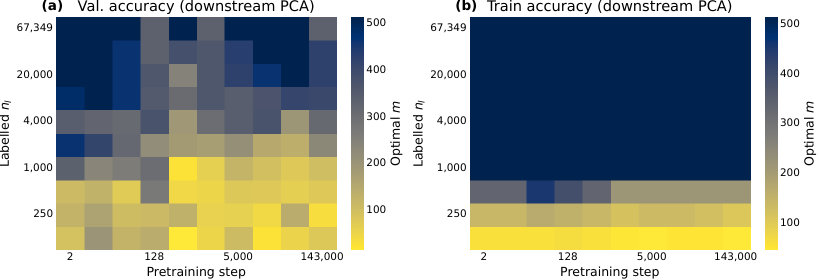}
\end{center}
\caption{Optimal $m$ using PCA on representations of the downstream task. We use the same layout as Figure~\ref{fig:fig3}\textbf{(e)}, but the PCA basis is fit on the downstream labelled set rather than on the pretraining corpus. \textbf{(a)} Optimal $m$ for validation accuracy; \textbf{(b)} optimal $m$ for training accuracy. Compared to doing PCA on the pretraining corpus (Figure~\ref{fig:fig3}\textbf{(e)}), downstream PCA tends to favor larger $m$ and the dependence on the checkpoint is weaker.} 
\label{fig:app_transformers_downstream_pca}
\end{figure}

\begin{figure}[h!]
\begin{center}
\includegraphics[width=\textwidth]{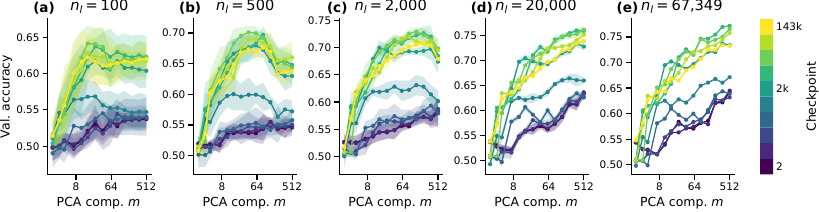}
\end{center}
\caption{Validation accuracy vs number of PCA components $m$ across five downstream task sizes. Note that the values of $m$ are on a log scale. Each panel shows a different value of $n_l$ spanning the full range from $n_l=100$ to $n_l=67{,}349$; curves are coloured by pretraining checkpoint. Shaded bands show the standard deviation over six random seeds. For smaller tasks, the curves have a clear peak and become progressively more monotone as $n_l$ grows, reflecting that small tasks benefit from aggressive compression while large tasks can exploit finer-grained directions in representation space.} 
\label{fig:app_transformers_val_accuracy}
\end{figure}

\begin{figure}[h!]
\begin{center}
\includegraphics[width=\textwidth]{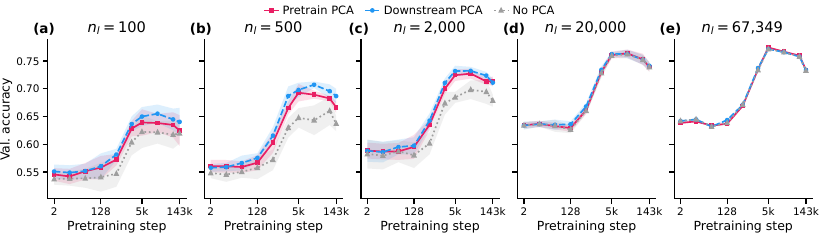}
\end{center}
\caption{Validation accuracy vs pretraining checkpoint for three probing conditions across five downstream task sizes. Each panel shows a different downstream task size $n_l$; the $x$-axis is the pretraining step on a log scale. Three curves are compared: Pretrain PCA (pink, solid) fits the PCA basis on \texttt{Pile} representations from the same checkpoint and selects the best $m$; Downstream PCA (blue, dashed) fits PCA on the downstream task representations (\texttt{sst2}) and selects the best $m$; No PCA (grey, dotted) uses the full-dimensional representations (largest available $m$). For each curve, the displayed value is the mean over six seeds at the per-seed optimal $m$, with the shaded area corresponding to 1 standard deviation. The gap between PCA and the no-PCA baseline is largest at small and intermediate $n_l$, where dimensionality reduction acts as regularization. Pretrain PCA matches downstream PCA for intermediate and large $n_l$, with the mismatch coming from the difference in the type of data available in the pretraining and downstream datasets.} 
\label{fig:app_transformers_compare}
\end{figure}

\subsubsection{More datasets}\label{sec:app_datasets}

We extended the transformer experiments to two model scales in the
Pythia-deduped family---70M and 410M parameters---and to nine GLUE benchmark
tasks: the eight classification tasks CoLA, SST-2, MRPC, QQP, QNLI, RTE,
WNLI, and MNLI, together with the regression task STS-B. Across both model
sizes and all nine datasets, we observe qualitatively similar behaviour in
the validation-optimal PCR dimension, with no systematic dependence on model
parameter count.

On both model sizes, we also tested the theoretical prediction that reducing
the effective signal-to-noise ratio favours stronger dimensionality
reduction. We independently corrupted 0\%, 20\%, 50\%, or 80\% of the
downstream training targets while leaving the validation targets unchanged.
For the classification tasks, each selected label was redrawn uniformly from
the task's class labels; for STS-B, selected continuous similarity scores
were redrawn uniformly from the observed range of training scores. The
downstream- and pretrain-PCA optimal dimensionality grids for Pythia-70M are
shown in Figures~\ref{fig:all_datasets_downstream} and
\ref{fig:all_datasets_pretrain} for three of the four corruption levels
($0$, $0.5$, $0.8$; the full sweep, including the intermediate $0.2$ level
omitted from the figures for space, is reported in
Tables~\ref{tab:optimal-m-downstream-pythia-70m-deduped}--\ref{tab:optimal-m-pretrain-pythia-410m-deduped}).

Consistent with the theoretical prediction, increasing target noise
generally shifts the validation-optimal dimension downward across datasets
and both model sizes, though the per-dataset trend is not perfectly
monotonic (Tables~\ref{tab:optimal-m-downstream-pythia-70m-deduped}--\ref{tab:optimal-m-pretrain-pythia-410m-deduped}).
Averaging the validation-optimal dimension across datasets, downstream PCA
becomes 46.8\% more compressed under the highest noise level relative to the
noise-free setting for Pythia-70M (mean dimension: $185.2 \to 98.5$) and
54.2\% for Pythia-410M (mean dimension: $320.8 \to 146.8$); pretraining PCA
shows the same direction of effect, at 39.9\% and 50.9\% respectively.

Note that dataset WNLI does not follow this otherwise consistent trend: its estimated optimal dimensionality increases at higher noise levels instead of decreasing. We believe this is primarily a finite-sample instability rather than evidence of a systematic reversal of the predicted relationship. WNLI is by far the smallest GLUE task in our experiments, containing only 634 training examples, while the other datasets range up to hundreds of thousands of samples. Moreover, its development set has a known adversarial construction artifact, as sentences showing up in training with one label often show up in validation with a flipped label~\cite{wang_superglue_2019a}. Consequently, both the measured validation performance and the selected minimizing value of $m$ are substantially more variable than for the other datasets, and we interpret WNLI as a noisy finite-sample exception rather than a systematic effect of increasing noise.

\begin{figure}
    \centering
    \includegraphics[width=.97\linewidth]{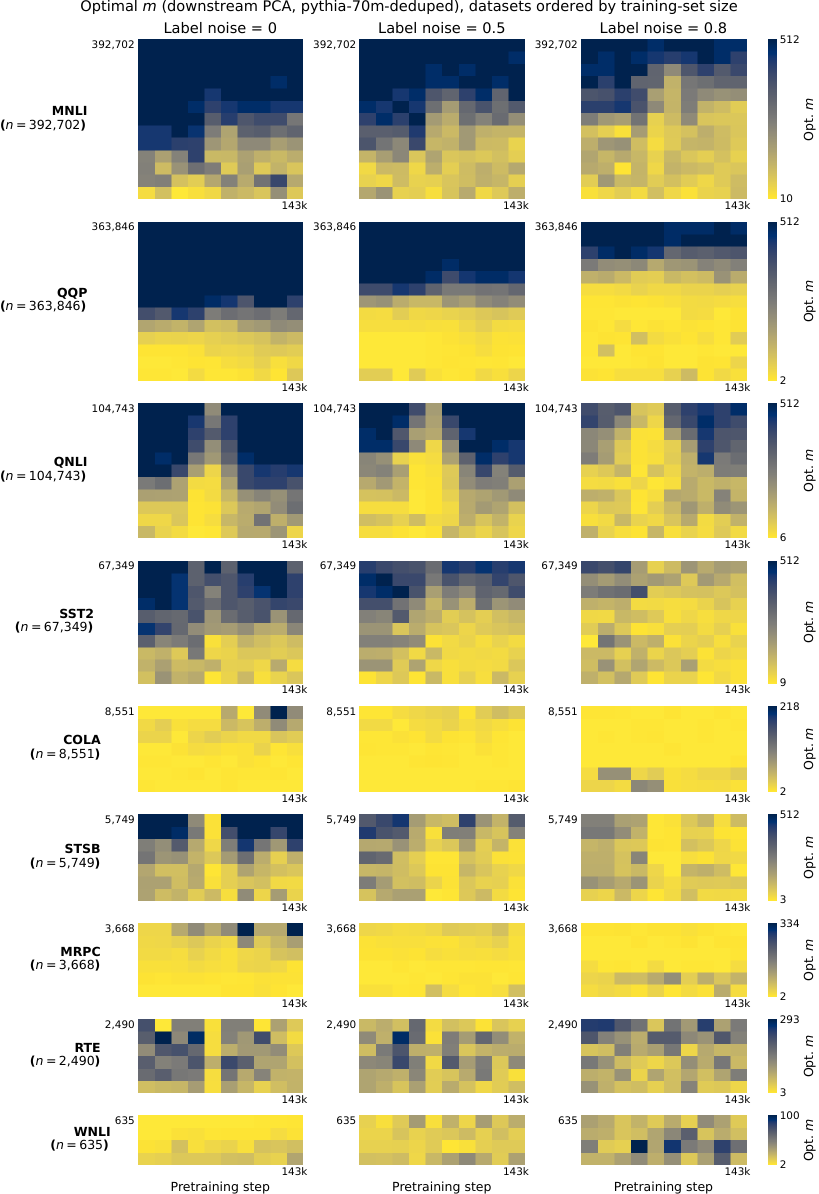}
    \caption{Validation-optimal number of PCA components $m$ for linear probes on
    \texttt{pythia-70m-deduped} representations ($p = 512$), with PCA fit on the
    labelled downstream data, across nine GLUE datasets and three label-noise levels
    (columns: fraction of corrupted training targets $0$, $0.5$, $0.8$). Each heatmap
    shows the optimal $m$ as a function of pretraining checkpoint (horizontal axis,
    up to step 143k) and labelled training-set size $n_l$ (vertical axis, largest at
    the top). The optimal $m$ is the argmax of validation accuracy for each seed
    (smaller $m$ on ties), averaged over 6 seeds. Datasets are ordered by training-set
    size, descending, and each dataset has its own colour scale. Increasing label noise
    generally lowers the optimal $m$, with WNLI the main exception
    (Table~\ref{tab:optimal-m-downstream-pythia-70m-deduped}).}
    \label{fig:all_datasets_downstream} 
\end{figure}

\begin{figure}
    \centering
    \includegraphics[width=\linewidth]{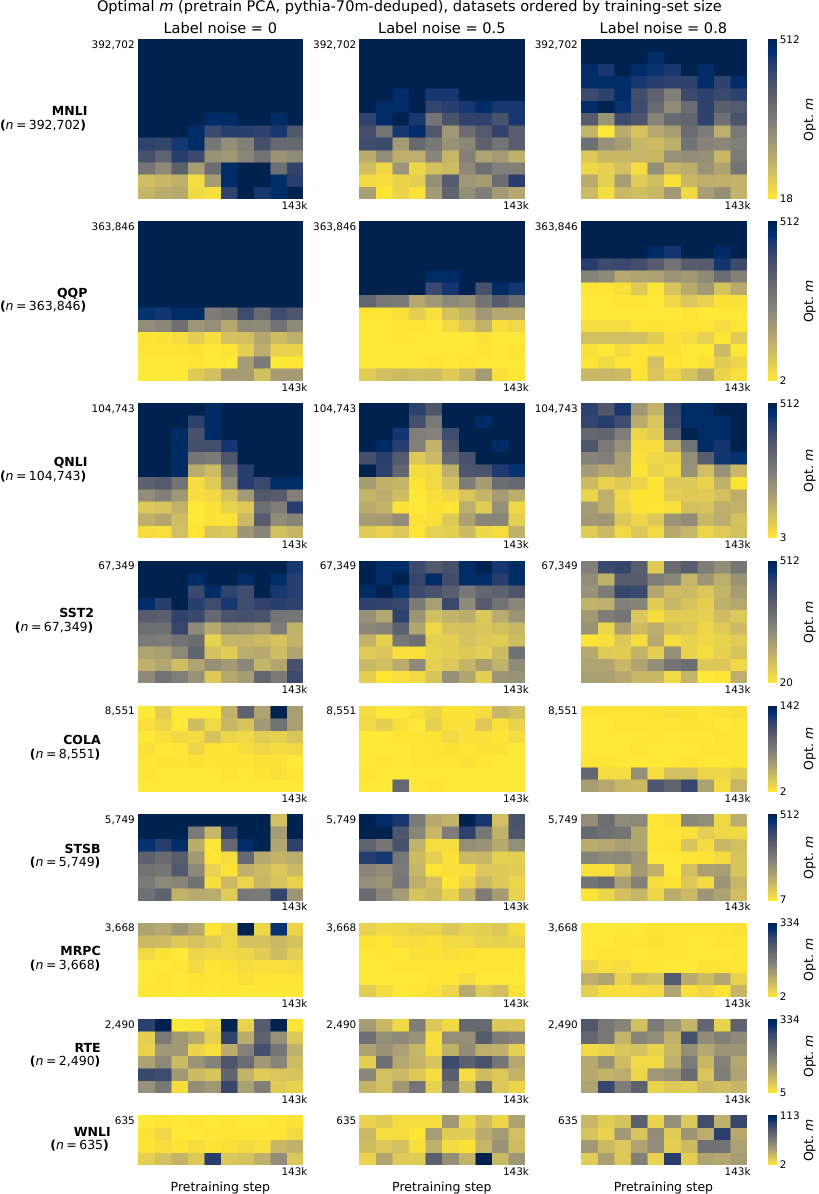}
    \caption{As Figure~\ref{fig:all_datasets_downstream}, but with PCA fit on
    pretraining data (the Pile) rather than the labelled downstream data. The
    validation-optimal $m$ shows the same downward shift with increasing label noise
    (Table~\ref{tab:optimal-m-pretrain-pythia-70m-deduped}).}
    \label{fig:all_datasets_pretrain}
\end{figure}

\begin{table}[t]
\centering
\begin{tabular}{lrrrr}
\toprule
Dataset & Noise = 0 & Noise = 0.2 & Noise = 0.5 & Noise = 0.8 \\
\midrule
MNLI & 364.1 & 351.7 & 305.9 & 227.3 \\
QQP & 320.5 & 288.4 & 243.5 & 145.1 \\
QNLI & 307.7 & 274.4 & 232.5 & 159.1 \\
SST-2 & 273.1 & 243.7 & 198.5 & 105.5 \\
CoLA & 17.0 & 13.3 & 6.8 & 10.4 \\
STS-B & 221.1 & 192.4 & 132.2 & 81.9 \\
MRPC & 41.5 & 24.1 & 14.1 & 18.5 \\
RTE & 113.0 & 85.5 & 89.6 & 103.9 \\
WNLI & 8.1 & 11.1 & 16.4 & 34.8 \\
\bottomrule
\end{tabular}
\caption{Mean validation-optimal PCA dimension $m$ for downstream PCA, \texttt{pythia-70m-deduped}, averaged over checkpoints, downstream training-set sizes, and seeds, for each GLUE dataset and label-noise level. Datasets are ordered by training-set size, descending. Hidden-layer size is $p = 512$.}
\label{tab:optimal-m-downstream-pythia-70m-deduped}
\end{table}

\begin{table}[t]
\centering
\begin{tabular}{lrrrr}
\toprule
Dataset & Noise = 0 & Noise = 0.2 & Noise = 0.5 & Noise = 0.8 \\
\midrule
MNLI & 431.9 & 419.9 & 367.5 & 306.1 \\
QQP & 340.6 & 315.4 & 267.3 & 183.6 \\
QNLI & 334.3 & 315.6 & 271.8 & 191.6 \\
SST-2 & 326.5 & 294.6 & 244.5 & 154.0 \\
CoLA & 13.7 & 12.2 & 6.5 & 13.6 \\
STS-B & 275.9 & 240.6 & 199.5 & 116.7 \\
MRPC & 34.9 & 27.6 & 15.9 & 22.6 \\
RTE & 127.6 & 116.8 & 107.9 & 113.2 \\
WNLI & 11.9 & 14.4 & 24.3 & 37.1 \\
\bottomrule
\end{tabular}
\caption{Mean validation-optimal PCA dimension $m$ for pretraining PCA, \texttt{pythia-70m-deduped}, averaged over checkpoints, downstream training-set sizes, and seeds, for each GLUE dataset and label-noise level. Datasets are ordered by training-set size, descending. Hidden-layer size is $p = 512$.}
\label{tab:optimal-m-pretrain-pythia-70m-deduped}
\end{table}

\begin{table}[t]
\centering
\begin{tabular}{lrrrr}
\toprule
Dataset & Noise = 0 & Noise = 0.2 & Noise = 0.5 & Noise = 0.8 \\
\midrule
MNLI & 672.1 & 633.2 & 532.8 & 355.0 \\
QQP & 576.6 & 522.1 & 419.9 & 211.6 \\
QNLI & 562.4 & 515.4 & 427.7 & 260.8 \\
SST-2 & 428.2 & 351.1 & 285.9 & 167.3 \\
CoLA & 65.3 & 33.6 & 12.1 & 16.3 \\
STS-B & 365.9 & 258.0 & 160.3 & 106.2 \\
MRPC & 47.1 & 31.3 & 14.5 & 21.6 \\
RTE & 161.8 & 160.0 & 130.7 & 140.4 \\
WNLI & 8.0 & 12.7 & 19.6 & 42.0 \\
\bottomrule
\end{tabular}
\caption{Mean validation-optimal PCA dimension $m$ for downstream PCA, \texttt{pythia-410m-deduped}, averaged over checkpoints, downstream training-set sizes, and seeds, for each GLUE dataset and label-noise level. Datasets are ordered by training-set size, descending. Hidden-layer size is $p = 1024$.}
\label{tab:optimal-m-downstream-pythia-410m-deduped}
\end{table}

\begin{table}[t]
\centering
\begin{tabular}{lrrrr}
\toprule
Dataset & Noise = 0 & Noise = 0.2 & Noise = 0.5 & Noise = 0.8 \\
\midrule
MNLI & 839.3 & 789.2 & 671.3 & 507.0 \\
QQP & 733.8 & 628.3 & 484.9 & 297.2 \\
QNLI & 665.0 & 584.8 & 505.4 & 315.0 \\
SST-2 & 601.7 & 504.4 & 405.5 & 252.7 \\
CoLA & 60.9 & 24.1 & 10.2 & 21.8 \\
STS-B & 584.5 & 437.2 & 301.5 & 191.1 \\
MRPC & 61.5 & 33.4 & 22.7 & 55.0 \\
RTE & 261.9 & 241.1 & 203.8 & 186.6 \\
WNLI & 9.3 & 15.1 & 27.4 & 49.2 \\
\bottomrule
\end{tabular}
\caption{Mean validation-optimal PCA dimension $m$ for pretraining PCA, \texttt{pythia-410m-deduped}, averaged over checkpoints, downstream training-set sizes, and seeds, for each GLUE dataset and label-noise level. Datasets are ordered by training-set size, descending. Hidden-layer size is $p = 1024$.}
\label{tab:optimal-m-pretrain-pythia-410m-deduped}
\end{table}

\subsection{Further discussion}
\label{app:discussion}

\subsubsection{Relation to classical principal component regression}
\label{app:pcr_comparison}

Our two-stage estimator is closely related to classical principal component regression (PCR), but differs in a key structural aspect: the PCA subspace and the regression estimator are learned from independent datasets. In classical PCR~\cite{green_high-dimensional_2025}, the principal components are computed from the same labelled design matrix used for regression. There is therefore a single sample size, whereas our setting separates the amount of unlabelled pretraining data $n_u$ from the amount of labelled downstream data $n_l$.

This difference is already visible at the level of the estimator. In our setting,
\begin{equation}
\what
=
(\X_l\Pm)^\dagger \y_l
=
(\X_l\Pm)^\dagger \X_l\wstar
+
(\X_l\Pm)^\dagger\bxi.
\label{eq:two_sample_estimator_comparison}
\end{equation}
In classical PCR, $\Pm$ is constructed from the right singular vectors of the same matrix $\X_l$. Consequently,
\begin{equation}
(\X_l\Pm)^\dagger\X_l=\Pm,
\label{eq:pcr_identity}
\end{equation}
and the noiseless component of the estimator reduces to $\Pm\wstar$.

In the two-sample setting, $\Pm$ is instead estimated from the independent pretraining matrix $\X_u$. Hence $\Pm$ and $\X_l^\top\X_l$ do not generally share an eigenbasis, and
\begin{equation}
(\X_l\Pm)^\dagger\X_l\neq\Pm.
\label{eq:two_sample_noncommuting}
\end{equation}
The noiseless estimator therefore cannot be reduced to a direct projection of $\wstar$ onto the retained empirical PCs. Its analysis requires additional quantities describing the interaction between the independently learned representation and the downstream design, including the spectra of $\X_l\U_m$ and $\X_l\U_\perp$ and deterministic equivalents for bilinear forms involving the downstream row-space projector $\Pi$. By contrast, the deterministic limits involving $\Pm$ alone are the same as, or closely related to, the corresponding quantities appearing in classical PCR.

The two-sample formulation also changes the relevant interpolation threshold. The downstream regression problem is controlled by
\begin{equation}
\gamma_{\mathrm{eff}}
=
\frac{m}{n_l}
=
\alpha\gamma_l.
\end{equation}
Because the representation is learned independently of the labelled dataset, $m$ is not constrained by the rank of $\X_l$, and both $\gamma_{\mathrm{eff}}<1$ and $\gamma_{\mathrm{eff}}>1$ can arise as the representation dimension is varied. This allows the underparameterized and overparameterized linear-probe regimes to be studied within the same model.

Finally, the failure of~\eqref{eq:pcr_identity} generates additional finite-sample contributions to the error. In particular,
\[
(\X_l\Pm)^\dagger\X_l\wstar
\]
contains contributions from signal components discarded by $\Pm$ but statistically coupled to the retained downstream features. These produce the additional missing-signal and leaked-signal bias terms, together with an associated variance contribution, appearing in our asymptotic expressions. Their magnitude depends jointly on the pretraining and downstream sample sizes, covariance structure, task alignment, and signal-to-noise ratio. In classical PCR, the same term reduces to $\Pm\wstar$, so the corresponding bias is determined directly by the signal excluded from the retained empirical-PC subspace.

Thus, although both settings use PCA followed by linear regression, separating representation learning from downstream fitting changes both the estimator algebra and the resulting high-dimensional asymptotics.

\paragraph{Numerical comparison with regression and PCR}
\label{app:compare_regressions}

We compare three approaches: standard regression, pretrained regression (PR) where PCA is learned on prior data and transferred to the downstream task, and principal component regression (PCR), where PCA is learned directly on the downstream data.

In the fully aligned setting ($\wstar \parallel \bv$, $\cos\theta = 1$), Figure~\ref{fig:compare_regressions} (left) shows that incorporating a PCA step consistently improves performance over standard regression. Moreover, when the pretraining dataset is larger than the downstream dataset ($n_u > n_l$), PR outperforms PCR, reflecting the more accurate estimation of the signal subspace from additional data.

In the misaligned setting ($\cos\theta = 1/\sqrt{2}$, right panel in Figure~\ref{fig:compare_regressions}), the advantage of PR is restricted to regimes with large $n_u$ and small $n_l$. As $n_l$ increases, PCR becomes comparable to PR, and both methods approach the performance of standard regression (within a small error gap). This reflects the diminishing benefit of transferring a pretrained subspace when sufficient downstream data is available to estimate the relevant structure directly.

\begin{figure}[h!]
\begin{center}
\includegraphics[width=0.8\textwidth]{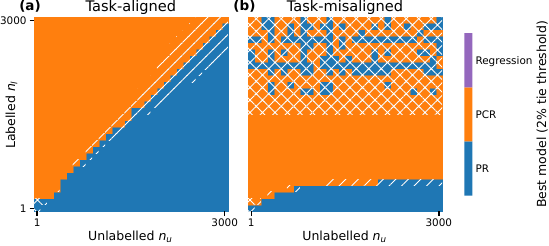}
\end{center}
\caption{Heatmaps over the $(n_u, n_l)$ grid ($p=500$, $\mathrm{SNR}=1.8$, $\lambda=5$) showing which method -- standard regression (violet), pretrained regression (PR, blue) or principal component regression (PCR, orange) -- achieves lower generalisation error. Forward hatching (//) marks regions where PR and PCR are tied (within 2\% error); cross hatching (×) marks regions where all three methods (PR, PCR, and standard regression) are tied. When the task is aligned with the pretraining signal, PR wins for $n_u > n_l$ and PCR wins otherwise, with the boundary tracking the diagonal. Under misalignment, PCR dominates across most of the grid; PR retains an advantage when $n_u \gg n_l$.} \label{fig:compare_regressions}
\end{figure}

\subsubsection{Fine-tuning}
\label{app:finetuning}
In practice, fine-tuning often involves attaching a re-initialized output head with small weights. This regime is closely related to the setting considered in our analysis. As demonstrated by \cite{domine2024lazy,nam2025position,kunin2024get}, a small initialization can lead to dynamics in which the encoder weights change relatively little, while most of the adaptation occurs in the output weights. We therefore expect our theory to provide a useful approximation to full fine-tuning in regimes where feature learning remains limited.
More generally, however, full fine-tuning may transform the projection matrix, turning the leakage term from a fixed inherited artefact into an adaptive component. Gradient descent can rotate the encoder and change the learned subspace, so the fixed decomposition into the projection matrix and its orthogonal complement need not remain valid throughout training. In the extreme case,the leakage may be reduced substantially or even eliminated. The leakage term depends on the missing signal , corresponding to task-relevant directions outside the initially inherited subspace. By updating the encoder, full fine-tuning may incorporate some of these directions into the learned representation, thereby reducing the contribution of the leakage term.

\subsection{Implementation details}
\label{app:implementation}

Code to reproduce all figures and experiments is available on \href{https://github.com/valentinanjaradi/opt_pretrain_repr}{GitHub}.

\subsubsection{Figure \ref{fig:fig0}}

\paragraph{Panel (a)} Diagram of the model made with \texttt{figma}.

\paragraph{Model.}
We define the signal-to-noise ratio as \( \mathrm{SNR} = \frac{\|\wstar\|^2}{\sigma^2} \) and the alignment between the regression target and the spike as \( r_v^2 = \cos\theta^2 \). Throughout, we fix \( p = 500 \) and \( \sigma^2 = 1 \) aligned spikes 
$\mathbf{v}=\mathbf{q}$, no regularizer.

\paragraph{Panels (b-c) (generalisation and training error vs \( \alpha \)).}
We fix \( n_u \approx n_l \approx 300 \), \( \mathrm{SNR} = 9.0 \), and \( r_v^2 = 1.0 \). The parameter \( \alpha \) is swept over 30 values in \( [0.01, 1.0] \) for each of four spike strengths \( \lambda \in \{1.1, 2.0, 5.0, 10.0\} \). Theoretical predictions (solid curves) are compared against empirical averages over multiple realisations (dashed curves, shading corresponds to 1 standard deviation).

\subsubsection{Figure \ref{fig:fig1}}

\paragraph{Model.}
We define the signal-to-noise ratio as \( \mathrm{SNR} = \frac{\|\wstar\|^2}{\sigma^2} \) and the alignment between the regression target and the spike as \( r_v^2 = \cos\theta ^2\). Throughout, we fix \( p = 500 \) and \( \sigma^2 = 1 \) aligned spikes $\mathbf{v}=\mathbf{q}$, no regularizer.

\paragraph{Panel (a) (minimum generalisation error).}
For $\mathrm{SNR} = 9.0$, $\lambda^\star = 5.0$, and $r_v^2 = 1.0$, we plot
$\min_\alpha E_{\rm gen}(\alpha)$ over a $100 \times 100$ grid with
$n_u, n_l \in [1, 3000]$.
Dashed white lines mark $n_u = p$ and $n_l = p$ ($p = 500$).
The colour scale is logarithmic.

\paragraph{Panel (b) (gain from tuning $\alpha$ relative to $\alpha = 1$).}
Under the same setting, we compute $E_{\rm gen}(\alpha\!=\!1) - E_{\rm gen}(\alpha^\star)$
over the same grid.
Positive values indicate regions where using the full prior representation is suboptimal
and tuning $\alpha$ reduces generalisation error. The colour scale is logarithmic.

\paragraph{Panel (c) (marginal substitution rate).}
Under the same setting, we compute the ratio
$(\partial E_{\rm gen} / \partial n_u) / (\partial E_{\rm gen} / \partial n_l)$
evaluated at the optimal $\alpha^\star$, with gradients approximated numerically
via central differences on the $100\times100$ grid.
Values above~1 (red) indicate that collecting an additional unlabelled sample $n_u$
reduces generalisation error more than collecting a labelled sample $n_l$;
values below~1 (blue) indicate the converse. The colour scale is centred at~1.

\paragraph{Panel (d) (generalisation-optimal $\alpha^\star$).}
Under the same setting, we plot the value of $\alpha^\star = \arg\min_\alpha E_{\rm gen}(\alpha)$
over the same grid, normalised so that $\alpha \in [0,1]$.

\subsubsection{Figure \ref{fig:fig2}}

\paragraph{Model.}
We define the signal-to-noise ratio as \( \mathrm{SNR} = \frac{\|\wstar\|^2}{\sigma^2} \) and the alignment between the regression target and the spike as \( r_v^2 = \cos\theta^2 \). Throughout, we fix \( p = 500 \) and \( \sigma^2 = 1 \).

\paragraph{Panel (a) (phase portrait SNR).}
At fixed \( \lambda = 5.0 \), \( r_v^2 = 1.0 \), aligned spikes  $\mathbf{v}=\mathbf{q}$  and no regulariser, we plot the smooth and sharp transition boundaries in the \( (n_u, n_l) \) plane for each of the six SNR values \( \{0.1, 0.5, 1.0, 3.0, 5.0, 9.0\} \). Increasing SNR progressively reduces the region in which full-rank representations (\( \alpha \approx 1 \)) are suboptimal.

\paragraph{Panel (b) (phase portrait spike strength).}
At fixed \( \mathrm{SNR} = 9.0 \), \( r_v^2 = 1.0 \), aligned spikes $\mathbf{v}=\mathbf{q}$ and no regulariser, we show the corresponding transition curves for \( \lambda \in \{2, 3, 5, 7, 10\} \). Stronger spikes shift the phase boundaries.

\paragraph{Panel (c) (phase portrait alignment).}
At fixed \( \mathrm{SNR} = 9.0 \),\( \lambda = 5.0 \), aligned spikes $\mathbf{v}=\mathbf{q}$  and no regulariser, we plot the transition curves for \( r_v^2 \in \{0.0, 0.25, 0.5, 0.75, 1.0\} \).  All panels are evaluated over \( n_u, n_l \in [1, 3000] \).

\paragraph{Panel (d) (phase portrait alignment).}
At fixed \( \mathrm{SNR} = 5.0 \), pretraining spike strength \( \lambda_u = 10.0 \), downstream spike strength \( \lambda_\ell = 5.0 \), and task-spike alignment \( \rho_{wq} = 1.0 \) (the task vector \( \wstar \) fully aligned with the downstream spike \( \mathbf{q} \)), we plot the smooth and sharp transition boundaries in the \( (n_u, n_l) \) plane for the two-spike setting, sweeping the spike-spike alignment \( r = \mathbf{v}^\top \mathbf{q} \in \{0.0, 0.25, 0.5, 0.75, 1.0\} \) between the pretraining spike \( \mathbf{v} \) and the downstream spike \( \mathbf{q} \). Increasing \( r \) enlarges the region in which compression (\( \alpha < 1 \)) is optimal.

\subsubsection{Figure \ref{fig:fig3}}

\paragraph{Autoencoders}

\paragraph{Data generation.}
Both covariances use $p = 200$ features, noise standard deviation $\sigma = 1$, $\mathrm{SNR} = 9, \|\wstar \|=3$.
\begin{itemize}
    \item \textit{Spiked identity ($\lambda=5.0$)}: $\bSigma = \mathbf I_p + \lambda \bv \bv^\top$, i.e.\ the top eigenvalue is $1+\lambda=6.0$; true weight vector aligned with the spike $\wstar = \|\wstar\| \bv$, where $\bv$ is the top eigenvector of $\bSigma$.
    \item \textit{Power-law} ($\beta = 0.3$): $\bSigma = \mathrm{diag}\big(1^{-\beta}, 2^{-\beta}, \ldots, p^{-\beta}\big)$; no additional spike is added, and $\wstar$ is aligned with the top eigenvector (the first standard basis vector, since eigenvalues are already sorted in decreasing order).
\end{itemize}
Each grid point $(n_u, n_l)$ averages five independent replications, with $n_u, n_l \in [1, 1200]$ each swept over $20$ uniform values.

\paragraph{Models.}
For each $(n_u, n_l)$ pair, $\alpha \in \{0.01, 0.05, \ldots, 1.0\}$ (20 values) sets the latent dimension $m = \lfloor \alpha p \rfloor$.  The three models below are evaluated.
\begin{enumerate}
    \item \textit{PCA + Regression (PR)}: Eigenvectors of $\Ss_u = \X_u^\top \X_u / n_u$ are computed; $\X_l$ is projected onto the top-$m$ eigenvectors and OLS is applied in the latent space.
    \item \textit{Linear AE}: Encoder ($\mathbb{R}^p \to \mathbb{R}^m$) and linear decoder ($\mathbb{R}^m \to \mathbb{R}^p$) trained by minimizing MSE on $\X_u$; the frozen encoder-decoder pair reconstructs $\X_l$ back into $\mathbb{R}^p$, and OLS is applied to this reconstructed $p$-dimensional representation (not the $m$-dimensional latent code).
    \item \textit{Nonlinear AE}: Same as the linear AE but with a sigmoid nonlinearity applied at the encoder, $\mathbf z = \sigma(\mathbf W_{\mathrm{enc}}\, \x)$.
\end{enumerate}
All autoencoders are implemented in \texttt{JAX/Equinox} and trained with Adam ($\mathrm{lr} = 10^{-3}$, 20{,}000 epochs, batch size 500). Weights are initialized as $\mathbf W_{i,j} \sim \mathcal{N}(0, 1/\sqrt{d_{\mathrm{in}}})$.

\paragraph{Panel descriptions.}
Each panel is a $20 \times 20$ heatmap with $n_u$ on the $x$-axis, $n_l$ on the $y$-axis, and color encoding $\alpha^\star \in [0,1]$ (the value minimizing mean loss over the five replications).

The autoencoder component of the figure is organized as 2 rows $\times$ 2 columns:

\begin{center}
\begin{tabular}{lcccc}
\toprule
 & Linear AE & Nonlinear AE \\
\midrule
Spiked identity   & (a) & (b)  \\
Power law & (c) & (d) \\
\bottomrule
\end{tabular}
\end{center}

Columns compare linear vs nonlinear AEs; rows contrast spiked identity and power-law covariance structures; each panel shows generalisation-optimal $\alpha^\star$. A single shared colorbar applies to all panels.

\paragraph{LLMs}
\paragraph{Model and representations.}
All experiments in this Figure use \texttt{Pythia-70M-deduped} \cite{biderman2023pythia} (70M parameters, trained on the deduplicated Pile \cite{gao2020pile, biderman2022datasheet}). We evaluate 10 checkpoints at steps 2, 8, 32, 128, 512, 2000, 5000, 20000, 80000, and 143{,}000. Representations are last-token hidden states from the final layer; downstream (SST-2) inputs are never truncated, while the Pile sequences used to fit the pretrain-PCA basis are truncated to 64 tokens, yielding vectors in $\mathbb{R}^{512}$.

\paragraph{Downstream task.}
Binary sentiment classification on SST-2 from the GLUE benchmark \cite{wang_glue_2019}. Fixed 872-sentence validation set; training sizes $n_l \in \{100, 250, 500, 1000, 2000, 4000, 10000, 20000, 40000, 67349\}$, $n_l$ training examples are randomly chosen out of the maximum 67349 for each seed. 

\paragraph{PCA.}
We sweep $m \in \{2, 3, 4, 7, 11, 16, 25, 39, 60, 92, 142, 218, 334, 512\}$ (approximate $\log_2$ grid), and consider the two conditions below.
\begin{enumerate}
    \item \textit{Downstream PCA} (Figure~\ref{fig:app_transformers_downstream_pca}): basis fit on the $n_l$ labelled training representations. The representation used is the hidden state at the last non-padding token.
    \item \textit{Pretrain-PCA} (Figure~\ref{fig:fig3}\textbf(e-g)): basis fit on up to 102{,}400 Pile sequences (100 batches $\times$ 1{,}024 sequences) from the same checkpoint, using \texttt{torch.linalg.eigh}. For earlier checkpoints, the basis is fit on $1024 \times n_\mathrm{checkpoint}$ sequences. Each Pile sequence is first truncated to its first $64$ tokens to match the lengths of sequences in SST2, and the representation used is the final-layer hidden state at the last token of that truncated window.

\end{enumerate}

\paragraph{Linear probe.}
Projected representations are fed to a 2-output linear layer trained with cross-entropy loss using Adam ($\text{lr} = 10^{-3}$, weight decay $10^{-4}$), for up to 10{,}000 epochs with early stopping (200 steps without improving the best loss by more than $10^{-4}$). Linear probes were randomly initialized with small weights for each seed.

\paragraph{Seeds and aggregation.}
Each (checkpoint, $n_l$, $m$) configuration is run with 6 random seeds (affecting probe initialization and downstream data shuffling). Figures show mean $\pm$ std over seeds; optimal-$m$ heatmaps report, for each (checkpoint, $n_l$), the mean over seeds of the per-seed argmax-$m$ (the $m$ maximizing validation accuracy is first selected independently for each seed, and these per-seed optimal values of $m$ are then averaged).

\paragraph{Panel descriptions.} 
\begin{itemize}
    \item \textbf{(e)} Heatmap for optimal $m$ under pretrain-PCA, selected by validation accuracy. 
    \item \textbf{(f)} Eigenvalue histogram at step 143{,}000, computed from the 872-sentence SST-2 validation covariance; log-density on the $y$-axis.
    \item \textbf{(g)} Validation accuracy vs $m$ ($x$-axis in $\log_2$ scale) under pretrain-PCA for $n_l = 2000$, with curves colored by pretraining progress.
\end{itemize}

\subsubsection{Compute}

\paragraph{Theory curves.}
Theory curves were computed on a single CPU (Intel Xeon E5-2680 v4 @ 2.40 GHz,
no GPU required). Each configuration sweeps a $100 \times 100$ grid of
$(n_u, n_l)$ values over 30 values of $\alpha$, amounting to approximately
$300{,}000$ evaluations of closed-form expressions. The full sweep covers
24 configurations,
$\lambda^* \in \{1.1, 2.0, 5.0, 10.0\}
\quad \times \quad
\mathrm{SNR} \in \{0.1, 0.5, 1.0, 3.0, 5.0, 9.0\},
$ with alignment $r_v^2 = 1.0$, totaling approximately $7.2$ million evaluations.
When executed in parallel, the full sweep completes in approximately 2h. Peak RAM usage was approximately 2 GB per configuration.

\paragraph{Simulations.}
Simulations were distributed over a SLURM cluster. Each job was allocated
4 CPUs and 30 GB of RAM. The full grid
$ 6\ \mathrm{SNR} \times 4\ \lambda_1 \times 4\ \lambda_2 \times 30\ n_b $
produced 2,880 SLURM jobs, each processing 30 values of $n_l$ with 100
independent replications at $p = 500$. Each replication involves an
eigendecomposition and pseudoinverse of a $p \times p$ matrix for each of the 30 values of $\alpha$. When executed in parallel, the full sweep completes in approximately 10h.

\paragraph{Autoencoder experiment.}
Simulations were distributed over a SLURM cluster. Each job was allocated
1 GPU, 8 CPUs, and 16 GB of RAM. The full grid
$20\ n_{\mathrm{base}} \times 20\ n_{\mathrm{task}}$
produced 400 SLURM jobs, each processing 20 values of $\alpha$ with 5
independent replications at $p = 200$. Each replication trains a linear and a
nonlinear autoencoder for 20,000 epochs for each value of $\alpha$, followed by
OLS regression in the learned latent space. The individual jobs completed in 10-50min depending on parameters.

\paragraph{LLM probe experiment.}
Simulations were distributed over a SLURM cluster. Each job was allocated
1 GPU, 16 CPUs, and 32 GB of RAM. Two sets of 10 SLURM jobs were submitted:
one per checkpoint of Pythia-70M-deduped for the downstream-PCA condition and
one for the pretrain-PCA condition. Each job sweeps over 10 training-set sizes
and 6 seeds, fitting a linear probe for up to 10,000 epochs with early stopping
across up to 14 PCA dimensionalities $m$.

\subsubsection{LLM Usage }
\label{llm}

Large language models (LLMs) were used throughout this research project to assist with writing, editing, and formatting the manuscript, including improving clarity and refining \LaTeX{} expressions. LLMs were also used as an aid for understanding and discussing technical concepts, as well as for debugging and drafting portions of the experimental and simulation code. All scientific results, analyses, and conclusions were verified and validated by the authors.

\end{document}